\documentclass[3p,article]{elsarticle}
\usepackage[utf8]{inputenc}
\usepackage[T1]{fontenc}
\usepackage{textcomp}
\usepackage{algorithm}
\usepackage{algorithmic}

\usepackage{xcolor}

\usepackage{hyperref}
\usepackage{amsthm}
\usepackage{newtxtext,newtxmath} 
\usepackage{CJK}

\everymath{\displaystyle}              

\newtheorem{proposition}{Proposition}[section]
\newtheoremstyle{exampleitalic}
  {3pt}   
  {3pt}   
  {\normalfont} 
  {}      
  {\itshape} 
  {.}     
  {0.5em} 
  {}      
\theoremstyle{exampleitalic}
\newtheorem{innerexample}{Example}[section]
\newenvironment{example}
  {%
   \begin{innerexample}}
  {\end{innerexample}}
\usepackage{array,color}
\usepackage{booktabs}
\usepackage{multirow}
\usepackage{chngpage}
\usepackage{lscape}
\usepackage{amsfonts}
\usepackage{amsmath}
\usepackage{enumitem}
\usepackage{graphicx}
\usepackage{bm}
\usepackage{bbm}
\usepackage[caption=false,font=footnotesize]{subfig}
\usepackage{float}
\usepackage[export]{adjustbox} 
\usepackage{mathtools}

\biboptions{numbers,sort&compress}

\journal{Information Sciences}

\newcounter{promptbox}
\usepackage{caption}          

\usepackage{listings}

\usepackage[most]{tcolorbox}  
\definecolor{boxblue}{RGB}{0,126,214}
\definecolor{boxbg}{RGB}{235,245,255}

\lstdefinestyle{prompttiny}{
  basicstyle=\scriptsize\ttfamily\selectfont,  
  breaklines=true,
  breakatwhitespace=true,
  keepspaces=true,
  columns=fullflexible
}

\newtcolorbox{promptmini}[2][]{%
  enhanced, 
  colback=boxbg, 
  colframe=boxblue,
  coltitle=white, 
  colbacktitle=boxblue,
  fonttitle=\bfseries\footnotesize,
  attach boxed title to top left={xshift=4mm,yshift*=-1mm},
  arc=1mm, boxed title style={arc=1mm},
  left=2mm, right=2mm, top=1.5mm, bottom=1.5mm,
  boxrule=0.6pt,
  listing only,                         
  listing options={style=prompttiny},   
  width=\columnwidth,
    before upper = {\refstepcounter{promptbox}},
  #1}

\begin{document}

\begin{frontmatter}

\title{Large Language Model as Meta-Surrogate for Offline Data-Driven Many-Task Optimization: A Proof-of-Principle Study}

\tnotetext[mytitlenote]{
    This work was supported by the Guangdong Natural Science Funds for Distinguished Young Scholars under Grant 2022B1515020049; by the National Natural Science Foundation of China under Grant 62276100; and by the National Research Foundation of Korea under Grant RS-2025-00555463. 
}
\author[aff1]{Xian-Rong Zhang}
\author[aff1]{Yue-Jiao Gong\corref{mycorrespondingauthor}}
\ead{gongyuejiao@gmail.com}
\author[aff1]{Yuan-Ting Zhong}
\author[aff2]{Ting Huang}
\cortext[mycorrespondingauthor]{Corresponding author}
\author[aff3,aff4]{Jun Zhang}

\address[aff1]{School of Computer Science and Engineering, South China University of Technology, Guangzhou, China.}
\address[aff2]{Guangzhou Institute of Technology, Xidian University, China.}
\address[aff3]{College of Artificial Intelligence, Nankai University, Tianjin, China.}
\address[aff4]{Hanyang University, 15588 Ansan, South Korea.}

\begin{abstract}
In offline data-driven optimization scenarios, where new evaluation data cannot be obtained in real time and each ground-truth evaluation is often costly, surrogate models become a key technology for reducing simulation or experimental overhead. This study proposes a novel meta-surrogate framework to assist many-task offline optimization, by leveraging the knowledge transfer strengths and emergent capabilities of large language models (LLMs). We formulate a unified framework for many-task fitness prediction, by defining a universal model with metadata to fit a group of problems. Fitness prediction is performed on metadata and decision variables, enabling efficient knowledge sharing across tasks and adaptability to new tasks. The LLM-based meta-surrogate treats fitness prediction as conditional probability estimation, employing a unified token sequence representation for task metadata, inputs, and outputs.  This approach facilitates efficient inter-task knowledge sharing through shared token embeddings and captures complex task dependencies via many-task model training. Experimental results demonstrate the model's emergent generalization ability, including zero-shot performance on problems with unseen dimensions. When integrated into evolutionary transfer optimization (ETO), our framework supports dual-level knowledge transfer---at both the surrogate and individual levels---enhancing optimization efficiency and robustness. This work establishes a novel foundation for applying LLMs in surrogate modeling, offering a versatile solution for many-task optimization.
\end{abstract}


\end{frontmatter}


\section{Introduction}
{Evolutionary} Algorithms (EAs) have been successfully applied to a wide range of complex optimization problems. However, in expensive optimization scenarios, EAs may perform poorly because they often require a large number of { real} function evaluations (FEs). To address this issue, existing studies incorporate surrogates into the {EA} framework~\cite{Expensive_Optimization_overview} to reduce the reliance on real FEs. This methodology is commonly referred to as Surrogate-Assisted Evolutionary Algorithms (SAEAs). Moreover, many expensive optimization problems do not allow direct computation of objective or constraint function values. Instead, they rely only on data collected from physical experiments, real-world events, or complex numerical simulations. This paradigm, which relies entirely on empirical data, is known as Data-Driven Evolutionary Algorithms {(DDEA)}~\cite{DDEA_overview}. 
Furthermore, problems that permit active sampling of individuals (data) through FEs are referred to as Online DDEAs~\cite{online_ddea}. {Conversely}, problems that only have access to historical statistical data 
are known as Offline DDEAs~\cite{offline_ddea}. 
Examples of Offline DDEA applications include ceramic formula design~\cite{Ceramic_Formula}, and hardware accelerator design~\cite{Hardware_Accelerators}.
Recent advances in offline DDEA research have addressed challenges such as data sparsity, surrogate model bias, and overfitting to noisy historical samples. For instance, 
Zhen \emph{et al.}~\cite{Zhen2025PDO} proposed a problem-driven model selection strategy that adaptively chooses the most appropriate surrogate according to the characteristics of the industrial dataset. { while Zhong \emph{et al.}~\cite{zhong2025data} applied a drift-aware streaming evolutionary algorithm to enable rapid adaptation in continuously changing dynamic offline environments.}
Despite these advances, most existing offline DDEAs focus on single-task settings, leaving multi-task or many-task offline optimization largely unexplored.

Evolutionary Transfer Optimization (ETO)~\cite{ETO} has recently emerged as a rapidly growing research topic that has attracted significant attention. {Its core idea is to leverage optimization experience or domain knowledge acquired from solving certain source tasks to enhance the search performance on other, related target tasks}. 
The technique of reusing information from source tasks to facilitate solving target tasks is known as knowledge transfer~\cite{OTMO}. Among the optimization problems addressed by ETO, {Multi-task Optimization Problems (MTOPs)} involve the simultaneous optimization of multiple tasks. A more challenging subclass, known as Many-task Optimization Problems (MaTOPs), concerns the simultaneous optimization of three or more tasks~\cite{thanh2022ensemble}. 
These problems are typically addressed under the assumption that some tasks share certain degree of similarity. 
Utilizing the EAs to optimize a set of problems concurrently through knowledge transfer has been shown to be more efficient than addressing each problem separately~\cite{TRADE}.
Nonetheless, in the context of ETO, the high costs associated with FEs continue to pose challenges in some real-world situations, presenting a major barrier to effective optimization. 
Consequently, recent studies of ETO~\cite{MaMPSO} have investigated surrogates to mitigate the high computational cost of FEs, demonstrating  their effectiveness in addressing the MaTOPs. 
However, these methods rely on actively sampling online data, which is infeasible in complex real-world applications where online FEs are not permitted.

{Traditional multi-task surrogate models, most notably multi-task Gaussian processes (MTGPs) and their neural counterparts, enhance data efficiency by exploiting cross-task correlations.}
MTGPs achieve this transfer by endowing the prior with a \emph{correlated kernel}, so that information gleaned from one task can inform predictions on another.
A fundamental limitation, however, is that \emph{most} MTGP frameworks presume a homogeneous input space; that is, every task must share identical feature dimensions and semantics.
This assumption seldom holds in real-world optimization, where tasks often differ in decision-vector length, data modality, or domain representation.
Recent work has begun to relax this restriction.
For instance, Min~\textit{et al.} \cite{invTrEMO} extend Bayesian optimization
to mismatched source-target domains by augmenting the GP with a learned
projection, while Liu~\textit{et al.} \cite{HSVLMC_liu} propose a heterogeneous
MTGP whose coregionalisation matrix aligns tasks with dissimilar inputs.
Yet three core challenges remain: (1) {Scalability.} With $T$ tasks and $N$ samples per task, the joint covariance is of size $(TN)\times(TN)$. Exact training via dense Cholesky costs $\mathcal{O}\!((TN)^{3})$ and even a single test prediction requires solving a linear system of cost $\mathcal{O}\!((TN)^{2})$ (or $\mathcal{O}\!(TN\log N)$ with Kronecker CG~\cite{Kronecker_Gp2}). When dealing with many-task scenarios, it has become untenable.
(2) {Input-space rigidity.} Most MTGP kernels mandate identical input dimensionality; heterogeneous tasks therefore need explicit domain-adaptation layers or hand-crafted mappings, which add {modeling} overhead and hyper-parameters. (3) Limited expressiveness of fixed kernels. Classical MTGPs rely on stationary kernels whose ability to capture highly non-linear or structured cross-task relations is limited; performance hinges on careful kernel selection and costly hyper-parameter tuning.

In recent years, Large Language Models (LLMs) have emerged as powerful tools capable of processing textual representations over large-scale heterogeneous datasets, capturing complex relationships between input features and output labels. Given that LLMs have demonstrated effectiveness beyond natural language processing (NLP) in tasks such as code generation~\cite{code_generation}, symbolic mathematics~\cite{Symbolic_Mathematics}, and scientific reasoning~\cite{science_infer}, a natural question arises: Can LLMs serve as many-task surrogate models for numerical regression?
The text-centric nature of LLMs is particularly appealing, as it offers the potential to bypass the need for laborious feature engineering and tensorization of raw inputs. To the best of our knowledge, no prior work has systematically investigated the feasibility and practicality of training a many-task regression predictor based on LLMs for offline DDEAs.
Specifically, we aim to fine-tune a pre-trained model suitable for many tasks and varied dimensions to directly predict the quality of new solutions. We refer to this model as \emph{meta-surrogate}. The main contributions of this paper are summarized below:

\begin{itemize}
    \item \textbf{A paradigm shift for many-task fitness prediction modeling:} 
    We formulate a fitness prediction paradigm through $\mathcal{T} = (\mathcal{X},\mathcal{M}, \mathcal{F}, \mathcal{D})$, where $\mathcal{X} \subset \mathbb{R}^n$ denotes the unified decision space,
 $\mathcal{M}$ denotes a metadata space with intrinsic descriptions to distinguish different tasks (e.g., objective description, dimensional information), $\mathcal{F}=\left\{f_m: \mathcal{X} \rightarrow \mathbb{R} \mid m \in \mathcal{M}\right\}$ is the ground-truth objective family, $\mathcal{D}=\left\{\left(m_i, x_i, y_i\right)\right\}_{i=1}^N$ is a dataset contains cross-objective evaluations adhering to $y_i=f_{m_i}\left(x_i\right)$. Our target is to construct a meta-surrogate $\hat{f}_{\rm{meta}}: \mathcal{X}\times\mathcal{M} \rightarrow \mathbb{R}$ that can process the decision values and problem metadata as input, in order to approximate the evaluation $y_i$ as $\hat{f}_{\rm{meta}}(x_i,m_i)$.
This framework enables knowledge sharing across many-task surrogate modeling, offering benefits like decreased sampling complexity per task and the adaptability to new tasks.

    \item \textbf{A LLM-based meta-surrogate:}
    We propose using a LLM as the meta-surrogate, casting fitness prediction as a conditional probability estimation problem given task metadata: $p(y|x, m)$. This framework represents metadata $m$, inputs $x$, and outputs $y$ by a unified token sequence representation,  
    enabling efficient inter-task knowledge sharing through shared token embeddings. 
    Leveraging the LLM's ability to model complex relationships in high-dimensional spaces, this approach excels in capturing intricate task dependencies and generalizing across diverse optimization problems. Then, by incorporating appropriate prompting during inference, the surrogate generates high-quality fitness predictions across diverse tasks. 

    \item \textbf{Proof of the emergent generalization ability of the meta-surrogate:}
    As an initial investigation, our experiments reveal that the meta-surrogate exhibits emergent generalization capabilities, particularly in dimensional scalability. For instance, it demonstrates zero-shot performance on optimization problems with unseen dimensions not encountered during training, highlighting its potential for dynamically adapting to new task environments.

    \item \textbf{Integration of the meta-surrogate to enhance MaTOP:}
    The meta-surrogate can be seamlessly integrated into most existing ETO algorithms, enabling an efficient offline data-driven many-task optimization framework.  
    Unlike previous data-driven ETO methods that primarily focus on knowledge sharing at the individual level, our approach bridges the gap by supporting knowledge transfer at both the surrogate and individual levels. This dual-level integration not only enhances optimization efficiency, but also provides new possibilities for more robust and versatile MaTOP solutions.
\end{itemize}

The remainder of this paper is organized as follows: \autoref{sec:bg} reviews the background and related work; \autoref{sec:method} presents a detailed description of the proposed algorithm; \autoref{sec:experiments} provides experimental comparisons and analyses; \autoref{sec:applications_potential} discusses the potential applications of MetaSurrogate in real-world scenarios; and \autoref{sec:conclusion} offers concluding remarks and outlines directions for future research.

\section{Preliminaries}
\label{sec:bg}

\subsection{Data-Driven Evolutionary Algorithms}

EAs have proven effective in solving many optimization problems under the common assumption that evaluating candidate solutions is both straightforward and inexpensive~\cite{xue2021evolutionary}. 
However, this assumption rarely holds for real-world optimization tasks. For instance, high-fidelity system optimization~\cite{high_fidelity} and {human-in-the-loop} interactive optimization~\cite{peole} often require computationally intensive numerical simulations or costly physical experiments to assess solution quality. Moreover, in certain practical scenarios such as trauma system optimization~\cite{trauma_systems} and blast furnace optimization~\cite{blast_furnace}, physical constraints may even prevent any evaluations during the evolutionary search process.
To mitigate computational costs, surrogates have been widely employed in evolutionary algorithms, giving rise to DDEAs~\cite{DDEA}. In DDEAs, a surrogate trained on limited data approximates the objective function and/or constraints, thereby reducing the number of expensive evaluations~\cite{DDEA}. 

DDEAs are typically classified into two categories.
In online DDEAs~\cite{online_ddea}, {a small number of FEs can still be obtained during the optimization process.}
Actively resampling data in the promising areas can enhance the accuracy of the surrogate models.
However, the heavy reliance on real FEs {renders} online DDEAs unsuitable for scenarios where objective functions are unavailable or prohibitively expensive to evaluate.
In contrast, offline DDEAs~\cite{offline_ddea} is not allowed to actively resample new data 
once the optimization process has started.
This {setting} is more appropriate for complex real-world problems where only historical data {are} available.
For example, in magnesia furnace optimization~\cite{fused_magnesium_furnaces}, the target values for electricity consumption per ton of magnesite cannot be freely obtained for several reasons. First, they are affected by complex environmental factors---for instance, at high temperatures the mixture of solid, liquid, and gas can cause various sensors to fail. Second, magnesite production involves multiple physical and chemical processes such as melting, impurity precipitation, and crystallization, which makes online evaluation impractical. 

\subsection{Multi-Task and Many-Task DDEAs}
Consider {an} MTOP comprising $NT$ optimization tasks. Each task $k$ ($k = 1, \ldots, NT$), denoted by $T_k$, can be formulated as follows:
\begin{align*}
\min\; y &= f_k(x), \\
\text{subject to }& x \in \mathcal{X}_k, \quad \mathcal{X}_k \subseteq \mathbb{R}^{D_k},
\end{align*}
where $f_k(\cdot)$ is the objective function of $T_k$, $\mathcal{X}_k$ is the search space, and $D_k$ denotes the dimension of that space.
While the pioneering Multi-Factorial Evolutionary Algorithm (MFEA)~\cite{gupta2015multifactorial} imposed no explicit limit on the number of tasks, recent studies~\cite{MaTDE} typically classify problems with more than three tasks as Many-Task Optimization Problems (MaTOP); this paper follows that convention (i.e., $NT>3$).

In recent two decades, several attempts have witnessed the significant development of using surrogates to address expensive optimization problems across multiple tasks. 
The commonly used surrogate models are polynomial regression (PR), Gaussian process (GP) model, radial basis function network (RBFN), and Conditional Neural Process (CNP), etc.
For example, 
SADE-KT~\cite{SADE-KT} develops a surrogate-based hybrid knowledge transfer strategy and a two-level surrogate-assisted search mechanism.

In contrast to standard GP, MTGP~\cite{bonilla2007multi} has gained significant attention for its ability to jointly model multiple tasks while capturing task-relevant information and the landscape knowledge.
Hetero-TBO~\cite{min2020generalizing} {addresses} feature-dimensional heterogeneity between source and target tasks using a transfer GP augmented with neural feature transformation. SELF~\cite{tan2024surrogate} introduces a two-phase optimization framework that employs MTGPs to achieve the tradeoff between the exploration and exploitation to improve the optimization efficiency.

\subsection{Research Motivation}
\label{subsec:motivation}
A surrogate model can be viewed as a regression predictor, with methods such as RBFNs and GPs having long served as the cornerstone of surrogate modeling. 
In the many-task setting, researchers have extended GPs to many-output forms to fit multiple related functions simultaneously. Early approaches include the Linear Model of Coregionalization (LMC) and other many-output GP frameworks, which introduce task-specific kernels combined with a shared latent process to enable information transfer across tasks~\cite{HSVLMC_liu}. However, GP-based many-task surrogates still face three fundamental challenges:
\begin{enumerate}[label=(\arabic*)]  
    \item \textbf{Limited scalability:} Training a GP on \(N\) data points incurs a computational complexity of \(O(N^3)\). When jointly modeling \(T\) tasks (totaling approximately \(T \times N\) points), this complexity escalates to \(O((TN)^3)\), making it infeasible for large-scale scenarios. Although sparse GP approximations and variational methods can partially alleviate the computational burden, they inevitably introduce additional approximation errors. In the originally planned scenario of 50 tasks with 1\,000 samples each (a total of 50\,000 data points), the joint covariance matrix of the MTGP would grow to $50\,000 \times 50\,000$. Even with a Kronecker-structured \emph{lazy tensor} to avoid explicit storage~\cite{Kronecker_Gp2}, \emph{exact} inference still requires repeated computation of 
    $\mathbf K^{-1}\mathbf v$ and $\log\det\mathbf K$, where $\mathbf K^{-1}\mathbf v$ denotes the solution of a linear system with the covariance matrix (needed for posterior means and variances) and $\log\det\mathbf K$ is the log-determinant term that enters the marginal likelihood.  
    The associated time complexity remains quadratic–to–cubic in $TN$: a single Cholesky factorisation alone would exceed $10^{15}$ floating-point operations and occupy about 10 GB of GPU memory---already beyond the capacity of a standard 24GB graphics card.  In practice one must resort to conjugate-gradient (CG) iterations, whose convergence speed is governed by the condition number of $\mathbf K$; however, the Kronecker product amplifies tiny eigenvalues in either the task kernel or the input kernel, often causing residual stagnation and even negative predictive variances.  With highly optimized GPU kernels, per-point inference still takes roughly 10 ms, so batch prediction of merely 1\,000 candidate points induces delays of several seconds---and on a 24~GB GPU we even observed inference failure when \emph{batch~=~1}.  Owing to these resource bottlenecks, the present experiments were forced to down-scale to 500 samples per task and a smaller number of tasks. By contrast, LLMs (e.g.\ GPT-4) exhibit inference speeds that remain essentially constant irrespective of the number of training tasks, reflecting their markedly superior scalability in practice.

    \item \textbf{Complexity of heterogeneous input representation:} Most MTGP frameworks assume that all tasks share the same input feature space, whereas in practice, tasks often differ in dimensionality or data type, making a single kernel difficult to accommodate. To address this, some studies encode task identity as an additional input feature---known as {Contextual Bayesian Optimization (Contextual BO)}~\cite{Context_BO}---by embedding task indicators or features into the input vector and incorporating task similarity into the kernel. However, the performance of contextual BO critically depends on the quality of the task representation; if the task kernel is mis-specified (for example, treating related tasks as unrelated or vice versa), negative transfer may occur. 

     \item \textbf{Complexity of task adaptation:} To address input heterogeneity, some approaches adopt domain adaptation techniques~\cite{Domain_Adaptation} that learn mapping functions to align semantically similar solutions across tasks. However, these mappings must be learned from data and typically require overlapping or strongly related tasks. Any inaccuracies in the learned mapping or task representation can significantly degrade model performance. 
     {Recent advances in deep learning, such as Transformers~\cite{tab_pfn} and graph neural networks (GNNs)~\cite{gnn_runtime_prediction}, relax some of these constraints, yet they still rely heavily on fixed-format \((x, y)\) representations.}
     
\end{enumerate}

{In summary, prior work, ranging from MTGP with contextual kernels to task-aligned neural-network surrogates, relies heavily on engineered input–output and task representations, making it difficult to scale to multi-task scenarios.}
Drawing on insights from these domains, we propose a novel paradigm that employs LLMs as meta-surrogates. For the first time in an offline, data-driven setting for expensive many-task optimization, we fine-tune a pre-trained language model by encoding numerical decision variables, fitness values, and textual metadata into unified token sequences to assess its feasibility and robustness as a many-task regression surrogate. The token paradigm has already proven effective in Reinforcement Learning from Human Feedback (RLHF)~~\cite{rlhf}, where LLMs emulate human evaluators via pairwise ranking to generate regression-style reward scores. More importantly, LLMs can seamlessly integrate textual context into their inputs, offering unprecedented flexibility for many-task surrogate modeling.

\section{Methodology}
\label{sec:method}

In this section, we introduce MetaSurrogate, {an} LLM-based surrogate for assisting many-task DDEAs. Our main objective is to fine-tune a many-task surrogate by harnessing the many-task learning capabilities of LLMs and subsequently develop an offline {DDEA} for MaTOP.

\begin{figure*}[htbp]
    \centering
    \includegraphics[width=0.9\linewidth]{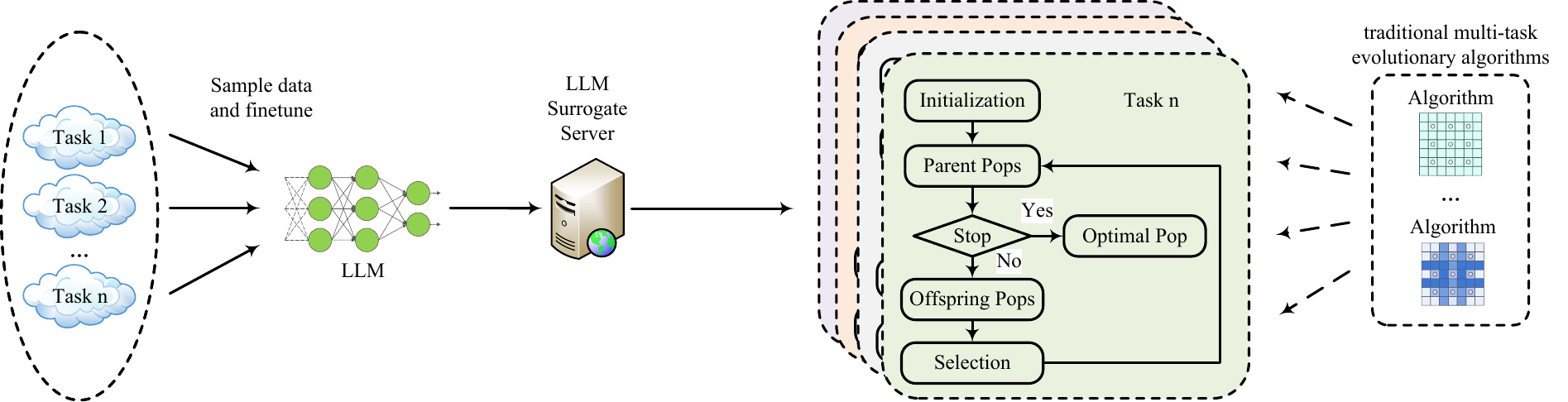}
    \caption{{Overview of the MetaSurrogate framework.}}
    \label{fig:overview}
\end{figure*}

As illustrated in Figure~\ref{fig:overview}, our approach follows the offline DDEAs pipeline. Before the optimization algorithm commences, offline data are sampled from multiple tasks. The LLM is then fine-tuned offline on the dataset corresponding to various tasks. After fine-tuning, the resulting model is deployed as a microservice, providing fitness prediction across different optimization tasks and enabling result retrieval via HTTP requests.
During many-task optimization, the ETO algorithms send HTTP POST requests to the LLM microservice to query fitness for solutions at specific tasks. Notebly, our paradigm suits most existing ETO methods.

\subsection{Many-Task Fitness Prediction
Modeling}
\label{subsec:llm_as_surrogate}
As a proof-of-concept investigation, this paper centers on single-objective numeric MaTOP. 
The task for many-task fitness prediction is defined as
\begin{equation}
    \mathcal{T} = (\mathcal{X},\mathcal{M}, \mathcal{F}, \mathcal{D})
\end{equation}
where 
\begin{itemize}
    \item $\mathcal{X} \subset \mathbb{R}^n$ denotes the unified decision space,
    \item $\mathcal{M}$ denotes a metadata space with intrinsic descriptions to distinguish different {tasks},
    \item $\mathcal{F}=\left\{f_m: \mathcal{X} \rightarrow \mathbb{R} \mid m \in \mathcal{M}\right\}$ is the set of ground-truth objective functions, 
    \item $\mathcal{D}=\left\{\left(m_i, x_i, y_i\right)\right\}_{i=1}^N$ is a dataset {that} contains cross-objective evaluations adhering to $y_i=f_{m_i}\left(x_i\right)$.
\end{itemize}
The construction of the dataset $\mathcal{D}$ involves aggregating evaluations across diverse tasks. For any task $k$ (where $k \in\{1,2, \ldots, N T\}$), solutions $x_i^k \in \mathcal{X}_k$ are evaluated under the corresponding objective function $f_k \in \mathcal{F}$, generating triples $(m_k, x_i^k, y_i^k)$ where $y_i^k=f_k(x_i^k)$ represents the fitness value, and $m_k \in \mathcal{M}$ encodes task-specific metadata to distinguish $f_k$ from other objectives. These triples collectively form the task-specific dataset $\mathcal{D}_k=\{(m_k, x_1^k, y_1^k), \ldots,(m_k, x_{N_k}^k, y_{N_k}^k)\}$, with $N_k$ denoting the number of function evaluations for task $k$. The global cross-task dataset $\mathcal{D}$ is then derived by unifying all task-specific data:

\begin{equation}
\mathcal{D}=\bigcup_{k=1}^{N T} \mathcal{D}_k
\end{equation}
This structure enables joint modeling of heterogeneous tasks, where metadata $m_k$ explicitly resolves ambiguities between objectives (e.g., distinguishing a high-dimensional optimization landscape from a low-dimensional one) while the unified decision space $\mathcal{X}$ facilitates latent knowledge transfer. The dataset $\mathcal{D}$ thus serves as a foundation for training surrogate models to infer implicit relationships across tasks.

Our target is to learn a meta-surrogate model as
\begin{equation}
\hat{f}_{\rm{meta}}: \mathcal{X} \times \mathcal{M} \rightarrow \mathbb{R}, \quad (x, m) \mapsto \hat{y} = \hat{f}_{\rm{meta}}(x, m)
\end{equation}
which processes the decision values $x$ and problem metadata $m$ as input, and output fitness $\hat{y}$ for the specific {task}.

In this paper, we implement the meta-surrogate via fine-tuning an encoder-decoder LLM, where the model's conditional probability formulation $p(y \mid m, x)$ directly parameterizes the fitness prediction $\hat{f}_{\text {meta }}(m, x)$. This allows us to learn universal weights $\theta$ that construct an adaptive predictor generalizable across the task space $\mathcal{T}$. Unlike traditional surrogates requiring per-task basis function engineering, our language model-based surrogate achieves three fundamental advantages: (1) unified representation learning across task metadata and decision variables, (2) elimination of task-specific model retraining, and (3) exponential scaling of data efficiency via pretrained priors.

\subsection{Textual Representation}
\label{subsec:represent}

This subsection introduces textual representations for the fundamental elements in the dataset: numeric variables ($x$ and $y$) and task descriptors ($m$).

\subsubsection{Numeric Representation}
The fundamental challenge in adapting LLMs as meta-surrogates lies in establishing scale-invariant representations for continuous decision variables $x$ and fitness values $y$. Naive string conversions induce spurious ordinal relationships (e.g., 3.11 > 3.9 in lexical comparison), while traditional normalization schemes risk task-specific bias. To address this, we utilize a Scientific Notation Encoding (SNE): each floating-point number $z$ undergoes deterministic conversion to:

\begin{equation}
  \phi(z)=[ \pm]\left\langle 10^k\right\rangle d_1 d_2 \ldots d_{\gamma}.  
\end{equation}
where $[\pm]$ is the sign bit, $k$ denotes the exponent of the most significant digit $d_1$, and $d_1 d_2 \ldots d_\gamma$ represent $\gamma$ mantissa digits. Notably, to save on tokens, we omit the decimal point following the first mantissa digit. 

\begin{table}[t]
\centering
\caption{Examples of text representations in language models}
\resizebox{0.9\textwidth}{!}{%
    {
            \begin{tabular}{|c|p{0.5\textwidth}|p{0.5\textwidth}|}
        \hline
         Data & Origin data & LLM-Ready Text\\ \hline
        $m$ & BBOB, instance=0, Sphere, dimension=4
        &  
        You are a many-task surrogate model, predict fitness given m and pop;
        \texttt{function name is \{Sphere\}, function ID is \{F1\}, key feature 1 is \{BBOB\} | key feature 2 is \{instance=0\}|$\cdots$|the dimensionality is dim=\{4\}.}  \\ \hline
        $x$ & [-2.065349139, -2.570456278, -3.38108745, -3.38108745] & 
        [- \text{<10\textasciicircum 0>} \text{2} \text{0} \text{6} \text{5} \text{3} \text{4} \text{9} \text{1} \text{3} \text{9}, 
         - \text{<10\textasciicircum 0>} \text{2} \text{5} \text{7} \text{0} \text{4} \text{5} \text{6} \text{2} \text{7} \text{8}, 
         - \text{<10\textasciicircum 0>} \text{3} \text{3} \text{8} \text{1} \text{0} \text{8} \text{7} \text{4} \text{5} \text{0}, 
         + \text{<10\textasciicircum 0>} \text{4} \text{4} \text{1} \text{2} \text{2} \text{6} \text{5} \text{2} \text{3} \text{9}] \\ \hline
        $y$ & 1740.050843 & 
        [+ \text{<10\textasciicircum 3>} \text{1} \text{7} \text{4} \text{0} \text{0} \text{5} \text{0} \text{8} \text{4} \text{3}] \\ \hline
    \end{tabular}%
    }
}
\label{tab:Textual_Representation}
\end{table}
As demonstrated in Table~\ref{tab:Textual_Representation}, the first dimension of variable $x$, \texttt{"-2.065349139"},  is converted into \texttt{"- <10\textasciicircum0> 2 0 6 5 3 4 9 1 3 9"}. 
We insert spaces between digits
and symbols in the scientific notation strings, ensuring that each digit or symbol is tokenized individually. Multiple variables are represented as an SNE array. The objective function value is represented in the same format, ensuring consistency between the input and the output.

Although the SNE addresses lexical ambiguity in numerical forms, its success heavily relies on how the exponential components are tokenized. Note that, standard numbers (0-9) and signs (+/-) usually exist in the LLM's pre-existing vocabulary, but the exponent parts inside angle brackets (e.g., <10\textasciicircum3>) do not. We suggest a hybrid strategy where only exponent parts within angle brackets (<10\textasciicircum$k$>) are created as new tokens, while keeping the digits (0-9) and basic symbols (+/-) from the original vocabulary. These new tokens for exponents have an initial random setup but are optimized with other parameters, creating a balance between adapting to the domain and preserving existing knowledge. 

We formalize the numeric scheme. For any $z\in\mathbb{R}$, let $s=\mathrm{sign}(z)\in\{+1,-1,0\}$ and $u=|z|$. When $u=0$ we encode a canonical zero; when $u>0$, define $k=\lfloor \log_{10} u \rfloor$ and write $u=m\cdot 10^k$ with $m\in[1,10)$. SNE keeps $\gamma$ significant digits by rounding the mantissa to $\tilde m=\mathrm{round}_\gamma(m)$ and forms $\tilde z=s\,\tilde m\,10^{k}$. This decomposition is unique because every positive $u$ admits a unique pair $(m,k)$ with $10^{k}\le u<10^{k+1}$.

\begin{proposition}[Error bounds]\label{prop:err}
Let $\tilde m$ be $m$ rounded to $\gamma$ significant digits. Then
\begin{subequations}\label{eq:err-bounds}
\begin{equation}\label{eq:abs-err}
|z-\tilde z| \;=\; 10^k\,|m-\tilde m| \;\le\; \tfrac{1}{2}\,10^{k+1-\gamma},
\end{equation}
\begin{equation}\label{eq:rel-err}
\frac{|z-\tilde z|}{|z|} \;=\; \frac{|m-\tilde m|}{m} \;\le\; \tfrac{1}{2}\cdot 10^{\,1-\gamma}.
\end{equation}
\end{subequations}
\end{proposition}
\begin{proof}
The last significant digit has step size $10^{-(\gamma-1)}$, hence rounding incurs at most half a step,
\begin{equation}\label{eq:round-step}
|m-\tilde m|\;\le\; \tfrac{1}{2}\cdot 10^{-(\gamma-1)} \;=\; \tfrac{1}{2}\cdot 10^{\,1-\gamma}.
\end{equation}
Absolute error \eqref{eq:abs-err} follows since $z$ and $\tilde z$ share the same exponent,
\begin{equation}\label{eq:abs-err-deriv}
|z-\tilde z| \;=\; |s\,m\,10^{k}-s\,\tilde m\,10^{k}| \;=\; 10^{k}\,|m-\tilde m| \;\le\; \tfrac{1}{2}\,10^{k+1-\gamma}.
\end{equation}
For the relative error \eqref{eq:rel-err}, use $|z|=m\,10^{k}$ and $m\ge 1$,
\begin{equation}\label{eq:rel-err-deriv}
\frac{|z-\tilde z|}{|z|} \;=\; \frac{|m-\tilde m|}{m}
\;\le\; \frac{\tfrac{1}{2}\,10^{\,1-\gamma}}{m}
\;\le\; \tfrac{1}{2}\cdot 10^{\,1-\gamma}.
\end{equation}
\end{proof}

\begin{example}
For $z=-2.065349139\times 10^{0}$ ($k=0$, $m=2.065349139$) and $\gamma=4$, the step is $10^{-(4-1)}=10^{-3}$, hence $|m-\tilde m|\le 5\times 10^{-4}$, giving $|z-\tilde z|\le 5\times 10^{-4}$ and worst-case relative bound $5\times 10^{-4}$; using the actual $m$ yields $(5\times 10^{-4})/2.065\approx 2.42\times 10^{-4}$.
\end{example}

In the decimal SNE with $\gamma$ significant digits and exponent $k$, the spacing between two adjacent representable numbers near magnitude $10^k$ equals the unit in last place (ULP),
\begin{equation}\label{eq:ulp-def}
\mathrm{ULP}(k,\gamma)=10^{k+1-\gamma}.
\end{equation}
Indeed, adjacent mantissae differ by $10^{-(\gamma-1)}$, and multiplying by $10^k$ yields \eqref{eq:ulp-def}. Rounding to the nearest representable value therefore incurs at most half an ULP in absolute error,
\begin{equation}\label{eq:half-ulp}
|z-\tilde z|\le \tfrac{1}{2}\,\mathrm{ULP}(k,\gamma)=\tfrac{1}{2}\,10^{k+1-\gamma}.
\end{equation}
Since $|z|=m\,10^{k}$ with $m\in[1,10)$, the worst-case relative error is scale-invariant:
\begin{equation}\label{eq:rel-ulp}
\frac{|z-\tilde z|}{|z|}\le \tfrac{1}{2}\cdot 10^{1-\gamma}.
\end{equation}

\begin{example}
With $\gamma=4$ and $k=0$, $\mathrm{ULP}=10^{-3}$, hence $|z-\tilde z|\le 5\times 10^{-4}$; for $k=6$, $\mathrm{ULP}=10^{3}$ and the bound scales accordingly, while \eqref{eq:rel-ulp} remains unchanged.
\end{example}

\begin{proposition}[Representable range]\label{prop:range}
Assume the exponent tokens cover $k\in K=[k_{\min},k_{\max}]$ and $\gamma$ significant digits are retained. Then every nonzero representable magnitude lies in
\begin{equation}\label{eq:range}
|z| \;\in\; \big[\, 1\cdot 10^{k_{\min}},\; (10-10^{\,1-\gamma})\cdot 10^{k_{\max}} \,\big].
\end{equation}
\end{proposition}
\begin{proof}
The smallest positive value is $m=1$ at $k_{\min}$, giving $1\cdot 10^{k_{\min}}$. The largest mantissa with $\gamma$ digits is $$9.\underbrace{99\ldots 9}_{\gamma-1}=10-10^{1-\gamma}.$$Multiplying by $10^{k_{\max}}$ gives the upper endpoint. Including the sign makes the range symmetric; $z=0$ is exactly representable.
\end{proof}

For any fixed $k$, the number of distinct mantissae is $9\cdot 10^{\gamma-1}$ (from $1.00\ldots 0$ to $9.99\ldots 9$). Across $|K|=k_{\max}-k_{\min}+1$ exponents and with both signs and zero, this yields an estimate of overall coverage density.

\begin{example}
    With $\gamma=4$ and $k\in[-6,6]$, the magnitude range is $[10^{-6},\ 9.999\times 10^{6}]$.
\end{example}

\begin{proposition}[Token-length linearity and dimension budget]\label{prop:linear}
Let a single scalar be encoded by sign, exponent, and $\gamma$ digits. Then
\begin{subequations}\label{eq:linear}
\begin{equation}\label{eq:lnum}
\ell_{\text{num}} \;=\; \gamma+2,
\end{equation}
\begin{equation}\label{eq:larray}
L_{\text{array}} \;=\; D(\gamma+2) + (D-1) + 2 \;=\; D(\gamma+3)+1,
\end{equation}
\begin{equation}\label{eq:lin}
L_{\text{in}} \;=\; L_m + D(\gamma+3) + 1,\qquad
D_{\max} \;=\; \Big\lfloor \frac{L_{\max}-L_m-1}{\gamma+3} \Big\rfloor ,
\end{equation}
\end{subequations}
where $D$ is the dimensionality, $L_m$ is the metadata token length, and $L_{\max}$ is the model input cap.
\end{proposition}
\begin{proof}
Equation \eqref{eq:lnum} counts one sign, one exponent, and $\gamma$ digits. For a length-$D$ array, there are $D$ scalars, $(D-1)$ commas, and two brackets, which gives \eqref{eq:larray}. Appending metadata of length $L_m$ yields \eqref{eq:lin}. The expression for $D_{\max}$ follows by solving $L_{\text{in}}\le L_{\max}$ for $D$ and taking the floor.
\end{proof}

\begin{example}
    For $L_{\max}=400$, $\gamma\in\{3,4,5\}$, and $L_m\in\{80,120,160\}$, the resulting $D_{\max}$ values are summarized in Table~\ref{tab:Dmax400} (computed via $D_{\max}=\lfloor (L_{\max}-L_m-1)/(\gamma+3)\rfloor$).
\end{example}

\begin{table}[htbp]
\centering
\caption{
$D_{\max}$ under $L_{\max}=400$ for varying $\gamma$ and metadata length $L_m$.}
\label{tab:Dmax400}
\resizebox{0.9\textwidth}{!}{%
    {
        \begin{tabular}{|l|c|c|c|}
            \hline
             & $L_m=80$ & $L_m=120$ & $L_m=160$ \\
            \hline
            $\gamma=3$ \,($\gamma+3=6$) & $\lfloor(400-80-1)/6\rfloor={53}$ & $\lfloor(400-120-1)/6\rfloor={46}$ & $\lfloor(400-160-1)/6\rfloor={39}$ \\ \hline
            $\gamma=4$ \,($\gamma+3=7$) & $\lfloor(400-80-1)/7\rfloor={45}$ & $\lfloor(400-120-1)/7\rfloor={39}$ & $\lfloor(400-160-1)/7\rfloor={34}$ \\\hline
            $\gamma=5$ \,($\gamma+3=8$) &  $\lfloor(400-80-1)/8\rfloor={39}$ & $\lfloor(400-120-1)/8\rfloor={34}$ & $\lfloor(400-160-1)/8\rfloor={29}$\\
            \hline
        \end{tabular}
    }
}
\end{table}

Larger \(\gamma\) improves precision but increases per-dimension token cost, thus reducing \(D_{\max}\) under the same \(L_{\max}\) and \(L_m\). In practice, one may first estimate the empirical distribution of exponents to choose \(K=[k_{\min},k_{\max}]\) and \(\gamma\); when \(D\) is large under a tight \(L_{\max}\), either compressing \(L_m\) (e.g., removing redundant fields or using shorter learned tags) or reducing \(\gamma\) achieves a controllable trade-off. 

\subsubsection{Metadata Representation}
\label{subsec:metadata_repr}

To enable cross--task generalisation, the textual metadata $m$ must convey \emph{what} is being optimised (function identity) as well as \emph{how} it is parameterised (dimensionality and other instance–specific traits). We designed the following structured template for the metadata as \autoref{fig:metadata}.
\begin{figure}[htbp]
      \centering
    \begin{promptmini}[colframe=green!60!black, colback=green!5]{metadata}
    \label{prompt_Small}
    \footnotesize
    You are a many-task surrogate model, predict fitness given m and pop;
    \\
    function name is \{$f_{\mathrm{name}}$\},\;
    function ID is \{$f_{\mathrm{ID}}$\},\;
    key feature 1 is \{$k_1$\}\,|\,key feature 2 is \{$k_2$\}\,|\,\dots\,|\,%
    the dimensionality is dim=\{$D$\}.
    \end{promptmini}
  \caption{
  Structured template for the metadata.}
  \label{fig:metadata}
\end{figure}

The tuple $(f_{\mathrm{name}},f_{\mathrm{ID}},k_1,\dots,k_L,D)$ is rendered as a single, self-contained sentence whose fields are separated by the vertical bar ``|'' to minimise syntactic ambiguity. The \textit{function ID} ($F\!1$–-$F\!24$ in the BBOB suite) serves as an unambiguous handle for functions whose names may overlap across libraries, while the free-form \textit{key features} (e.g. \textit{``BBOB''}, \textit{``instance=0''}) allow arbitrary, domain-specific descriptors to be appended without changing the grammar. The ellipsis in the prompt indicates that additional tokens can be inserted transparently, making the scheme extensible to real-world scenarios that require richer contextual cues (noise level, budget category, simulator version, \emph{etc.}).

Each word, digit, or symbol is separated by a single space before submitting the sequence to the tokenizer, ensuring a one-token–per-character granularity identical to the numerical encoding
of decision variables and fitness values.  In contrast to earlier studies that create bespoke tokens for entire metadata strings, we reuse the LLM's pretrained vocabulary wherever possible
and introduce \emph{at most}~$L$ new tokens for domain-specific keywords (e.g.\ ``instance''), thereby preserving semantic consistency with the base model.

Table~\ref{tab:Textual_Representation} illustrates the conversion pipeline. For the \textit{Sphere} function (F1) in four dimensions, the original CSV entry 
\lstset{breaklines=true, basicstyle=\ttfamily}
\lstinline|BBOB, instance=0, Sphere, dimension=4|
is transformed into the prompt shown in the third column, whose compact yet compositional format allows the encoder to localise task semantics with negligible prompt overhead. During inference, the metadata string is concatenated with the scientific-notation array of the decision vector $x$ and fed into the encoder–decoder LLM, which decodes the fitness $y$ in the same textual format.  This unified representation of $(m,x,y)$ enables a single conditional language model $p_\theta(y\mid m,x)$ to learn joint embeddings for heterogeneous tasks without per-task retraining.

The revised metadata representation thus provides a scalable, machine-readable interface between heterogeneous optimisation tasks and the language model, while the consistent token-level granularity across $m$, $x$, and $y$ preserves the pretrained token semantics that underpin
effective transfer learning in the surrogate.

Finally, as shown in Table~\ref{tab:Textual_Representation}, the algorithm sends the problem name and dimensionality information as metadata $m$, along with the original data of an individual's decision variables $x$, to the LLM. The LLM converts these into a textual representation and concatenates them into its input as \autoref{fig:metadata_example}.

\begin{figure}[htbp]
  \centering
  \begin{promptmini}[colframe=green!60!black, colback=green!5]{metadata}
  \footnotesize
      You are a many-task surrogate model, predict fitness given m and pop;\\
  function name is Sphere,\;
  function ID is F1,\;
  key feature 1 is BBOB \,|\, key feature 2 is instance=0 \,|\, \(\cdots\) \,|\, the dimensionality is \(\mathrm{dim}=4\);
  \(x=\)\,[-\,<10\textasciicircum0>\,2\,0\,6\,5\,3\,4\,9\,1\,3\,9,\;
          -\,<10\textasciicircum0>\,2\,5\,7\,0\,4\,5\,6\,2\,7\,8,\;
          -\,<10\textasciicircum0>\,3\,3\,8\,1\,0\,8\,7\,4\,5\,0,\;
          +\,<10\textasciicircum0>\,4\,4\,1\,2\,2\,6\,5\,2\,3\,9]
  \end{promptmini}
  \caption{{Input example: metadata \(m\) (Sphere/F1, BBOB, instance=0, \(D{=}4\)) concatenated with the scientific-notation tokenization of the decision vector \(x\); the resulting sequence is provided to the LLM to predict \(y\) via \(p_{\theta}(y\mid m,x)\).}}
  \label{fig:metadata_example}
\end{figure}


The LLM then infers the string \texttt{"[+ <10\textasciicircum3> 1 7 4 0 0 5 0 8 4 3]"}, which is subsequently converted to the floating-point number 1740.050843 by a post-processing module. This surrogate fitness 1740.050843 is then returned to the algorithm. An illustrative example is shown in Figure~\ref{fig:dy}.

\begin{figure}[htbp]
    \centering
    \includegraphics[width=0.6\linewidth]{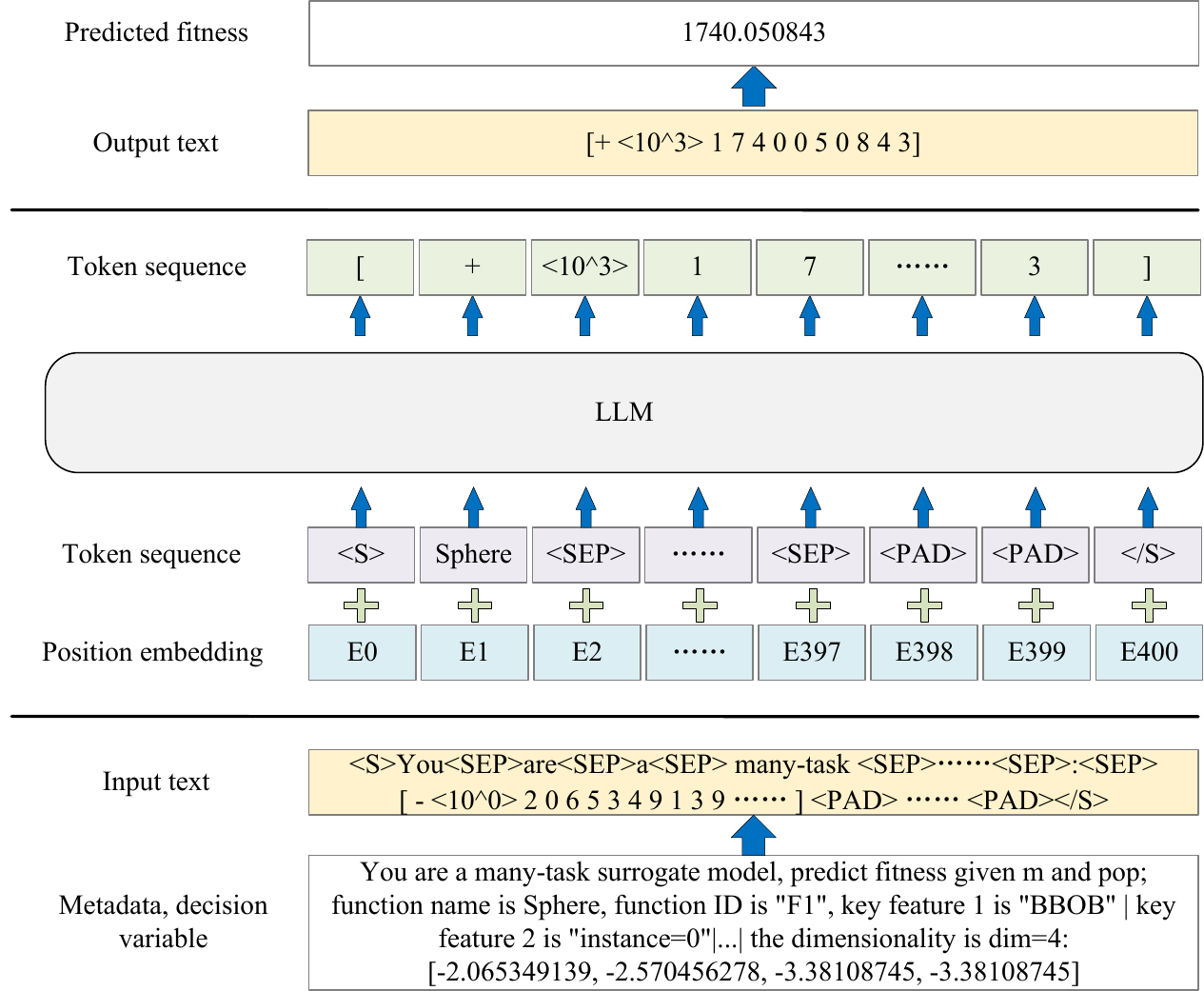}
    \caption{Prediction of decision variables and objective values based on LLM.}
    \label{fig:dy}
\end{figure}

\subsection{LLM Inference}
\label{subsec:finetune}

The formal concept of ``language model'' is rooted in conditional probability. Each sample (sentence or document) consists of a variable-length sequence of tokens $({s}_{1}$, ${s}_{2}$, $\ldots$ , ${s}_{n})$. Owing to the intrinsic sequential nature of language, the joint probability of the token sequence is typically factorized into a product of autoregressive conditional probabilities~\cite{NLM}:
\begin{equation}
    p(x) = \prod_{i=1}^{n} p({s}_{i} \mid {s}_{1},\ldots ,{s}_{i-1})
\end{equation}
This formulation allows generating tokens step-by-step by iteratively drawing $s_i \sim$ $p(s_i \mid$ $ s_{<i})$, as well as computing probabilities for arbitrary subsequences, e.g., 
$p({s}_{n-a}$, $\ldots$ ,${s}_{n}$ $\mid$ ${s}_{1}$,$\ldots$ ,${s}_{n-a-1})$. 

This probabilistic framework inherently supports task-specific learning. For a single task, the model estimates $p(\text{output} \mid \text{input})$ by treating the output tokens as a sequence conditioned on the input tokens. However, a practical LLM system must handle multiple tasks over the same input. This requires extending the conditioning to explicitly include task identity: $p(\text{output} \mid \text{task}, \text{input})$.
Following this, our meta-surrogate conditions the output $y$ on the metadata token sequence $m$ and the input token sequence $x$ as
\begin{equation}
    p(y) = p(y \mid m,x)
\end{equation}

Defined in ~\autoref{subsec:llm_as_surrogate}, a data sample in the training set is given as $(m_k, x^k_i, y^k_i)$. Then, according to ~\autoref{subsec:represent}, $m_k$ is a text string, and the floating-point values $x^k_i$ and $y^k_i$ can be converted into text strings, which are denoted as $src^k$ and $trg^k$, respectively.
In the encoder-decoder architecture, the encoder and decoder typically share the same embedding matrix $\mathbf{E} \in \mathbb{R}^{|V| \times dm}$, where $|V|$ is the vocabulary size and $dm$ is the model's hidden dimension. The tokenizer in the encoder module converts the metadata $m_k$ and $src^k$ into token ID sequences, i.e., $m_k \mapsto \{mt_1, mt_2, \dots, mt_u\}$ and $src^k \mapsto \{xt_1, xt_2, \dots, xt_n\}$, where each $mt_i$ or $xt_i$ is an index in the LLM's vocabulary. The LLM then retrieves the corresponding learnable vector from the shared embedding matrix:
\begin{equation}
\begin{split}
    \mathbf{M}^{k} = \big[\, \mathbf{E}(mt_1), \mathbf{E}(mt_2), \dots, \mathbf{E}(mt_u) \big]^\top \in \mathbb{R}^{u \times dm},
    \\
    \mathbf{W}^{k} = \big[\, \mathbf{E}(xt_1), \mathbf{E}(xt_2), \dots, \mathbf{E}(xt_n) \big]^\top \in \mathbb{R}^{n \times dm}
\end{split}
\end{equation}
where $\mathbf{E}(\cdot)$ denotes the embedding lookup. 

Then, 
 the encoder's input combines metadata and decision variable embeddings with positional encoding:
\begin{equation}
    X = PE + \begin{pmatrix}
    \mathbf{M}^{k} \\[0.8em]
    \mathbf{W}^{k}
    \end{pmatrix}
\end{equation}
where 
the positional encoding $PE$ is added ensure the sequential nature of language. The encoder of LLM then transforms the input into a latent representation $\text{Encoder}(X)$, and the decoder module generates the target token ID sequence based on this latent representation. 

\subsection{Training Objective}

The training objective is to minimize the negative log-likelihood loss of the target token ID sequence 
$trg^k \mapsto \{yt_1, yt_2, \dots, yt_n\}$:

\begin{equation}
    \mathcal{L} = -\sum_{i=1}^{n} \log P(yt_i \mid yt_{<i}, \text{Encoder}(X))
\end{equation}
where $P(yt_i | yt_{<i}, \text{Encoder}(X))$ denotes the probability of generating the token $yt_i$ by the decoder, given the $X$ and the previously generated target sequence $yt_{<i}$.

The cross-entropy loss compares the decoder's probability distribution $P(yt_i \mid yt_{<i}$, $\text{Encoder}(X))$ at each position $i$ with the true distribution of the target token IDs. For an output ID sequence $\{gt_1, gt_2, \ldots, gt_n\}$, the cross-entropy loss is computed as

\begin{equation}
\mathrm{CE}=-\sum_{i=1}^n \sum_{v=1}^{|V|} gt_i^{(v)} \cdot \log P\left(y t_i^{(v)}\right)
\end{equation}
where $|V|$ is the vocabulary size, $P(y t_i^{(v)})$ is the model's predicted probability for vacabular index $v$ at output position $i$, $gt_i^{(v)} \in\{0,1\}$ is the one-hot indicator of the target token ID at the $v$-th position. 

However, the conventional fine-tuning approach for LLMs based on $\mathcal{L}$ and $\mathrm{CE}$ is not so suitable for our meta-surrogate. For example, consider a label data instance as shown in Table~\ref{tab:Textual_Representation}: \texttt{[+ <10\textasciicircum3> 1 7 4 0 0 5 0 8 4 3]}. The traditional cross-entropy computation would treat the sign token `\texttt{+}', the exponent token `\texttt{<10\textasciicircum3>}', and the last digit `\texttt{3}' as equally important. 
This contradicts our core intuition: errors in structural components (sign, exponent, first mantissa digit) should incur larger penalties than those in trailing digits, as they fundamentally alter the numerical magnitude.

To address this, we propose a \textbf{priority-aware weighted cross-entropy (PWCE)} that emphasizes critical tokens. The modified loss becomes:
\begin{equation}
\mathrm{PWCE}=-\sum_{i=1}^{n} \sum_{v=1}^{|V|} gt_i^{(v)} \cdot w_i \cdot \log P\left(y t_i^{(v)}\right)
\end{equation}
where $n=2+\gamma$ is the total sequence length, the weight $w_i$ of each position $i$ is defined as:
    \begin{equation}
        w_i=
        \begin{cases}
            2\alpha, & i = 1, 2, 3,\\[6pt]
            \displaystyle
            \max \Bigl(
            1,\;
            \alpha \;-\;(i-4)\,\frac{\alpha - 1}{\gamma - 1}
            \Bigr),
            & i \ge 4.
        \end{cases}
    \end{equation}

For $i=1$ (the sign bit), $i=2$ (the exponent), and $i=3$ (the first mantissa digit), the weight is $2\alpha$. Starting from the second mantissa digit ($i=4$) to the end of the mantissa ($i = 2+\gamma$), the weight linearly decays from $\alpha$ to 1; if it decays to 1 before reaching the end, subsequent mantissa digits are assigned a weight of 1. For example, consider the label data \texttt{[+ <10\textasciicircum3> 1 7 4 0 0 5 0 8 4 3]}; when $\alpha=10$, the weights for the sign token \texttt{+}, the exponent token \texttt{<10\textasciicircum3>}, and the first mantissa digit are 20, while the second mantissa digit has a weight of 10, and the weights for the remaining digits decay to 1.  
Please note that \texttt{[} and \texttt{]} are the start and end markers of the token sequence for the fitness value, respectively, and should be removed when calculating $\mathrm{PWCE}$.

\subsection{Decoding Strategy}
\label{subsec:sample_decode}
Decoding refers to the process of generating text using a trained LLM. By introducing a controlled degree of randomness, decoding strategies can produce text that is both coherent and contextually appropriate, while occasionally exhibiting unexpected creative expressions. Therefore, in application scenarios that require high levels of creativity, narrative depth, and emotional nuance, an appropriate decoding strategy can significantly enhance the overall quality and fluency of the generated text, serving as an important means to improve the performance of LLMs.

Currently, various decoding strategies have been proposed in the literature, with the most common being greedy search, beam search, temperature sampling, Top-K, and Top-P (nucleus sampling) methods~\cite{LLM_servey}. It is important to note that although diversity and creativity are generally considered advantageous in text generation tasks, when LLMs are used as predictors of individual fitness in EAs, generating tokens that are overly creative may not yield the desired outcomes. This issue will be explored in depth in the experimental section.

\subsection{Integration of MetaSurrogate with ETO Algorithms}

\begin{algorithm}[htbp]\algsetup{linenosize=\tiny} \small
    \caption{MaTDE with the MetaSurrogate.}
    \label{algo:MaTDE_LLM}
    \begin{algorithmic}[1]
\REQUIRE~\\
    $im$;\COMMENT{Migration probability} \\
    \textit{aUp}; \COMMENT{Archive update probability} \\
    \textit{shk}; \COMMENT{Shrink factor} \\

\ENSURE The best solutions of all tasks
\STATE Meta-Surrogate $\gets$ Fine-tune the LLM
\STATE Randomly initialize the population for each task $\{\textit{pop}[t]\}_{t=1}^T$
\STATE Initialize the reward matrix $\textit{rew}$ with zeros
\STATE Construct the archives $\{\textit{arc}[t]\}_{t=1}^T$ to store individuals

\WHILE{termination condition is not met}
    \FOR{$t = 1$ \TO $T$}
        \IF{$\text{rand()} > im$}
            \STATE Randomly set the DE-related parameters for each individual in $\textit{pop}[t]$
            \STATE $\textit{offspring} \leftarrow \textit{generate}(\textit{pop}[t])$ \COMMENT{Generate offspring}
            \IF {each individual in $\textit{offspring}$}
            \STATE convert it into a token sequence and use the meta-surrogate to predict pseudo-fitness
            \ENDIF
            \STATE $\textit{pop}[t] \leftarrow \textit{selection}(\textit{pop}[t], \textit{offspring})$ \COMMENT{Select individuals for the next generation}
        \ELSE
            \STATE Update the probability table based on the reward matrix $\textit{rew}$
            \STATE $\textit{tfTsk}$ $\leftarrow$ Select a target task based on the probability table and the adaptive choice algorithm
            
            \FOR{each individual $i$ in $\textit{pop}[t]$}
                \STATE $\textit{offspring}[i]$ $\leftarrow$ Crossover between individual $i$ and a random individual from $\textit{pop}[\textit{tfTsk}]$
            \ENDFOR
            \STATE Evaluate $\textit{offspring}$ using meta-surrogate
            \STATE $\textit{pop}[t] \leftarrow \textit{selection}\bigl(\textit{pop}[t], \textit{offspring}\bigr)$
            
            \IF{the best solution is improved}
                \STATE $\textit{rew}[t,\ \textit{tfTsk}] \leftarrow \textit{rew}[t,\ \textit{tfTsk}] \, / \, \textit{shk}$;
            \ELSE
                \STATE $\textit{rew}[t,\ \textit{tfTsk}] \leftarrow \textit{rew}[t,\ \textit{tfTsk}] \,\times\, \textit{shk}$;
            \ENDIF
        \ENDIF

        \FOR{each individual $i$ in $\textit{pop}[t]$}
            \IF{$\text{rand()} < \textit{aUp}$}
                \STATE Insert individual $\textit{pop}[t][i]$ into $\textit{arc}[t]$ (replace randomly if archive is full)
            \ENDIF
        \ENDFOR
    \ENDFOR
\ENDWHILE

\STATE \textbf{Output:} The best solutions of all tasks

\end{algorithmic}
\end{algorithm}

Our meta-surrogate can be seamlessly incorporated into any ETO algorithm to assist FEs. 
Taking the MaTDE~\cite{MaTDE} as an example, the {pseudocode} is presented in \autoref{algo:MaTDE_LLM}. 
The meta-surrogate is utilized to predict the pseudo-fitness values (Lines 10 and 18). This method is applicable to optimization scenarios with expensive fitness evaluations or in offline data-driven settings where only sampled data points are available and no explicit evaluation metrics are provided. Notably, offline DDEAs can also be seamlessly integrated into online DDEAs: they can offer a solid starting point for online algorithms or serve as a preprocessing step for new data-driven optimization tasks, thereby enhancing overall performance and efficiency. Future research will focus on exploring more efficient strategies for online updating of the meta-surrogate, as the update process can be relatively complex when applying LLMs to online data-driven many-task optimization scenarios.

From the perspective of EAs, our integration preserves the fundamental evolutionary mechanisms---selection, variation, and inheritance---while employing the LLM-based surrogate exclusively to improve evaluation efficiency. 
Although prior work has explored leveraging LLMs to directly generate candidate solutions or even synthesize complete evolutionary strategies, our approach intentionally takes a different direction. We focus on enhancing fitness approximation through cross-task meta-learning. 
By incorporating tokenized task descriptions, our method enables flexible and effective knowledge transfer across tasks, thereby maximizing the generalization capability of the pre-trained LLM.

\section{Experiments}
\label{sec:experiments}

In this section, we delve into the four research questions raised in \autoref{subsec:motivation}. \textbf{(RQ1) Surrogate Feasibility:}
Whether LLMs, trained on tokenized representations of decision variables, fitness values, and task metadata, can serve as a meta-surrogate for cross-task fitness prediction. 
\textbf{(RQ2) Reliability \& Calibration:} Does the meta-surrogate produce landscape-faithful predictions? This question is investigated through the token-level mechanistic study, coordinate-wise response-slice analysis, and uncertainty–error correlation experiments that follow.
\textbf{(RQ3) Emergent Capability:} 
To what extent the meta-surrogate generalizes to tasks beyond its training distribution, specifically those with unseen input dimensions. 
\textbf{(RQ4) Performance Evaluation:}
How effectively the LLM surrogate enhances evolutionary search when integrated into ETO frameworks.

\subsection{Surrogate Evaluation (RQ1)}

\subsubsection{Problem Setup}
We evaluate our surrogate using the BBOB test suite~\cite{COCO} and employ the T5 model as the backbone LLM. The BBOB suite consists of 24 noise-free single-objective test functions, each with 15 test instances. We select the $0$-th test instance from each function to generate our training data. When generating the data, the search range for the decision variables is set to [-5, 5], and the considered dimensions are 5, 10, 15, and 20. In surrogate training, whereas traditional surrogates require training a separate model for each problem and dimension, the LLM adopts prefix tuning to train a single model across all tasks and dimensions. The ratio of training data to test data is $\textit{5:3}$. Each task's training dataset consists of 500 samples.

In all experiments, we employ the \emph{max--min scaled mean absolute error} (sMAE) together with the coefficient of determination ($R^{2}$)\footnote{The $R^{2}$ score is computed using the \texttt{scikit-learn} implementation; consequently, its theoretical range is $(-\infty,\,1]$.} as our primary evaluation metrics.  
To remove scale discrepancies across tasks, sMAE normalises the absolute prediction error within each task.  
Let ${I}_k^{\text{test}}$ be the index set of test samples for task~$k$;   
let $y_i^{k}$ and $\hat y_i^{k}=\hat f_{\mathrm{meta}}(x_i^{k},m_k)$ denote the ground-truth and predicted fitness, respectively;  
and define the target value range
$R_k=\max\nolimits_{j\in I_k^{\mathrm{test}}} y_j^{k}
    -\min\nolimits_{j\in I_k^{\mathrm{test}}} y_j^{k}.$

The per-task sMAE and its overall average are then defined as follows:
\begin{equation}
    \operatorname{sMAE}=\frac{1}{|{I}_k^{\text{test}}|}
    \sum\nolimits_{i\in{I}_k^{\text{test}}}
    \frac{|\,y_i^{k}-\hat y_i^{k}\,|}{R_k}
\end{equation}

\subsubsection{Training Setup}
For the LLM, our pre-trained language model uniformly uses the pre-trained weights available from Hugging Face\footnote{https://huggingface.co/}. 
{The input sequence (metadata plus numeric tokens) has 20×D tokens, and the output sequence length is 50 tokens.} The LLM is trained for 65 epochs. Our experiments were conducted on the following hardware: an Intel Xeon E5-2680 v4 (56) @ 3.3GHz, one NVIDIA GeForce RTX 4090 Ti 24G GPU, and 128GB of memory. The code implementation was done in Python using PyTorch 1.8.2.~\autoref{tab:training_config} summarizes the key hyperparameters, optimizer settings, and other training configurations used in all experiments.

\begin{table}[htbp]
    \centering
    \caption{Training configuration.}
    \label{tab:training_config}
    \resizebox{0.7\linewidth}{!}{
    \begin{tabular}{|p{0.35\linewidth}|p{0.60\linewidth}|}
    \hline
    Aspect & Setting \\ \hline
    $\gamma$ & 15 \\ \hline
    $\alpha$ & 10 \\ \hline
    Backbone & T5 family ({t5-small} if per-task training size $<1000$; otherwise {t5-base}) \\ \hline
    Max input / target length & $20\times D$ tokens (covers metadata + $x$; e.g., $D{=}50\Rightarrow 1000$, $D{=}60\Rightarrow 1200$) \,/\, 50 tokens \\ \hline
    Tokenizer & T5Tokenizer \\ \hline
    Optimizer & Adafactor, learning rate $1\times 10^{-3}$, weight decay $0$ \\ \hline
    Learning rate schedule & Constant-with-warmup, warmup ratio $0.06$ \\ \hline
    Batch size (per GPU) & {t5-small}: 24; {t5-base}: 12 \\ \hline
    Epochs & 65 \\ \hline
    Precision & FP32 \\ \hline
    Hardware & NVIDIA GeForce RTX 4090 \\ \hline
    \end{tabular}}
\end{table}

\subsubsection{Comparison Experiments of Surrogate Models}
\begin{table}[htbp]
\caption{Configurations of surrogate baselines included in our comparison. Architectural and training hyper-parameters strictly follow the official papers unless otherwise specified.}
\label{tab:surrogate_baselines}
\footnotesize
\centering
    \resizebox{1\linewidth}{!}{
    \begin{tabular}{|p{0.2\linewidth}| p{0.8\linewidth}|}
\hline
Baseline & {Model configuration} \\
\hline
SAINT~\cite{SAINT} &
Six layers with alternating interaction attention and self-attention; hidden size 256; 8 heads; numerical features are treated as tokens via a linear projection and concatenated with a CLS token; pre-normalization and residual connections; scalar regression via CLS readout. \\\hline
FT-Transformer~\cite{FTTransformer} &
rtdl configuration with 4–6 pre-norm Transformer encoder blocks; hidden size 256; 8 heads; one token per feature with linear projection, positional encoding, and a CLS token; scalar readout from CLS. \\\hline
Evidential-MLP~\cite{Evidential-MLP} &
Three-layer MLP producing NIG parameters; loss is the sum of NIG negative log-likelihood and the evidence regularizer; positivity and \(\alpha>1\) ensured via softplus/offsets. \\\hline
MTGP~\cite{MTGP} &
Intrinsic coregionalization with cross-covariance \(k((x,i),(x',j)) = k_x(x,x')\,B_{ij}\) where \(B \in \mathbb{R}^{T\times T}\) is positive semidefinite and the input kernel is Matérn–\(5/2\) with ARD,
\(k_x(x,x')=\sigma_f^2\!(1+\sqrt{5}\,r+{5/3}*r^2)\exp(-\sqrt{5}\,r)\),
\(r = (\sum\nolimits_d ((x_d-x'_d)/\ell_d)^2)^{1/2}\).
Training uses SVGP.
\\\hline
RBFN &
Single-task Gaussian radial basis network per function and per dimension; centers from k-means/++ on training inputs; basis width \(\sigma_j\) scaled by the mean distance to nearest neighbors; output weights by least squares; number of centers selected via a small validation grid. \\
\hline
\end{tabular}
    }
\end{table}

\begin{table}[htbp]
\centering
\caption{sMAE of surrogates on 20-dimensional BBOB functions.}
\label{tab:tradition_resulta}
\resizebox{0.7\linewidth}{!}{
   \begin{tabular}{|c|c|c|c|c|c|c|}
\hline
 & Meta-Surrogate & RBFN & MTGP & FT-Transformer & Evidential-MLP & SAINT \\
\hline
Attractive\_Sector & \textbf{0.062} & 0.096 & 0.099 & 2.111 & 2.104 & 2.112 \\
\hline
Bent\_Cigar & \textbf{0.022} & 0.035 & 0.029 & 0.658 & 0.658 & 0.658 \\
\hline
Buche\_Rastrigin & \textbf{0.042} & 0.082 & 0.043 & 1.614 & 1.480 & 1.615 \\
\hline
Composite\_Grie\_rosen & 0.059 & \textbf{0.059} & 0.061 & 34.526 & 2.153 & 35.620 \\
\hline
Different\_Powers & 0.032 & 0.042 & \textbf{0.025} & 16.227 & 1.662 & 16.394 \\
\hline
Discus & \textbf{0.077} & 0.129 & 0.124 & 0.744 & 0.744 & 0.744 \\
\hline
Ellipsoidal & 0.069 & 0.087 & \textbf{0.025} & 1.439 & 1.439 & 1.439 \\
\hline
Ellipsoidal\_high\_cond & 0.069 & 0.101 & \textbf{0.053} & 1.310 & 1.310 & 1.310 \\
\hline
Gallagher\_101Peaks & \textbf{0.068} & 0.072 & 0.072 & 351.233 & 26.721 & 365.661 \\
\hline
Gallagher\_21Peaks & \textbf{0.053} & 0.068 & 0.061 & 829.454 & 65.288 & 850.508 \\
\hline
Katsuura & 0.153 & \textbf{0.150} & 2.194 & 109.701 & 7.974 & 115.322 \\
\hline
Linear\_Slope & \textbf{0.045} & 0.128 & 0.065 & 39.834 & 4.871 & 40.307 \\
\hline
Lunacek\_bi\_Rastrigin & 0.084 & 0.079 & \textbf{0.073} & 20.353 & 1.592 & 20.586 \\
\hline
Rastrigin & \textbf{0.052} & 0.066 & 1.913 & 1.279 & 0.590 & 1.288 \\
\hline
Rastrigin\_F15 & \textbf{0.043} & 0.063 & 0.999 & 3.798 & 1.735 & 3.825 \\
\hline
Rosenbrock\_original & \textbf{0.043} & 0.129 & 0.074 & 2.324 & 2.317 & 2.324 \\
\hline
Rosenbrock\_rotated & 0.055 & \textbf{0.052} & 0.055 & 1.996 & 1.983 & 1.996 \\
\hline
Schaffers & 0.037 & 0.064 & \textbf{0.033} & 7.464 & 0.567 & 7.812 \\
\hline
Schaffers\_high\_cond & \textbf{0.052} & 0.076 & 0.068 & 6.832 & 0.601 & 6.941 \\
\hline
Schwefel & \textbf{0.044} & 0.120 & 0.548 & 2.136 & 2.110 & 2.136 \\
\hline
Sharp\_Ridge & 0.060 & 0.110 & \textbf{0.018} & 6.711 & 2.354 & 6.762 \\
\hline
Sphere & 0.058 & 0.114 & \textbf{0.006} & 36.044 & 5.623 & 36.440 \\
\hline
Step\_Ellipsoidal & \textbf{0.077} & 0.120 & 0.118 & 3.538 & 1.339 & 3.562 \\
\hline
Weierstrass & 0.146 & 0.143 & \textbf{0.142} & 111.862 & 9.976 & 113.153 \\
\hline
Average & 0.063 & 0.091 & 0.287 & 66.383 & 6.133 & 68.271 \\
\hline
\end{tabular}
}
\end{table}

\begin{table}[htbp]
\centering
\caption{$R^2$ of surrogates on 20-dimensional BBOB functions.}
\label{tab:tradition_resultb}
\resizebox{0.7\linewidth}{!}{
   \begin{tabular}{|c|c|c|c|c|c|c|}
\hline
 & Meta-Surrogate & RBFN & MTGP & FT-Transformer & Evidential-MLP & SAINT \\ 
\hline
Attractive\_Sector & \textbf{0.695} & 0.441 & -13.370 & -4.458 & -4.426 & -4.459 \\
\hline
Bent\_Cigar & \textbf{0.164} & 0.068 & 0.130 & 0.033 & 0.033 & 0.043 \\
\hline
Buche\_Rastrigin & \textbf{0.720} & 0.106 & 0.710 & 0.004 & 0.078 & 0.008 \\
\hline
Composite\_Grie\_rosen & 0.455 & \textbf{0.508} & 0.470 & -34.526 & -6.254 & -35.620 \\
\hline
Different\_Powers & 0.335 & 0.081 & \textbf{0.600} & -16.227 & -3.285 & -16.394 \\
\hline
Discus & \textbf{0.617} & 0.069 & -0.310 & 0.034 & 0.035 & 0.004 \\
\hline
Ellipsoidal & 0.348 & 0.124 & \textbf{0.910} & -2.071 & -2.070 & -2.071 \\
\hline
Ellipsoidal\_high\_cond & \textbf{0.496} & 0.029 & 0.470 & -1.716 & -1.716 & -1.716 \\
\hline
Gallagher\_101Peaks & -0.006 & 0.176 & \textbf{0.190} & -35.233 & -33.352 & -36.661 \\
\hline
Gallagher\_21Peaks & 0.049 & 0.101 & \textbf{0.240} & -82.454 & -80.300 & -85.508 \\
\hline
Katsuura & -0.078 & \textbf{-0.003} & -13.820 & -10.701 & -10.118 & -11.321 \\
\hline
Linear\_Slope & -0.739 & 0.076 & \textbf{0.766} & -39.834 & -33.359 & -40.307 \\
\hline
Lunacek\_bi\_Rastrigin & 0.567 & 0.626 & \textbf{0.690} & -20.353 & -2.848 & -20.586 \\
\hline
Rastrigin & \textbf{0.351} & 0.258 & -33.723 & -1.635 & -0.006 & -1.657 \\
\hline
Rastrigin\_F15 & \textbf{0.442} & 0.154 & -10.134 & -14.421 & -3.049 & -14.635 \\
\hline
Rosenbrock\_original & \textbf{0.915} & 0.261 & 0.740 & 0.053 & 0.054 & 0.053 \\
\hline
Rosenbrock\_rotated & 0.314 & \textbf{0.501} & 0.410 & -3.984 & -3.929 & -3.985 \\
\hline
Schaffers & 0.527 & 0.015 & \textbf{0.750} & -55.706 & 0.465 & -61.029 \\
\hline
Schaffers\_high\_cond & \textbf{0.305} & 0.120 & 0.260 & -46.680 & 0.018 & -48.182 \\
\hline
Schwefel & \textbf{0.889} & 0.202 & -47.150 & 0.045 & 0.046 & 0.036 \\
\hline
Sharp\_Ridge & 0.747 & 0.138 & \textbf{0.940} & -45.041 & -5.404 & -45.728 \\
\hline
Sphere & 0.831 & 0.290 & \textbf{0.990} & -36.044 & -42.083 & -36.440 \\
\hline
Step\_Ellipsoidal & \textbf{0.621} & 0.138 & -2.170 & -12.521 & -1.554 & -12.692 \\
\hline
Weierstrass & -0.099 & -0.017 & \textbf{-0.000} & -111.862 & -12.249 & -113.153 \\
\hline
\end{tabular}
}
\end{table}

We propose using {an} LLM with an encoder–decoder architecture (T5) as the meta-surrogate. For comprehensive evaluation, we consider five representative baselines~\autoref{tab:surrogate_baselines} that capture both classical and recent surrogate paradigms. The first is a conventional single-task RBFN, which remains the most widely adopted surrogate in many-task optimization algorithms despite the absence of established multi-task RBFN frameworks. The second is a state-of-the-art MTGP (i.e., the Adaptive Transfer Gaussian Process~\cite{MTGP}), implemented under the intrinsic coregionalization model (ICM) with a Matérn–\(5/2\) kernel and trained by sparse variational inference. In addition, we include three recent deep regression models designed for tabular or continuous prediction: SAINT (stacked interaction and self-attention blocks), FT-Transformer (rtdl configuration with pre-normalization and CLS readout), and an evidential regression MLP that outputs Normal–Inverse-Gamma parameters with an NLL-plus-evidence loss. Except for minimal adaptations to match dimensionality, we strictly follow the architectural and training hyper-parameters specified in the original papers. 

All methods are trained with identical data budgets of $FEs=500$ per task and consistent train/validation/test splits. The single-task RBFN is trained independently for each function–dimension pair. MTGP is trained in a joint fashion by treating the target problem as the main task and the remaining 23 BBOB problems as auxiliaries, while our proposed meta-surrogate is trained once by pooling data across all 24 problems and all dimensions. This unified protocol ensures a fair assessment of accuracy, reliability, and computational cost under comparable data and compute budgets.

{As shown in \autoref{tab:tradition_resulta} and \autoref{tab:tradition_resultb}, under the same evaluation budget the meta-surrogate achieves} lower prediction error compared to traditional surrogate models, with particularly strong performance in settings with limited evaluations. These results not only highlight the proposed method's cross-task modeling capabilities, but also demonstrate its superior data efficiency---a critical property in scenarios where historical optimization data is scarce.

\subsubsection{Investigation on Decoding Strategies}
After the LLM is trained, decoding strategies also play a crucial role. Here, we discuss how the LLM's sampling method affects the surrogate by comparing  greedy search, beam search, Top-K, Top-P, and temperature sampling methods~\cite{LLM_servey}.
We examine prediction errors for the sign, exponent, and mantissa tokens separately. Specifically, for the sign token, we compare the model's predicted sign with the true sign (treating a mismatch as incorrect) and compute the average error. For the exponent token, we extract and compare the exponent parts, computing the absolute difference and then averaging over all test samples. For the mantissa digits, we directly calculate the absolute difference between the predicted and true digits, again averaging across test samples.

\autoref{fig:sample} shows the prediction errors of the first 8 generated tokens versus the first 8 true tokens on the test set. The results suggest that at the most crucial token positions (e.g., sign and exponent), performance is similar across various sampling methods; for instance, the error at the first token is zero, indicating 100\% sign accuracy. However, for subsequent digit tokens (e.g., fifth to seventh), greedy sampling exhibits lower error rates.

We also convert the predicted fitness sequence into floating-point numbers and compute the sMAE relative to the true fitness. As seen in~\autoref{tab:sample_result}, in an offline many-task {DDEA} setting, a deterministic generation approach---namely, greedy sampling---yields more reliable pseudo-fitness estimates.

\begin{figure}[htbp]
    \centering
    \includegraphics[width=0.6\linewidth]{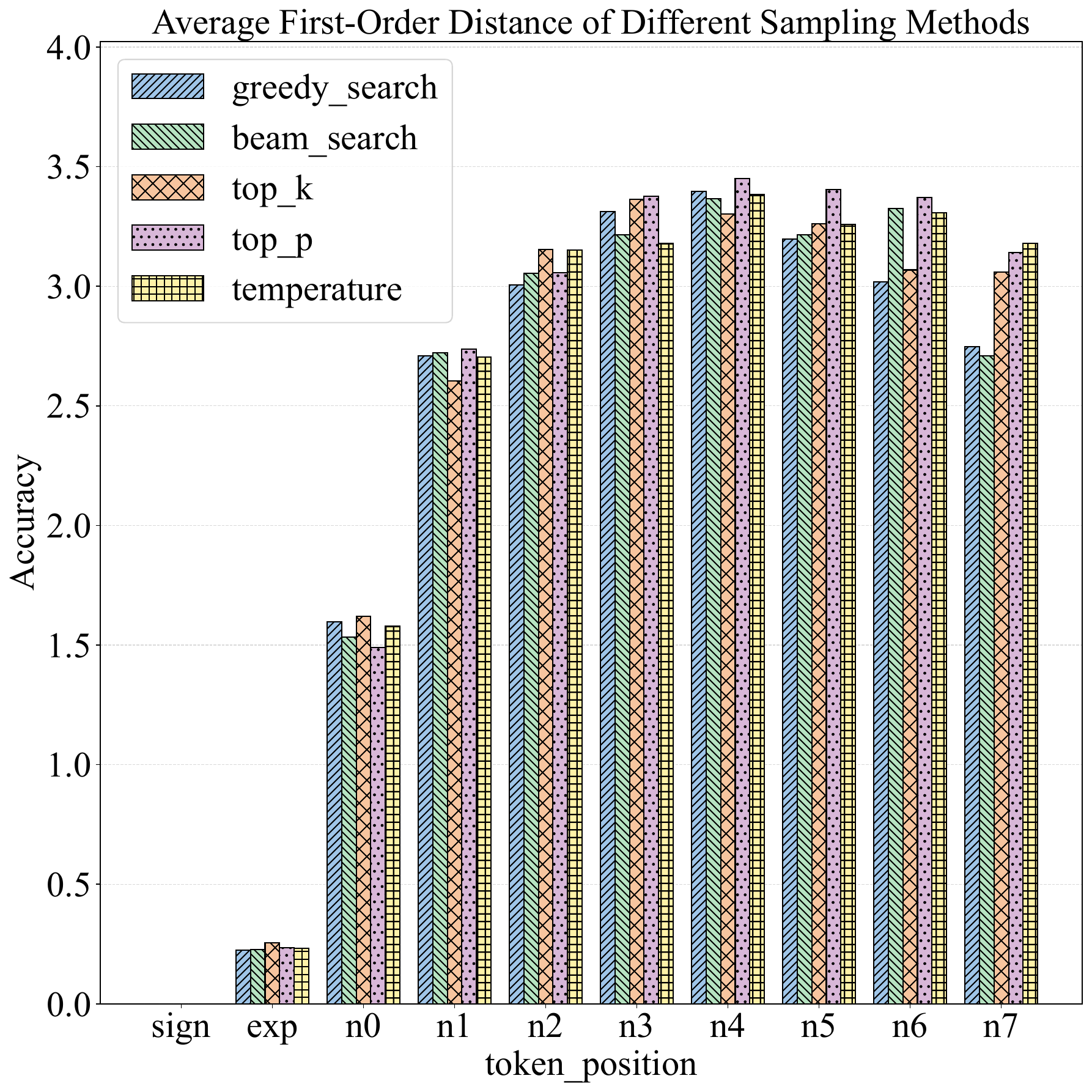}
    \caption{Prediction errors associated with different sampling methods.}
    \label{fig:sample}
\end{figure}

\begin{table}[htbp]
  \centering
  \caption{Prediction errors of different sampling methods on 20-dimensional BBOB-24 test functions.}
      \label{tab:sample_result}
  \resizebox{0.8\linewidth}{!}{
          \begin{tabular}{|c|c|c|c|c|c||c|c|c|c|c|}
\hline
 & \multicolumn{5}{|c||}{sMAE} & \multicolumn{5}{c|}{R$^2$} \\ 
 \hline
 & Greedy & Beam & Top\_K & Top\_P & temperature & Greedy & Beam & Top\_K & Top\_P & temperature \\
\hline
Attractive\_Sector      &  \textbf{0.0619} &           0.0695 &  0.0814 &  0.0778 &  0.0795 &   \textbf{0.6948} &            0.6303 &   0.5370 &           0.6006 &   0.5493 \\
\hline
Bent\_Cigar             &  \textbf{0.0222} &           0.0230 &  0.0332 &  0.0276 &  0.0276 &            0.1636 &   \textbf{0.3757} &  -0.1416 &           0.1601 &  -0.2478 \\
\hline
Buche\_Rastrigin        &  \textbf{0.0422} &           0.0428 &  0.0536 &  0.0501 &  0.0500 &            0.7201 &   \textbf{0.7241} &   0.5970 &           0.6179 &   0.6486 \\
\hline
Composite\_Grie\_rosen  &  \textbf{0.0588} &           0.0624 &  0.0818 &  0.0763 &  0.0741 &   \textbf{0.4546} &            0.4232 &   0.1498 &           0.2341 &   0.2696 \\
\hline
Different\_Powers       &  \textbf{0.0319} &           0.0324 &  0.0462 &  0.0440 &  0.0400 &   \textbf{0.3349} &            0.3055 &  -0.1300 &           0.0524 &   0.1371 \\
\hline
Discus                  &  \textbf{0.0770} &           0.0789 &  0.1080 &  0.0969 &  0.0986 &   \textbf{0.6172} &            0.6027 &   0.2357 &           0.3772 &   0.4174 \\
\hline
Ellipsoidal             &  \textbf{0.0691} &           0.0696 &  0.0921 &  0.0841 &  0.0802 &            0.3479 &   \textbf{0.3493} &   0.0115 &           0.1378 &   0.2108 \\
\hline
Ellipsoidal\_high\_cond &  \textbf{0.0689} &           0.0713 &  0.0907 &  0.0873 &  0.0795 &   \textbf{0.4962} &            0.4889 &   0.2073 &           0.2838 &   0.4150 \\
\hline
Gallagher\_101Peaks     &  \textbf{0.0679} &           0.0694 &  0.0961 &  0.0867 &  0.0809 &  \textbf{-0.0062} &           -0.0224 &  -0.4684 &          -0.2419 &  -0.1066 \\
\hline
Gallagher\_21Peaks      &  \textbf{0.0527} &           0.0539 &  0.0678 &  0.0591 &  0.0609 &            0.0487 &            0.0254 &  -0.2710 &  \textbf{0.0595} &   0.0031 \\
\hline
Katsuura                &           0.1526 &  \textbf{0.1518} &  0.2238 &  0.2105 &  0.1923 &           -0.0781 &  \textbf{-0.0619} &  -1.2109 &          -0.9470 &  -0.6847 \\
\hline
Linear\_Slope           &           0.0454 &  \textbf{0.0347} &  0.0661 &  0.0563 &  0.0528 &           -0.7389 &   \textbf{0.9240} &  -1.3500 &          -0.6758 &   0.4179 \\
\hline
Lunacek\_bi\_Rastrigin  &  \textbf{0.0837} &           0.0882 &  0.1144 &  0.1045 &  0.0995 &   \textbf{0.5671} &            0.5146 &   0.2105 &           0.3433 &   0.3715 \\
\hline
Rastrigin               &  \textbf{0.0520} &           0.0674 &  0.0732 &  0.0655 &  0.0652 &   \textbf{0.3508} &           -0.0252 &  -0.2736 &           0.0294 &  -0.0109 \\
\hline
Rastrigin\_F15          &           0.0435 &  \textbf{0.0433} &  0.0505 &  0.0466 &  0.0468 &            0.4415 &            0.4502 &   0.3189 &  \textbf{0.4535} &   0.4095 \\
\hline
Rosenbrock\_original    &  \textbf{0.0425} &           0.0442 &  0.0618 &  0.0623 &  0.0553 &   \textbf{0.9146} &            0.9097 &   0.8232 &           0.8285 &   0.8637 \\
\hline
Rosenbrock\_rotated     &  \textbf{0.0546} &           0.0549 &  0.0633 &  0.0614 &  0.0590 &            0.3136 &   \textbf{0.3149} &   0.2247 &           0.3059 &   0.2813 \\
\hline
Schaffers               &           0.0371 &  \textbf{0.0353} &  0.0411 &  0.0383 &  0.0375 &            0.5266 &   \textbf{0.5565} &   0.4312 &           0.5555 &   0.5160 \\
\hline
Schaffers\_high\_cond   &  \textbf{0.0518} &           0.0534 &  0.0711 &  0.0624 &  0.0618 &   \textbf{0.3046} &            0.2913 &  -0.2232 &           0.2088 &   0.2831 \\
\hline
Schwefel                &           0.0442 &  \textbf{0.0427} &  0.0570 &  0.0547 &  0.0532 &            0.8887 &   \textbf{0.8988} &   0.8106 &           0.8272 &   0.8385 \\
\hline
Sharp\_Ridge            &           0.0602 &  \textbf{0.0596} &  0.0743 &  0.0720 &  0.0715 &            0.7467 &   \textbf{0.7484} &   0.5926 &           0.6373 &   0.6484 \\
\hline
Sphere                  &  \textbf{0.0576} &           0.0582 &  0.0746 &  0.0692 &  0.0671 &   \textbf{0.8309} &            0.8271 &   0.7031 &           0.7363 &   0.7567 \\
\hline
Step\_Ellipsoidal       &  \textbf{0.0767} &           0.0817 &  0.1060 &  0.0975 &  0.0907 &   \textbf{0.6211} &            0.5496 &   0.2534 &           0.3816 &   0.4676 \\
\hline
Weierstrass             &  \textbf{0.1457} &           0.1458 &  0.1875 &  0.1751 &  0.1799 &           -0.0993 &  \textbf{-0.0980} &  -0.7193 &          -0.4921 &  -0.5850 \\
\hline
\end{tabular}
  }
\end{table}

\subsection{Reliability \& Calibration (RQ2)}

\subsubsection{Token-Level Mechanistic Study}

\begin{figure}[t]
    \centering
    \includegraphics[width=1\linewidth]{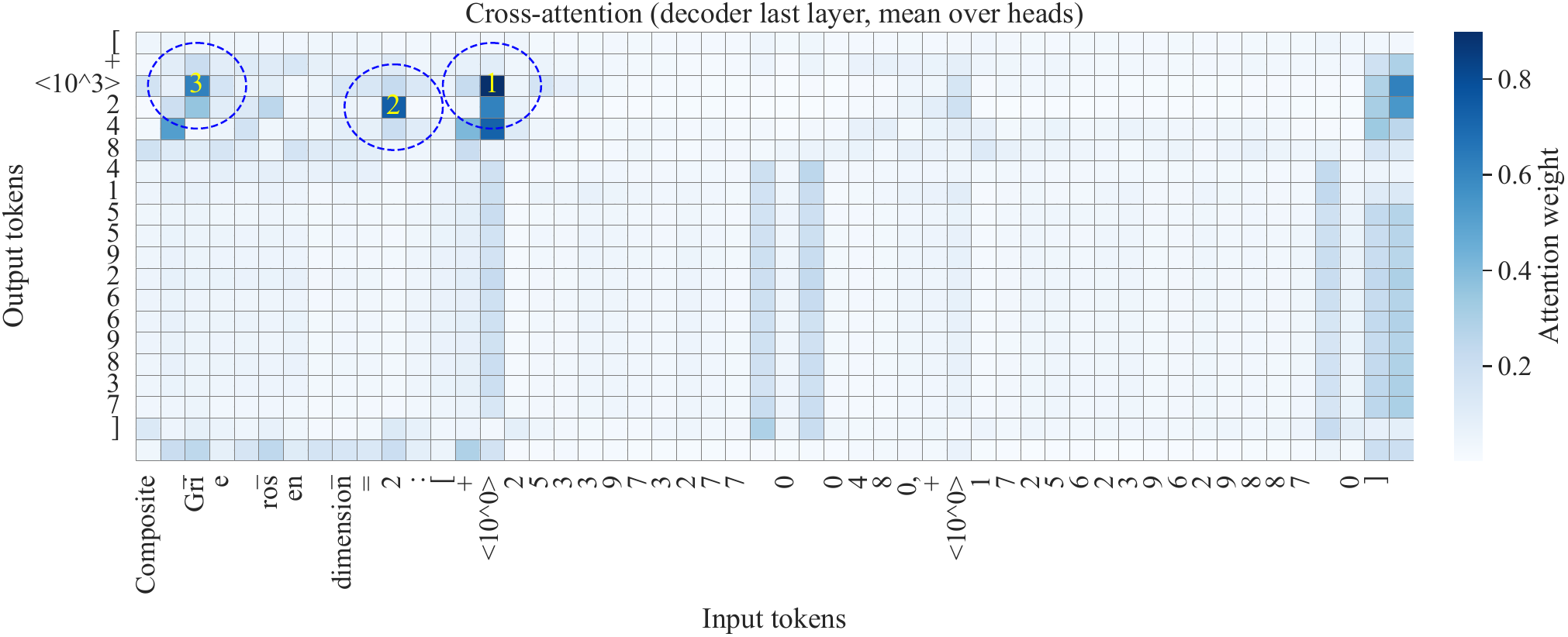}
    \caption{Decoder cross-attention heatmap for the \texttt{Composite\_Grie\_Rosen} prompt ($d=3$) at the final layer of the meta-surrogate. Color intensity reflects the average attention weight across all heads. Most of the probability mass is concentrated on numeric tokens---particularly their sign and exponent symbols---indicating that the decoder first leverages order-of-magnitude information to delimit the search region. The function name and dimensionality tokens maintain a consistent block of attention, confirming that the model explicitly queries task metadata before generating the regression output.}
    \label{fig:cross_attention_layer}
\end{figure}

\begin{figure}[t]
    \centering
    \includegraphics[width=1\linewidth]{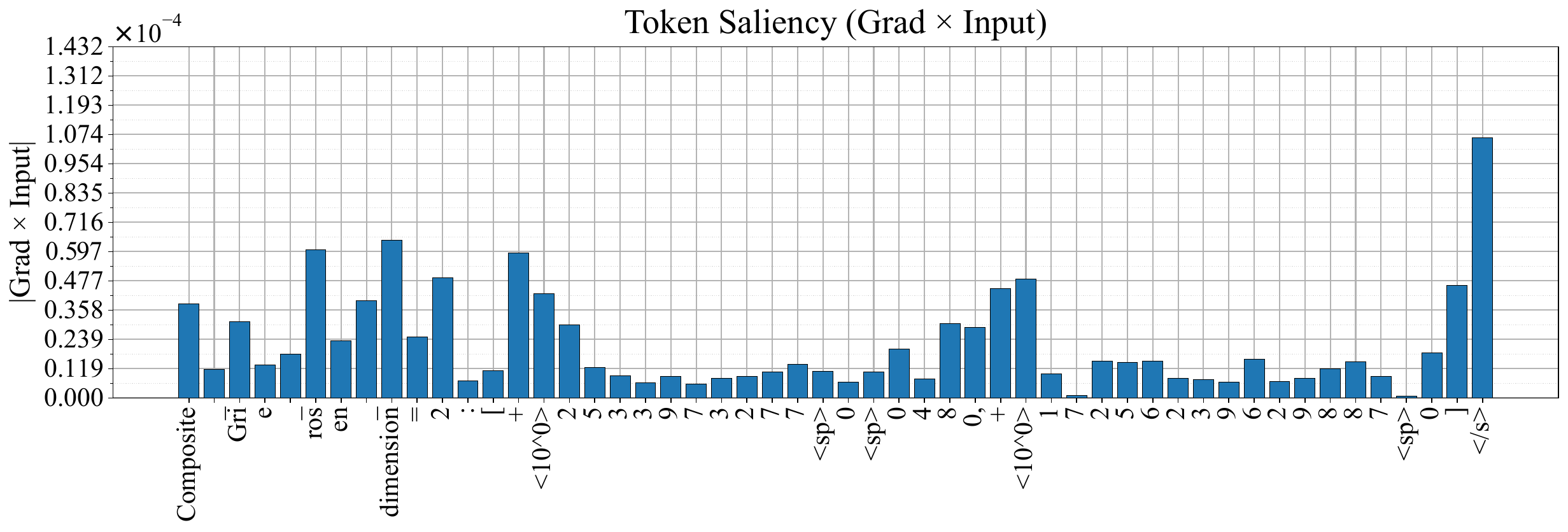}
    \caption{\emph{Grad~$\times$~Input} saliency profile for the same prompt. Each bar shows the $\ell_1$-norm of the element-wise product between the input embedding and its gradient with respect to the mean encoder representation. 
    The highest scores are observed for the task prompt, the sign token, the exponent token, and the end-of-sequence token `</s>`, followed by high-order mantissa digits; lower-order digits contribute negligibly. 
    Although no manual weighting is applied to sign or exponent positions in the input sequence, the model autonomously assigns greater gradient magnitudes to them. This indicates that it has learned to prioritize scale-related cues for numerical reasoning while still leveraging explicit task context.}
    \label{fig:Token_Saliency}
\end{figure}

To verify how {the} meta-surrogate integrates \textit{task-level semantics} with \textit{numerical decision variables} during prediction, we conducted two complementary interpretability studies---cross-attention visualisation and \emph{Grad~$\times$~Input} saliency~\cite{dientbasedattribution} on the same Composite\_\allowbreak Grie\_\allowbreak Rosen example in three dimensions. 

The cross-attention heatmap ~\autoref{fig:cross_attention_layer}
shows that, when the decoder generates the predicted objective value, the overwhelming majority of attention mass is assigned to numeric tokens, with a marked concentration on sign and exponent symbols. 
This indicates that the model initially extracts order-of-magnitude cues to efficiently localize the fitness region. Meanwhile, task-metadata tokens---{specifically, the function name and dimensionality}---consistently retain a stable share of attention, suggesting that the decoder explicitly leverages task context when forming the regression mapping, thereby enabling conditional inference.

To quantify the contribution of each input token to the hidden representation, we backpropagate the gradient of the mean encoder output $\bar{\mathbf{h}}$ with respect to each input embedding $\mathbf{e}_j$, and compute:
\begin{equation}
S_j =
\left\lVert
\frac{\partial\bar{\mathbf h}}{\partial\mathbf e_j} \odot \mathbf e_j
\right\rVert_1,
\end{equation}
where $\mathbf e_j$ is the embedding of the $j$-th input token and $\odot$ denotes the Hadamard product. The saliency ranking ~\autoref{fig:Token_Saliency} reveals that the function name and the task prompt occupy the top positions; sign and exponent tokens obtain the largest $S_j$ values, followed by the higher-order mantissa digits, whereas lower-order digits contribute negligibly. Importantly, the encoding scheme of input sequence assigns no explicit weight to sign or high-order digits, yet the model still allocates disproportionately high gradients to them, indicating that it has autonomously learned to centre its numerical reasoning on order-of-magnitude information.

Taken together, the two lines of evidence expose a hierarchical computation inside the model: it first selects an appropriate latent {subspace} via the task tokens, subsequently captures coarse-scale information through sign and exponent symbols, and finally refines the estimate with a small subset of high-order mantissa digits. This mechanism explains the model's robust generalisation across tasks and in zero-shot settings, and, at a microscopic level, validates the representational efficiency and interpretability of the ``symbol $+$ scientific-notation'' encoding adopted for high-dimensional optimisation.

In summary, the meta-surrogate is able to process symbolic task context and high-precision numerical information within a single input sequence, combining explicit attention to task semantics with implicit extraction of numeric scale to model complex many-task fitness functions accurately. These findings provide theoretical support for the trustworthy application of LLMs in DDEAs.

\subsubsection{Coordinate-Wise Response-Slice Analysis}
\label{subsec:slice}
To visualise how each surrogate behaves on local response surfaces, we
slide a single coordinate while fixing the remaining dimensions.
Although this procedure was applied to the complete set of
24 five-dimensional BBOB functions, for most problems the curves produced
by the {meta-surrogate} and the {RBFN} are virtually
indistinguishable from the ground truth, whereas the
{MTGP} deviates so markedly that the plots offer little
additional insight.  
We therefore report results for four representative function.
For each function we linearly sweep every coordinate
\(x_{d}\in[-5,5]\;(d=1,\dots,5)\) from a randomly selected reference
point \(\mathbf{x}_{0}\), generating five one-dimensional slices per
function and a total of {20 response-slice plots}.
These slices allow a direct visual comparison of the ground-truth curve
\(f(\mathbf{x})\) with the prediction curves
\(\hat f(\mathbf{x})\) produced by the three surrogate models. \autoref{fig:Slice} {illustrates} the close agreement of T5 and RBFN with the true responses and the comparatively larger errors of MTGP.

\begin{figure*}[htbp]
    \centering
        \subfloat[\label{fig:offline-curves-ackley-100}]{
        \includegraphics[width=0.19\textwidth]{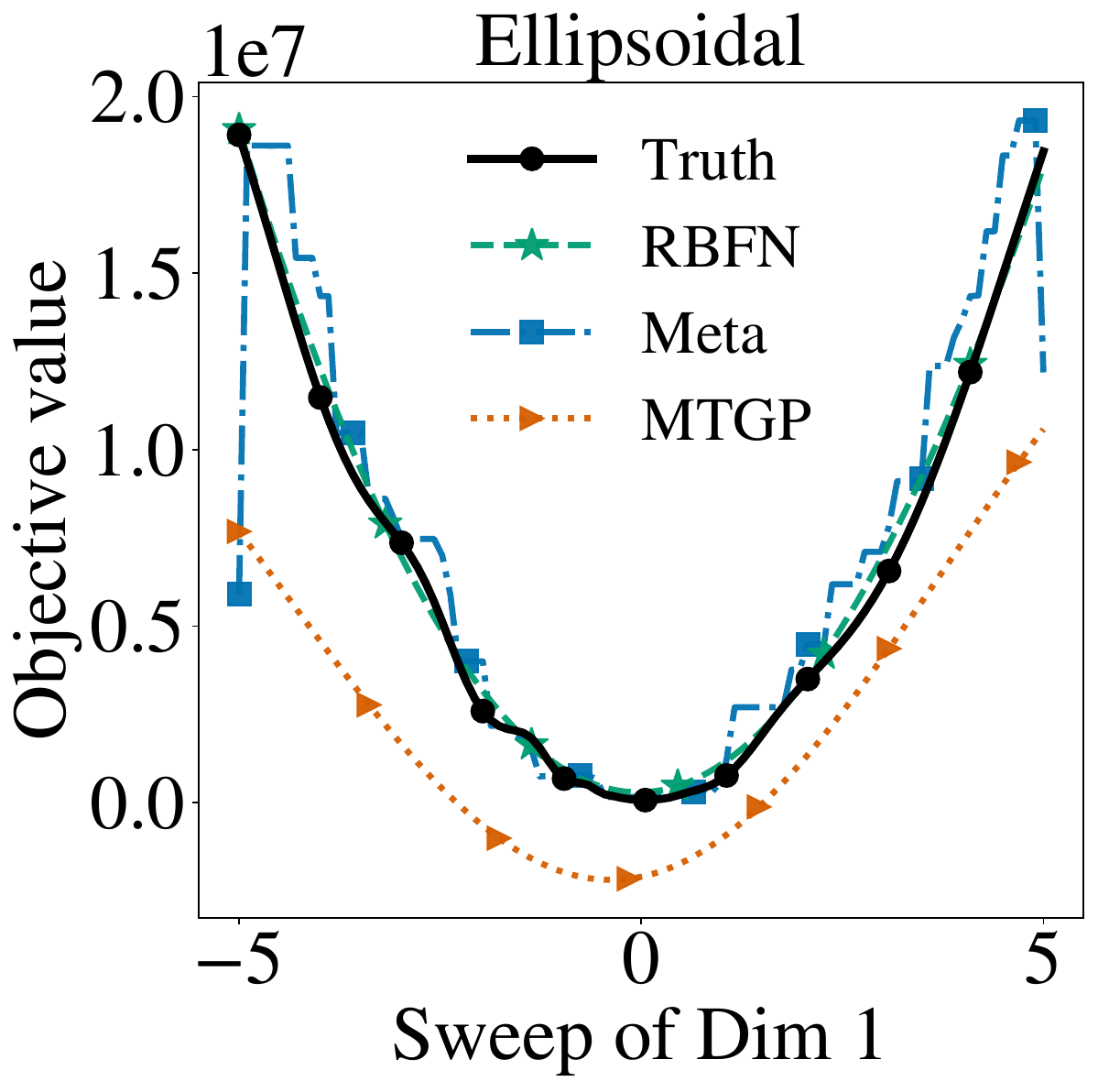}
    }
    \subfloat[\label{fig:offline-curves-ackley-200}]{
        \includegraphics[width=0.19\textwidth]{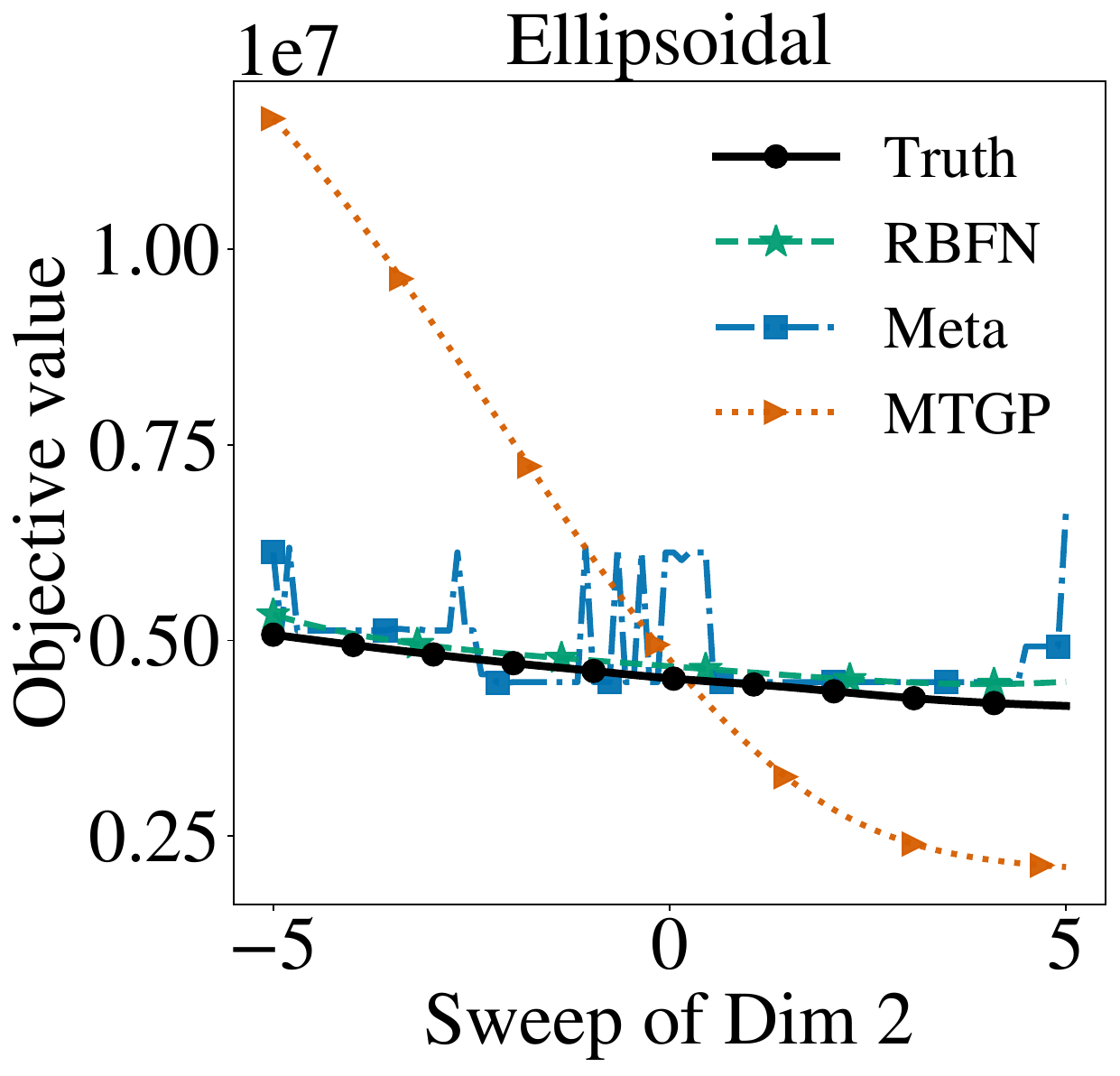}
    }
    \subfloat[\label{fig:offline-curves-ackley-300}]{
        \includegraphics[width=0.19\textwidth]{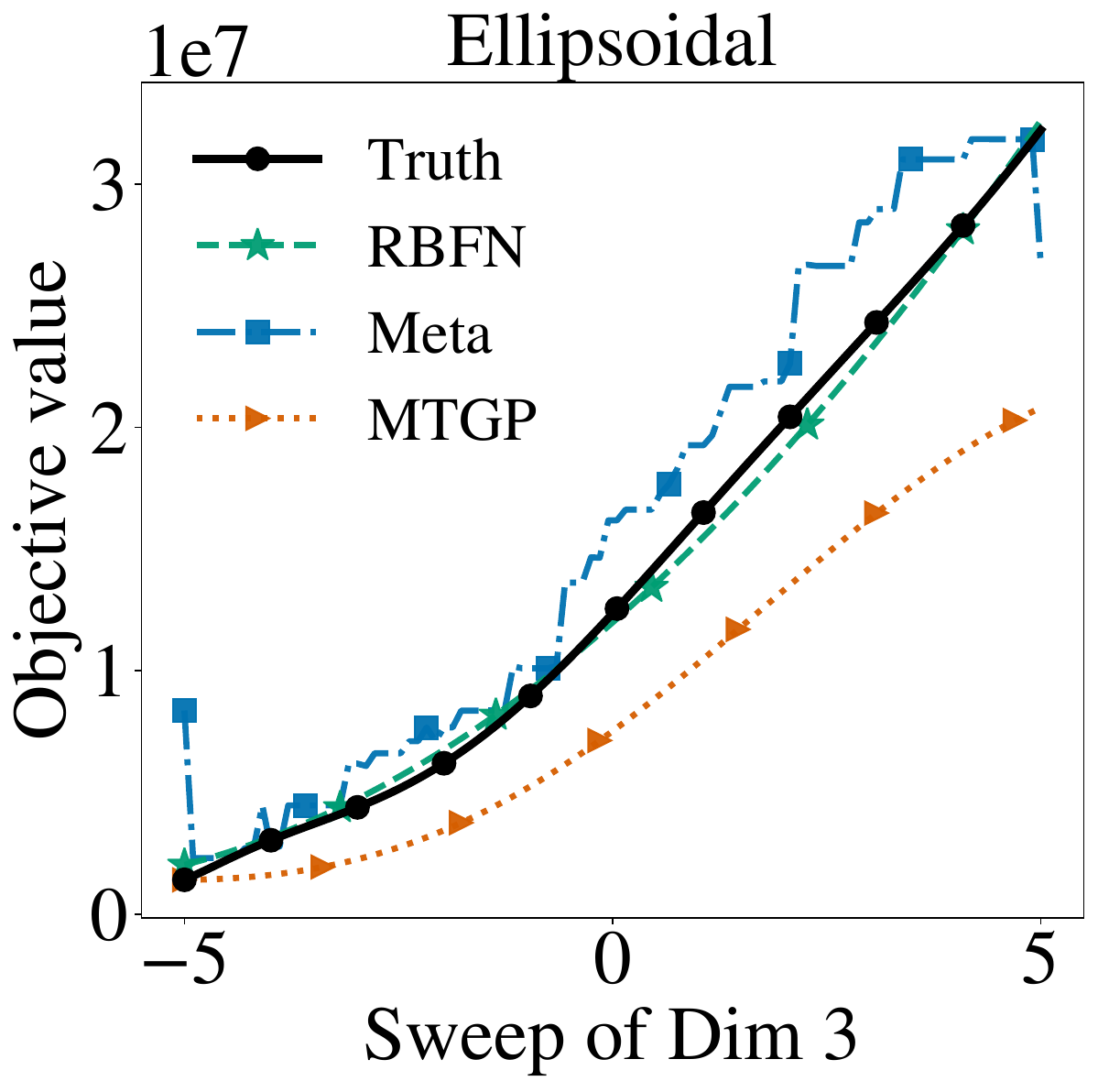}
    }
    \subfloat[\label{fig:offline-curves-ackley-500}]{
        \includegraphics[width=0.19\textwidth]{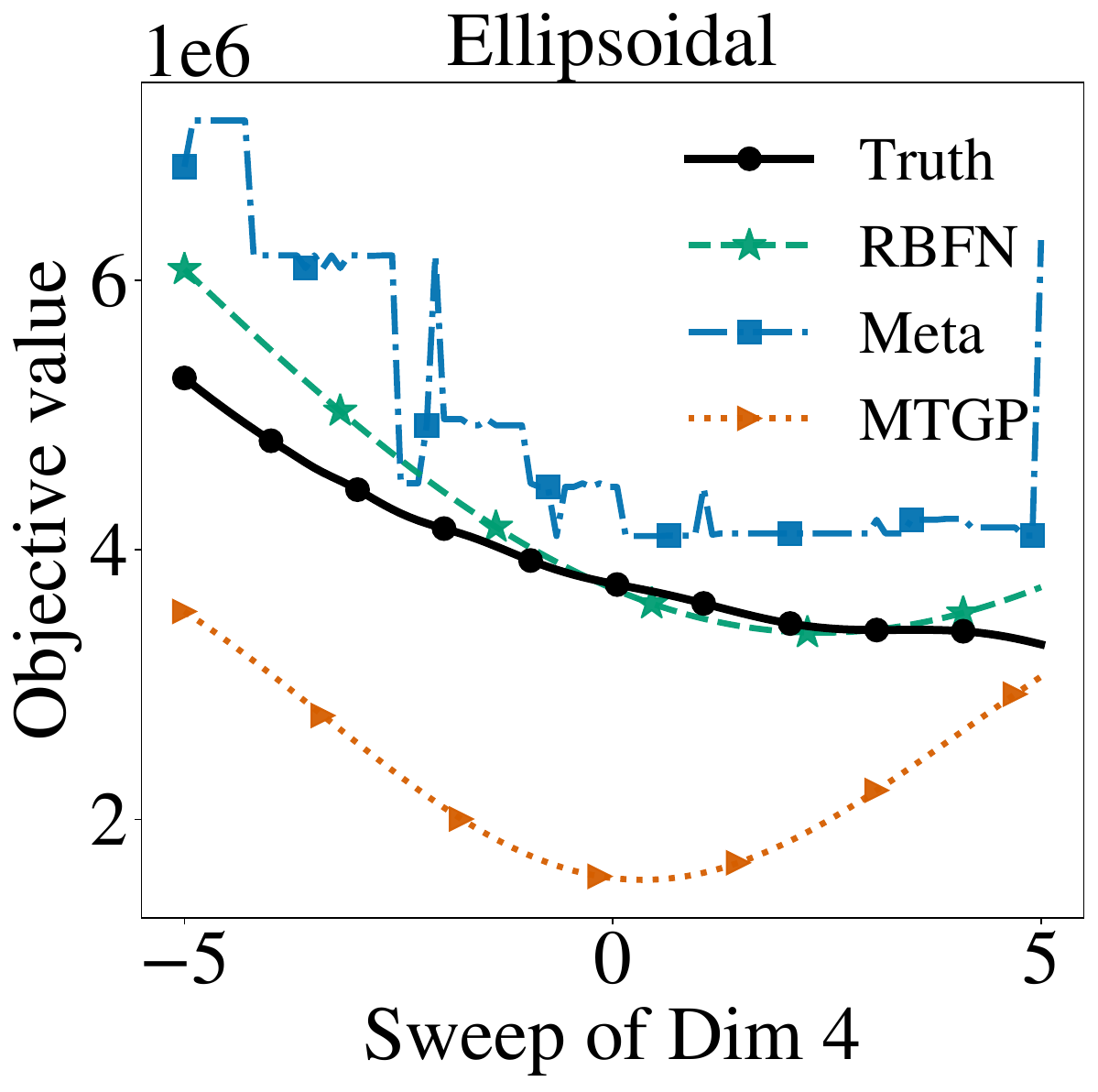}
    }
    \subfloat[\label{fig:offline-curves-ackley-1000}]{
        \includegraphics[width=0.19\textwidth]{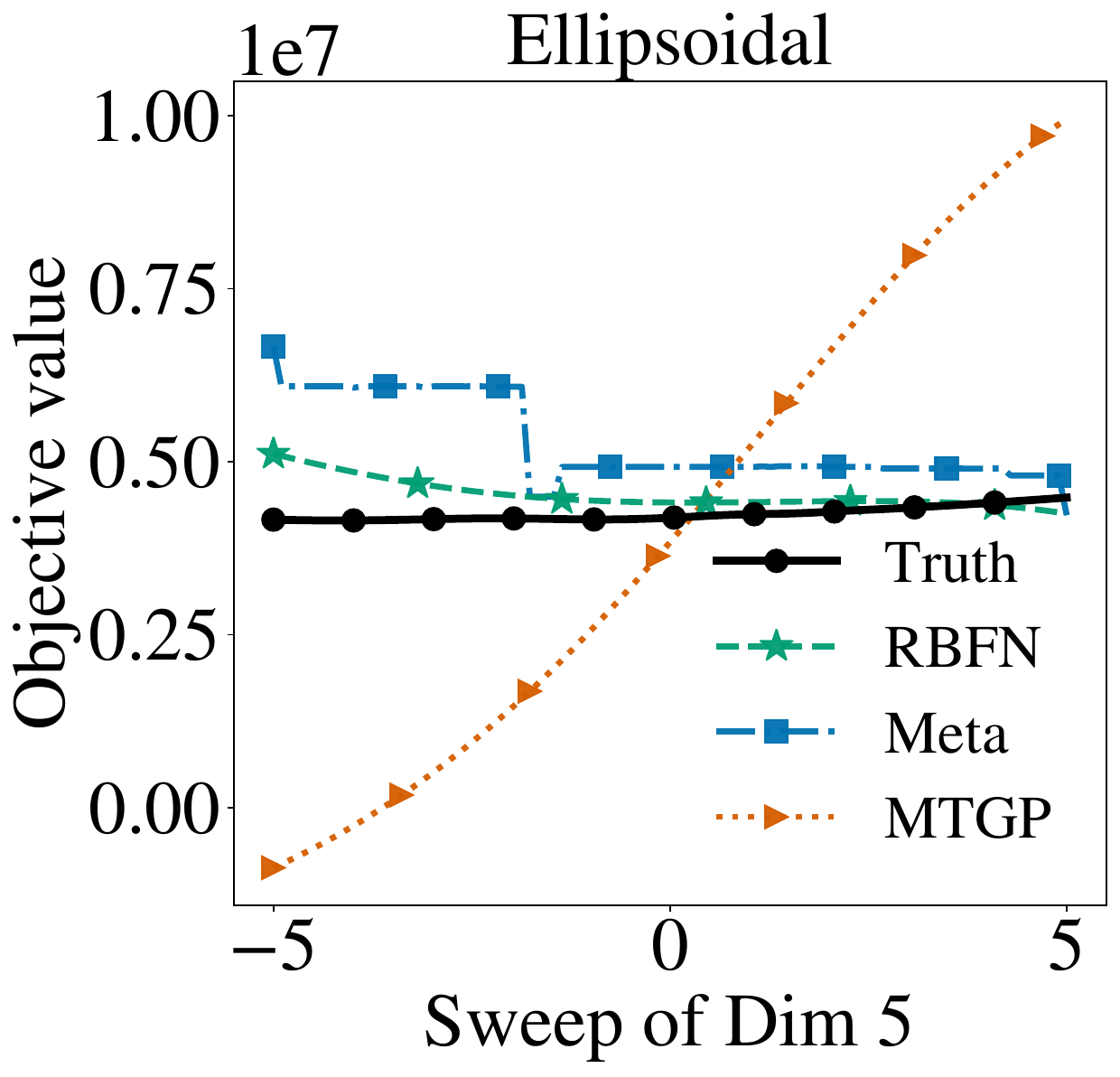}
    }
    \\
    \subfloat[\label{fig:offline-curves-ackley-100}]{
        \includegraphics[width=0.19\textwidth]{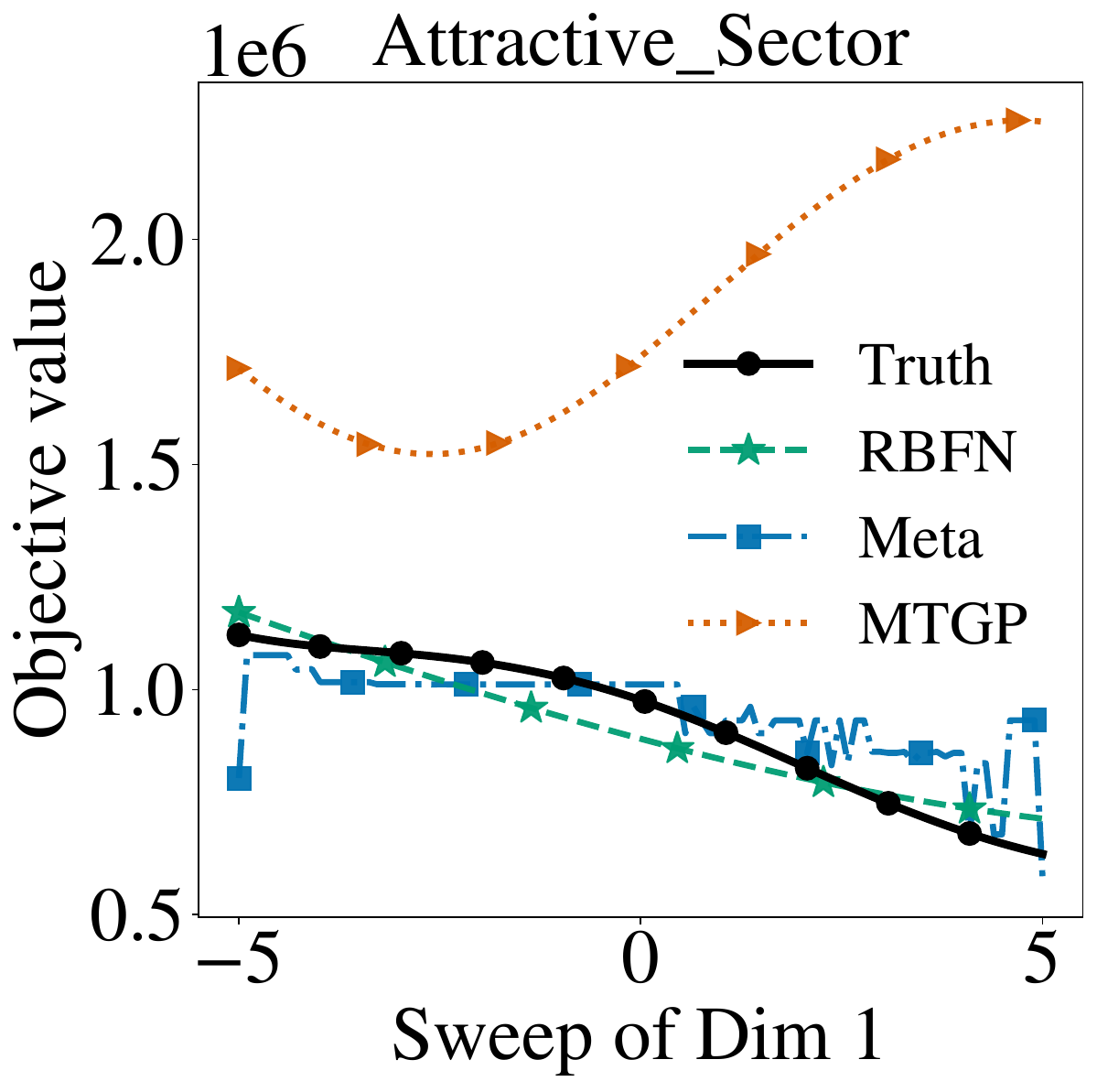}
    }
    \subfloat[\label{fig:offline-curves-ackley-200}]{
        \includegraphics[width=0.19\textwidth]{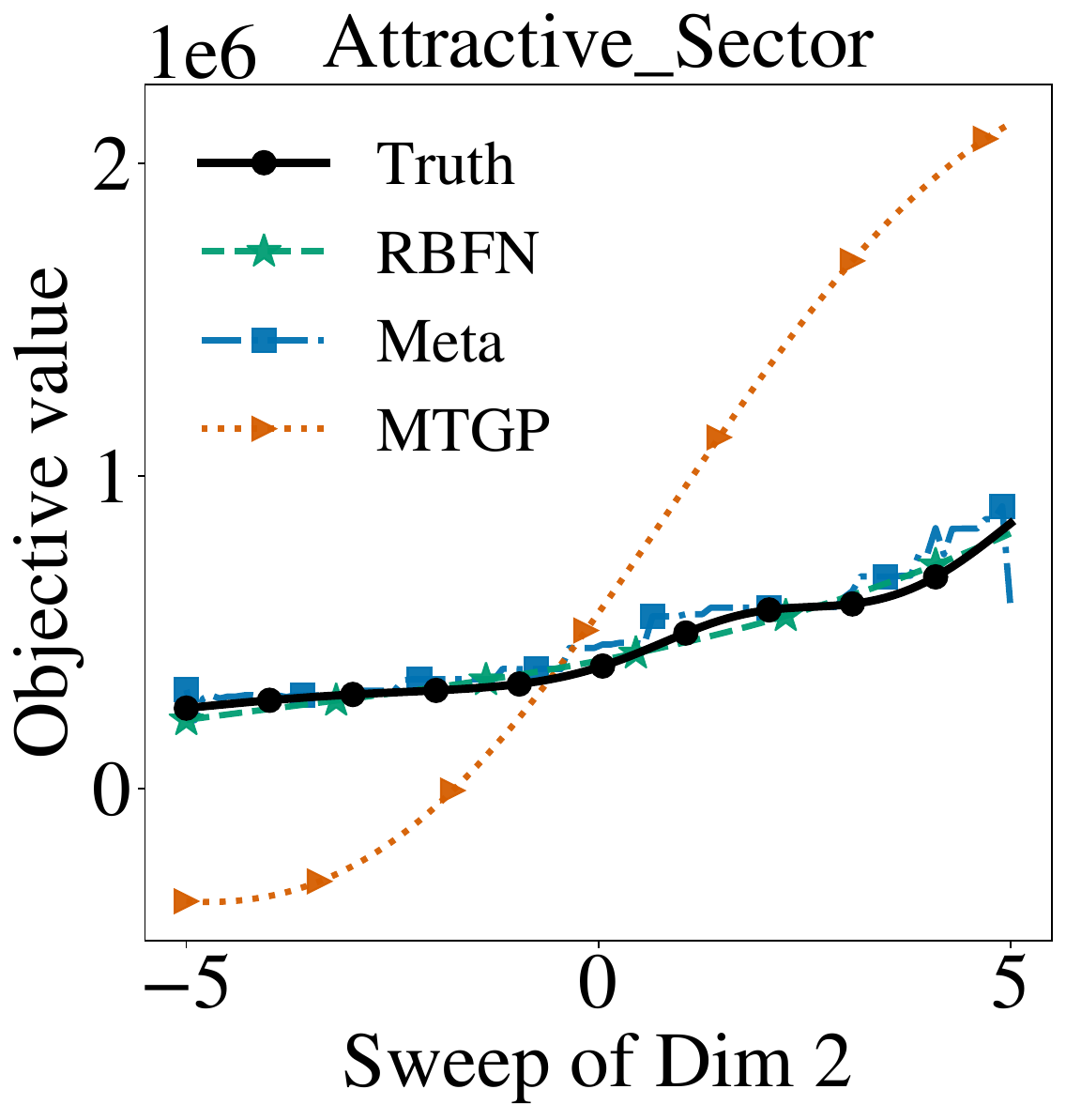}
    }
    \subfloat[\label{fig:offline-curves-ackley-300}]{
        \includegraphics[width=0.19\textwidth]{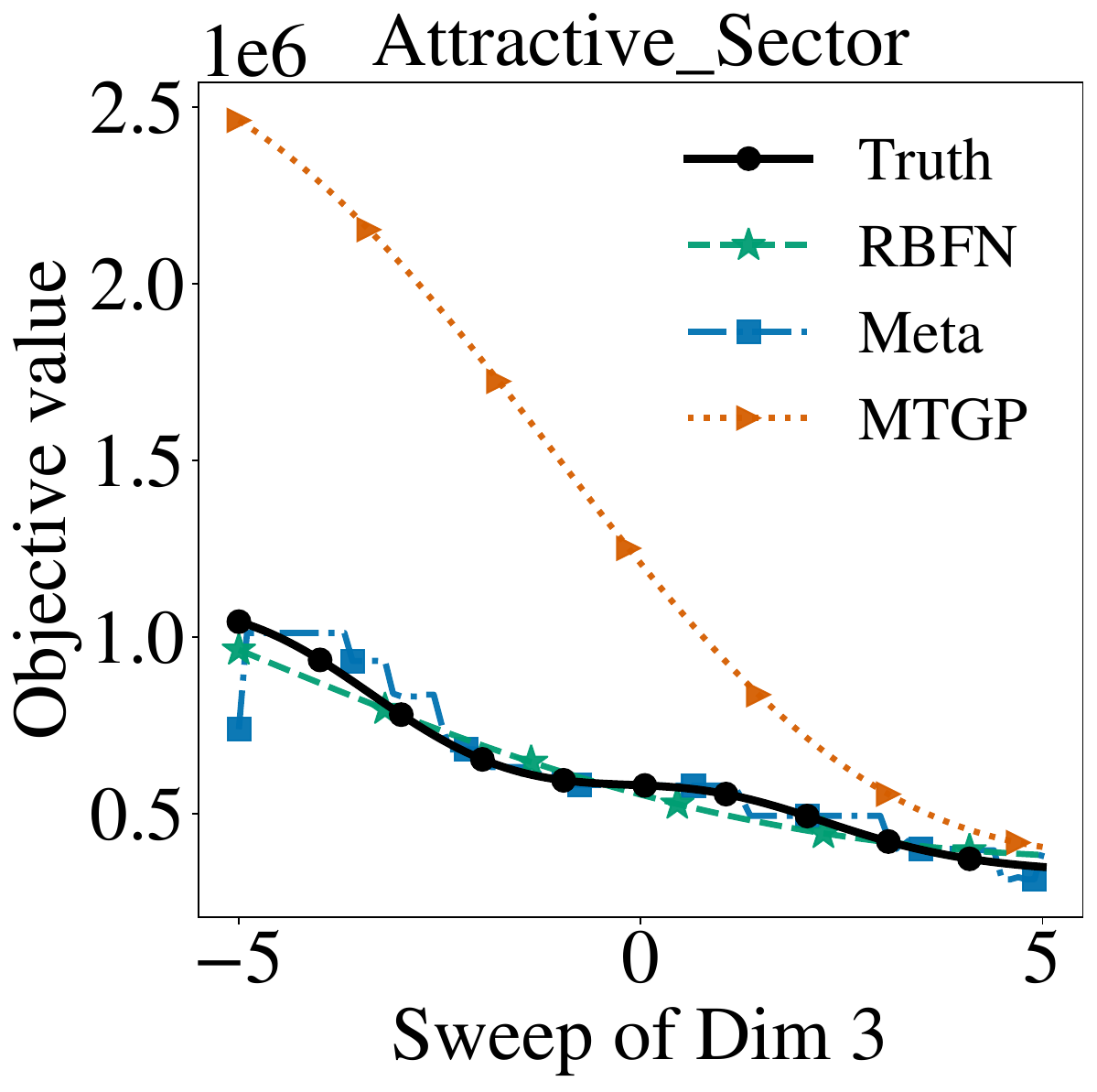}
    }
    \subfloat[\label{fig:offline-curves-ackley-500}]{
        \includegraphics[width=0.19\textwidth]{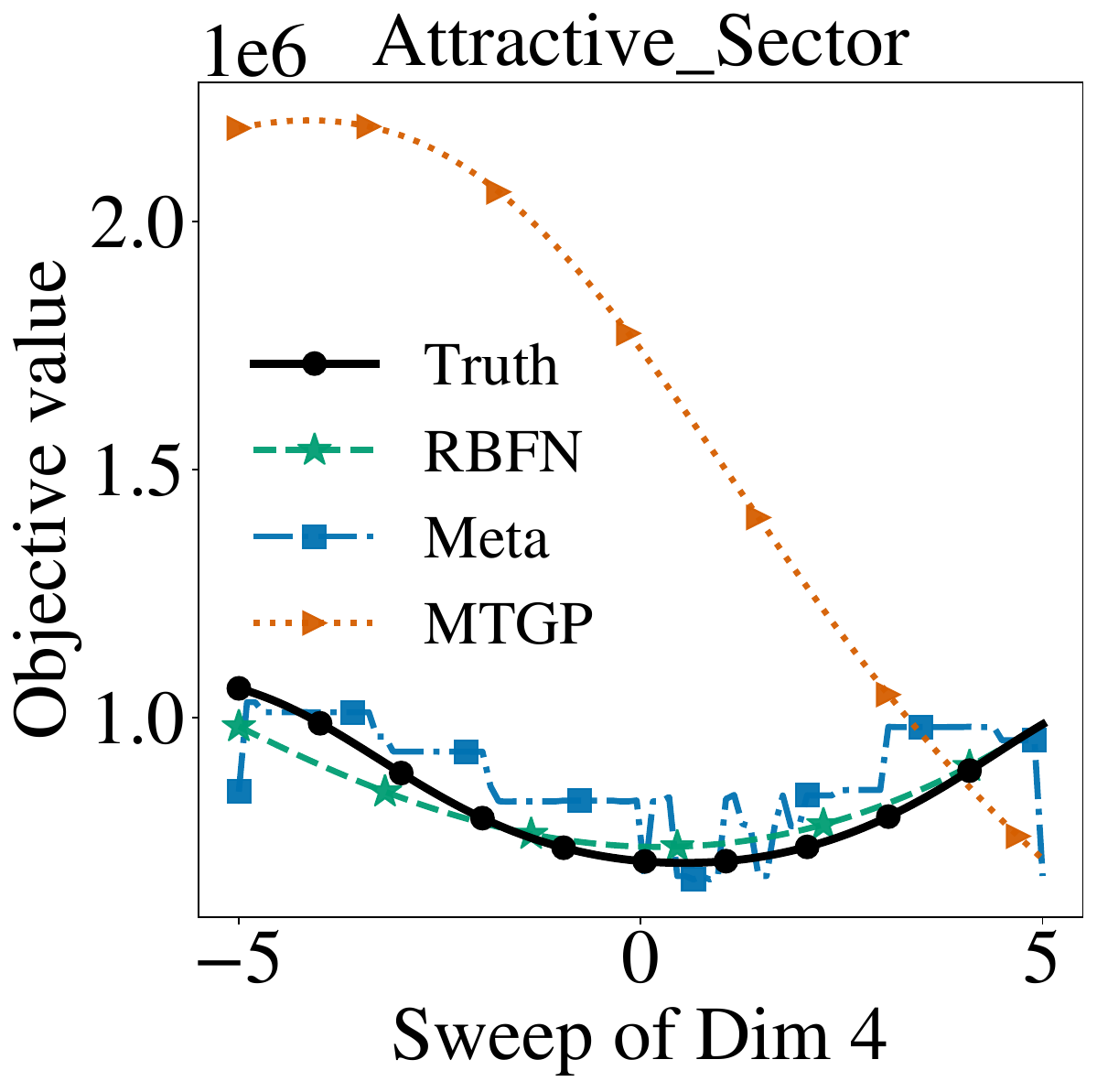}
    }
    \subfloat[\label{fig:offline-curves-ackley-1000}]{
        \includegraphics[width=0.19\textwidth]{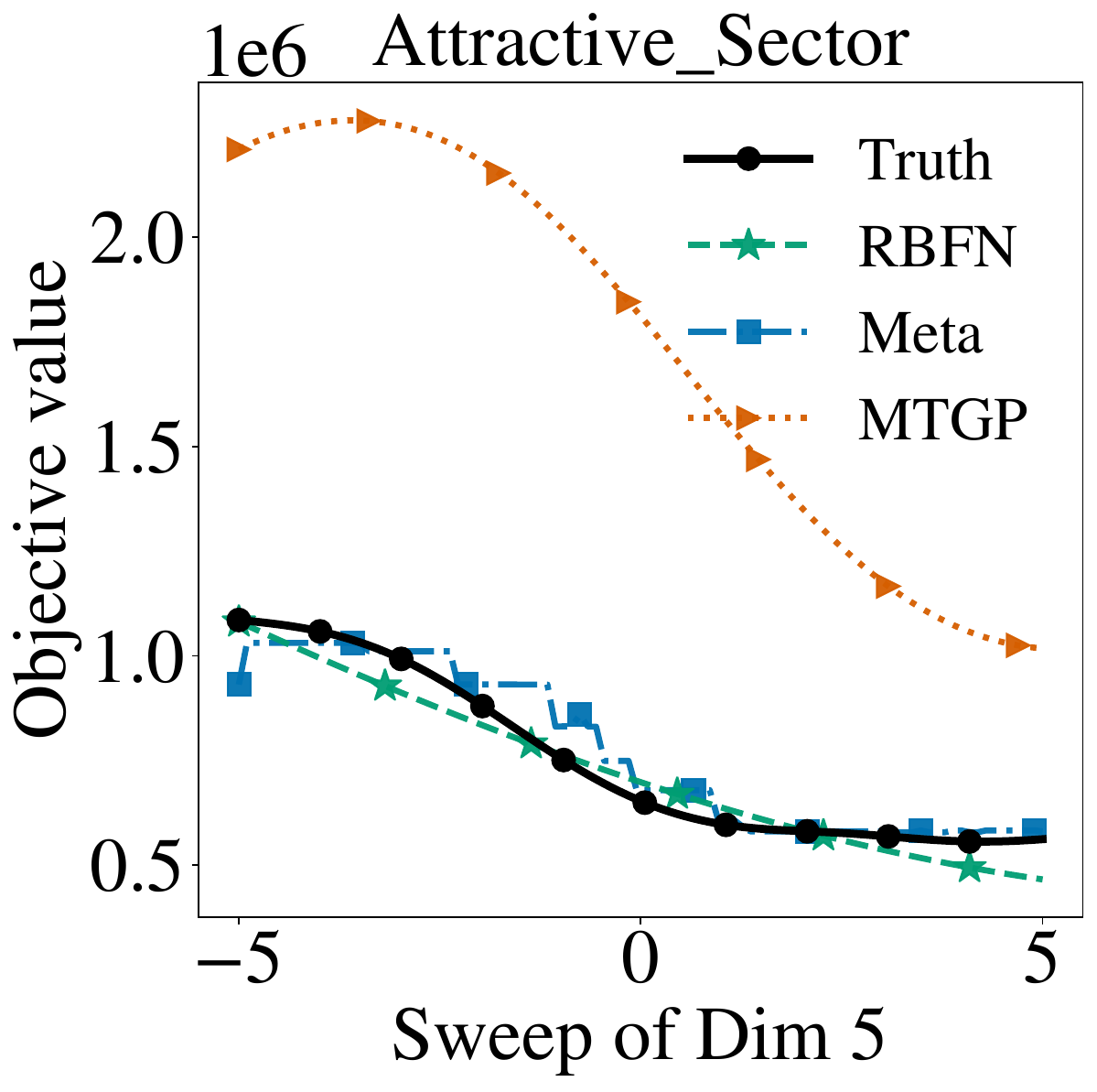}
    }
    \\
        \subfloat[\label{fig:offline-curves-ackley-100}]{
        \includegraphics[width=0.19\textwidth]{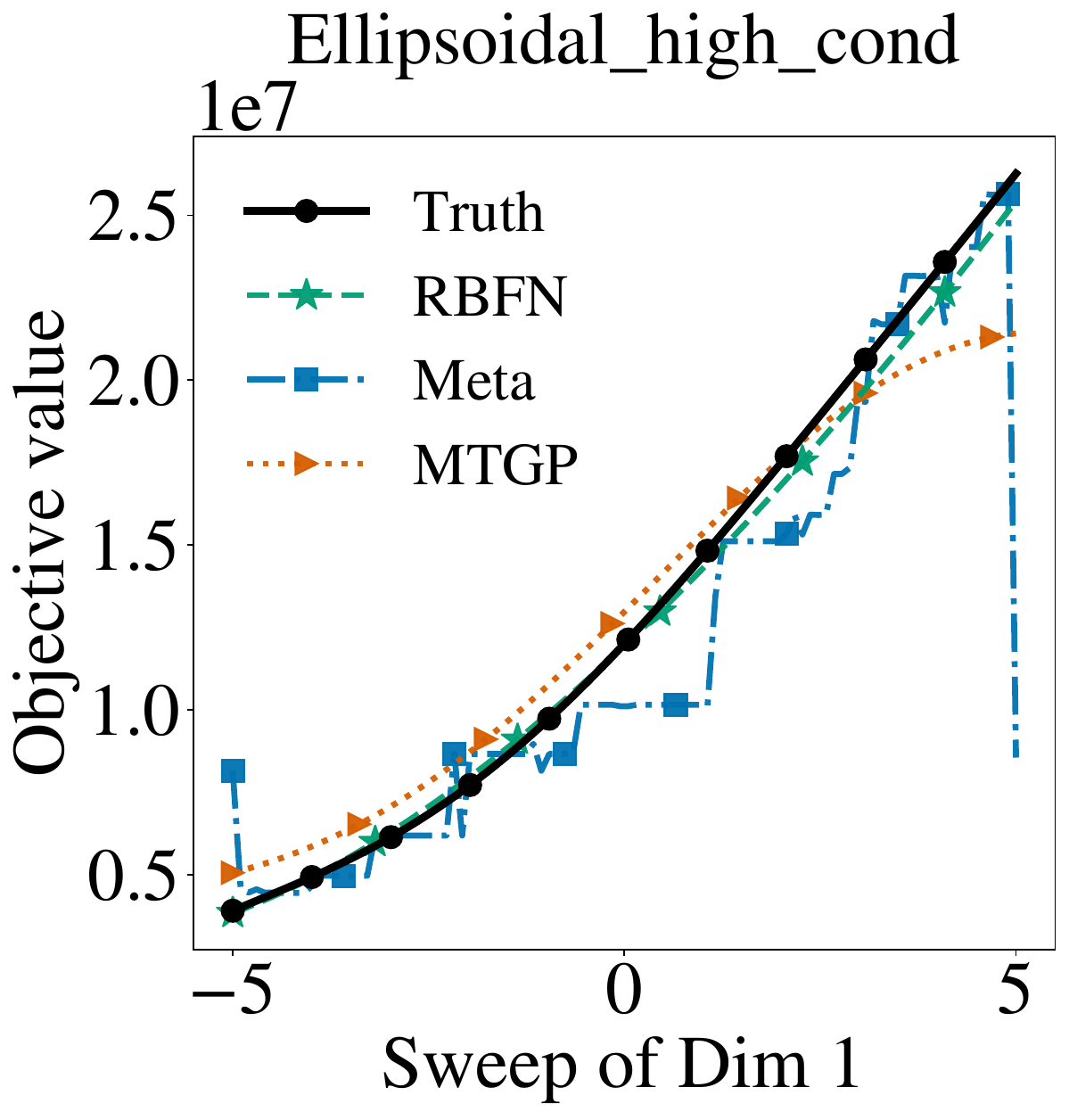}
    }
    \subfloat[\label{fig:offline-curves-ackley-200}]{
        \includegraphics[width=0.19\textwidth]{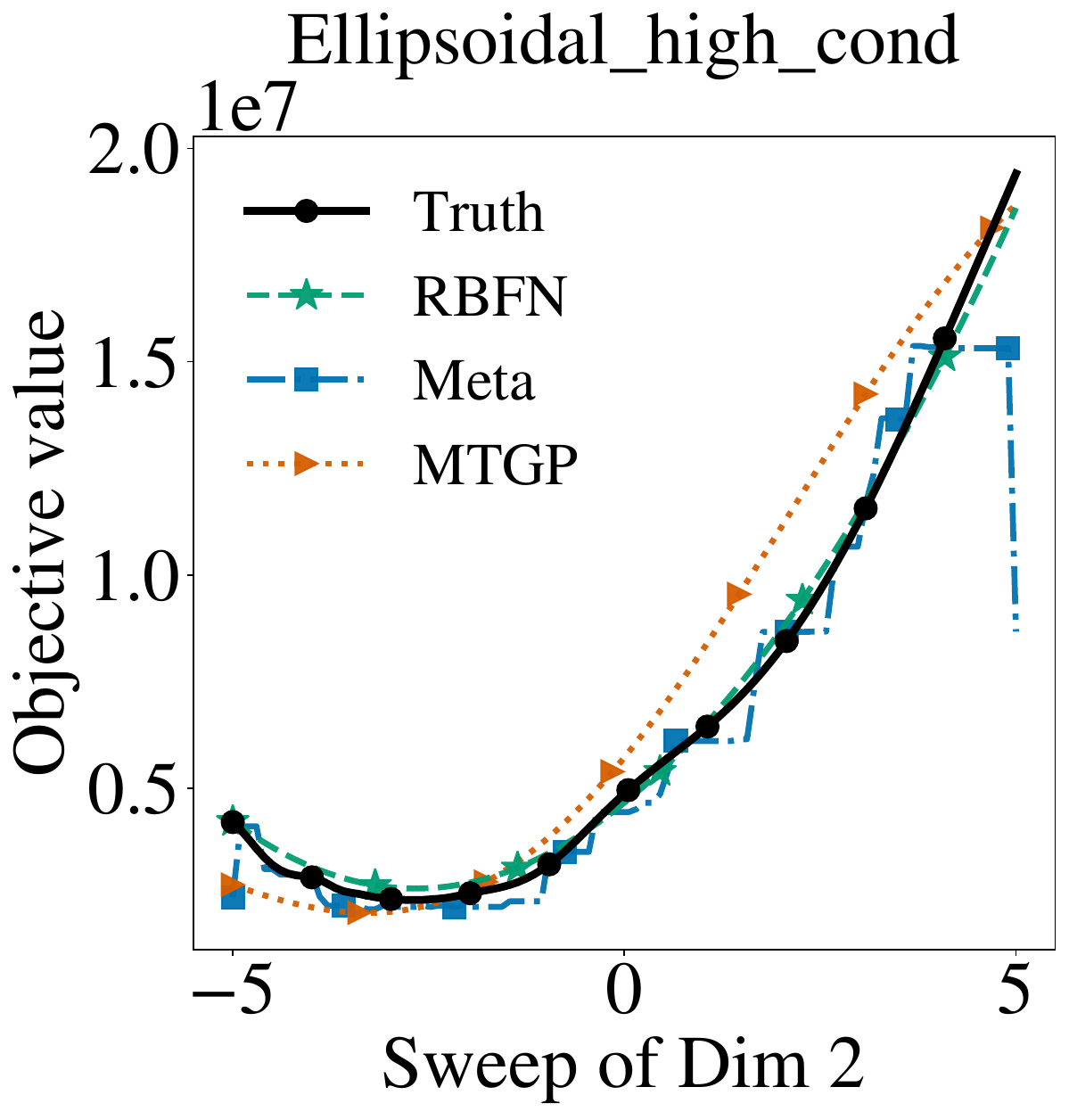}
    }
    \subfloat[\label{fig:offline-curves-ackley-300}]{
        \includegraphics[width=0.19\textwidth]{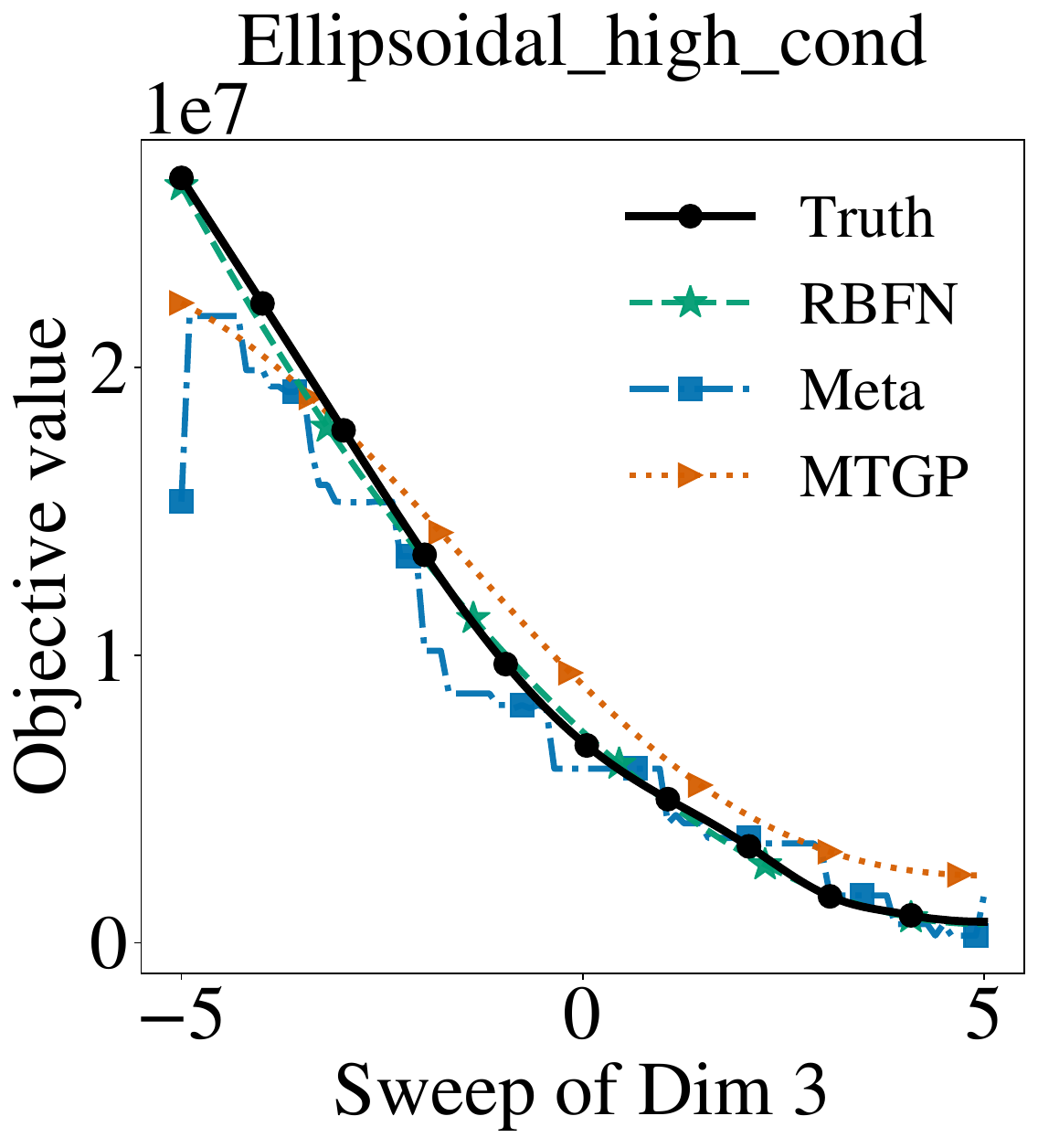}
    }
    \subfloat[\label{fig:offline-curves-ackley-500}]{
        \includegraphics[width=0.19\textwidth]{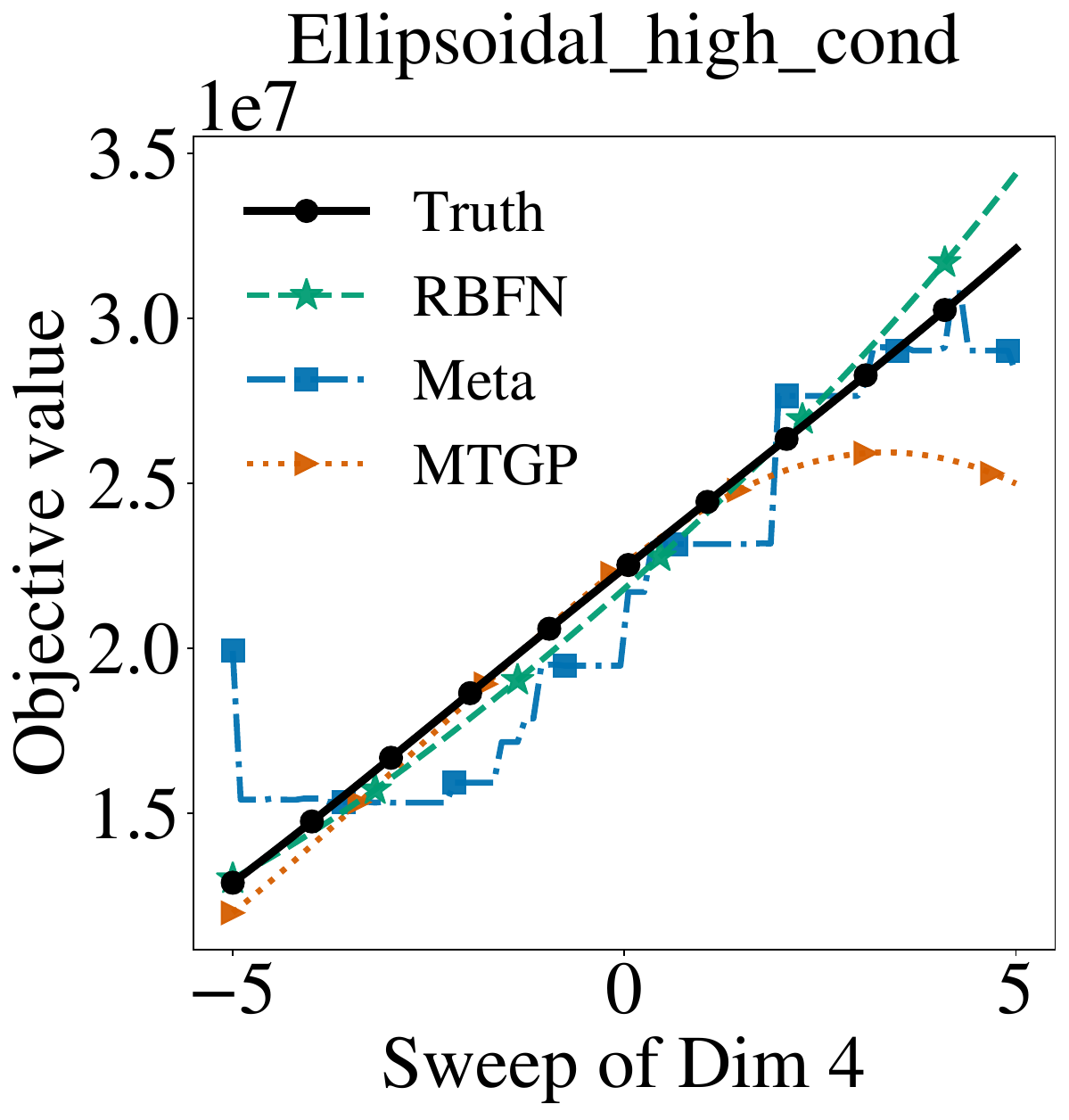}
    }
    \subfloat[\label{fig:offline-curves-ackley-1000}]{
        \includegraphics[width=0.19\textwidth]{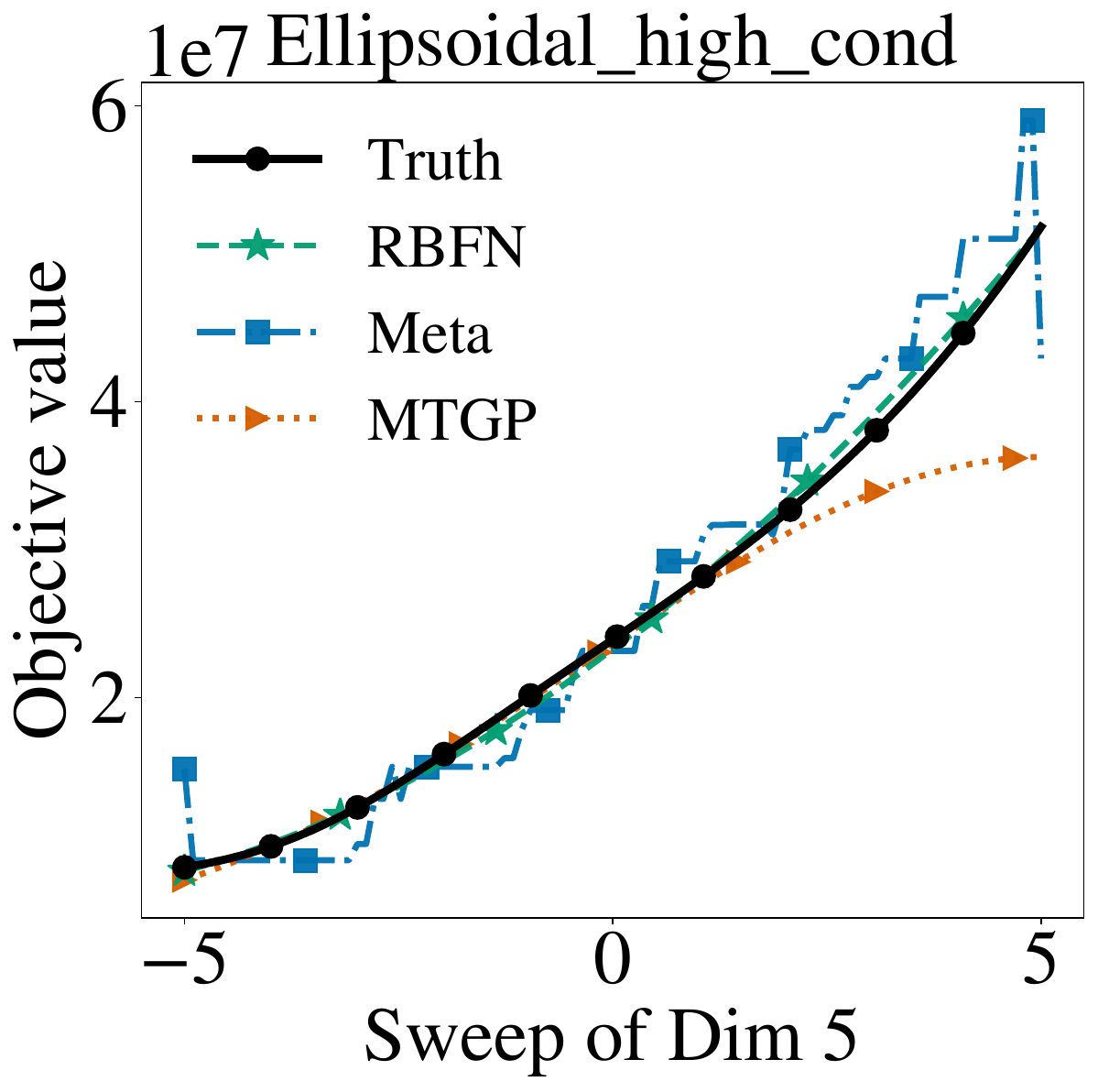}
    }
    \\
        \subfloat[\label{fig:offline-curves-ackley-100}]{
        \includegraphics[width=0.19\textwidth]{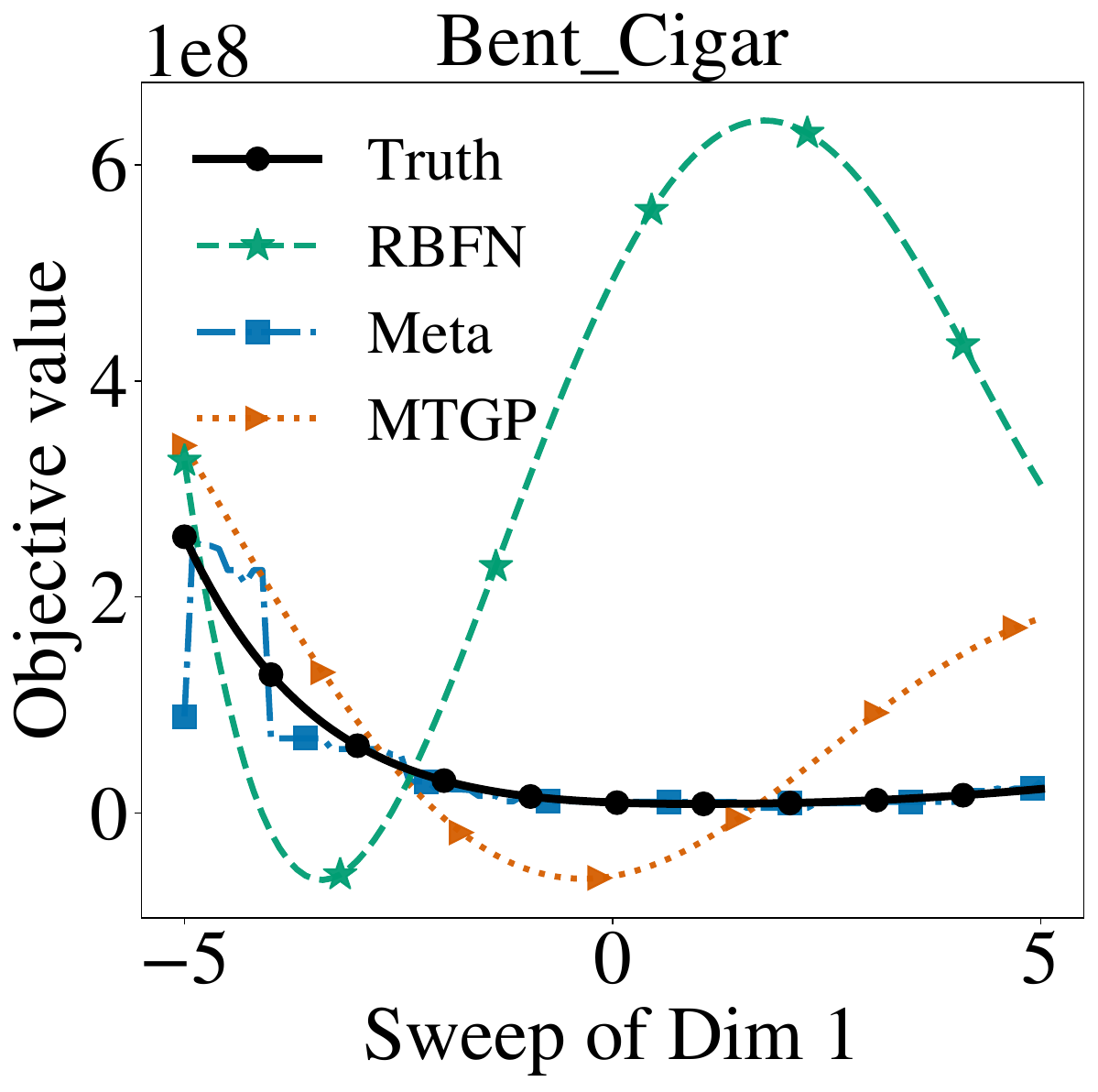}
    }
    \subfloat[\label{fig:offline-curves-ackley-200}]{
        \includegraphics[width=0.19\textwidth]{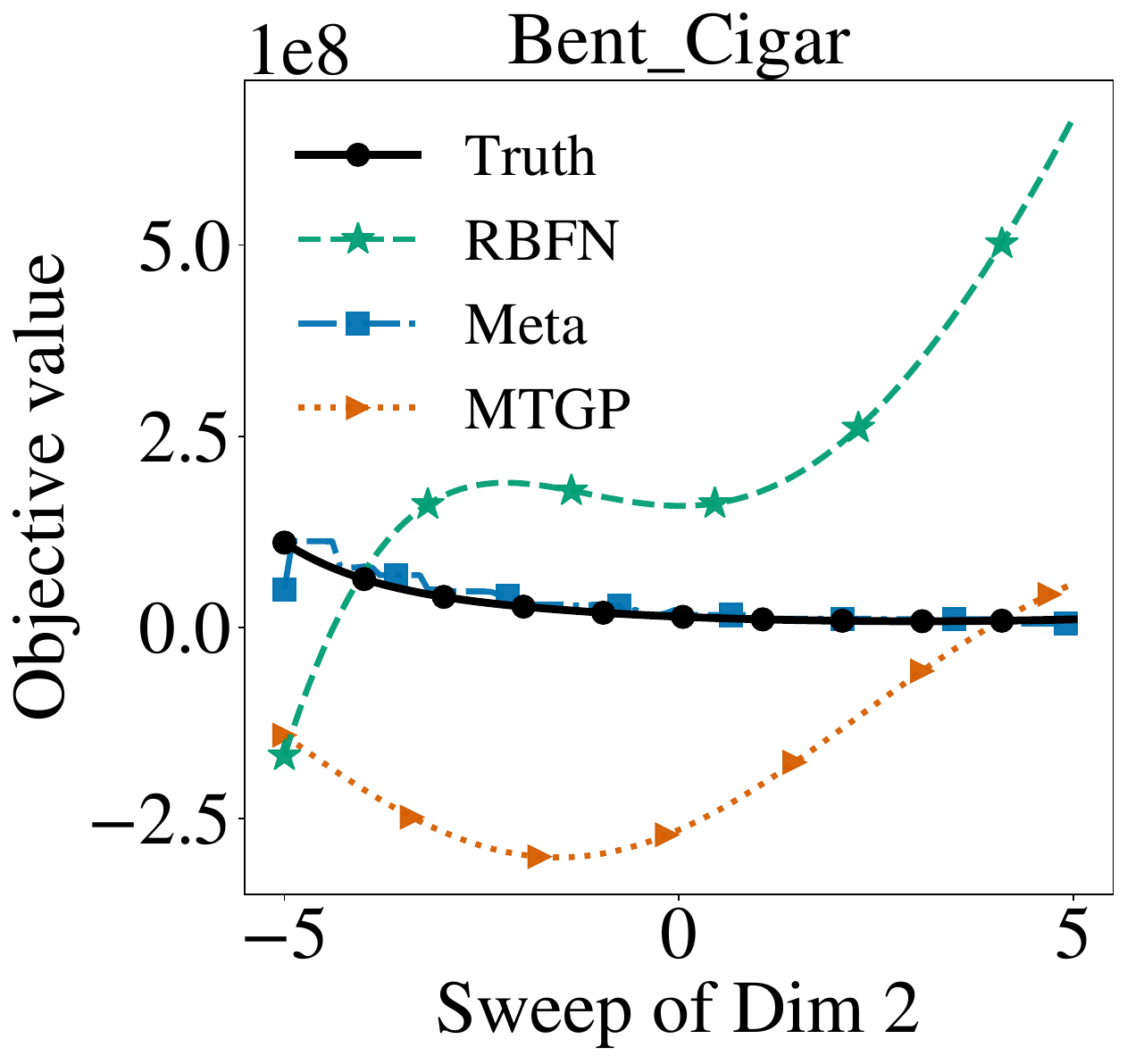}
    }
    \subfloat[\label{fig:offline-curves-ackley-300}]{
        \includegraphics[width=0.19\textwidth]{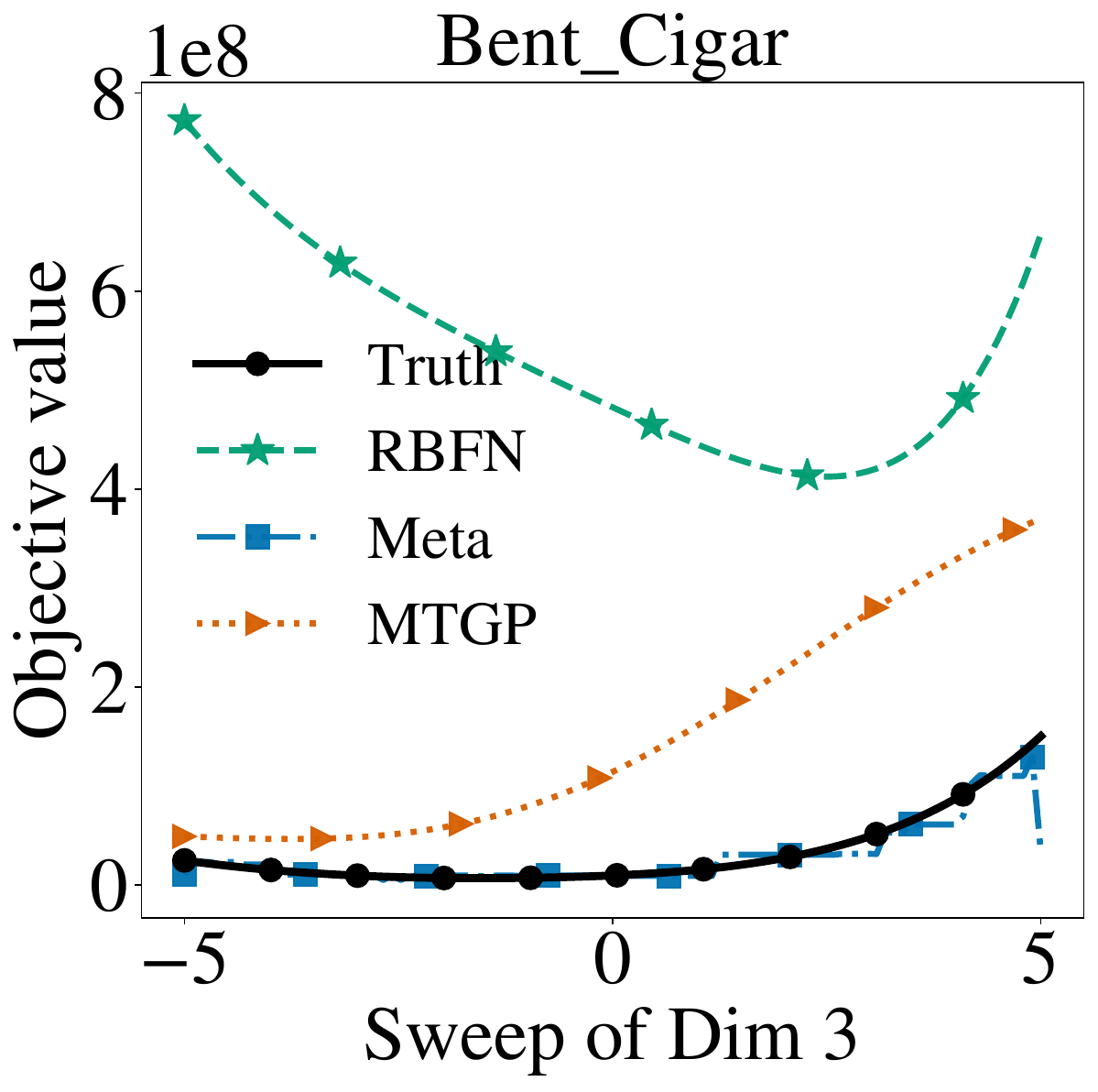}
    }
    \subfloat[\label{fig:offline-curves-ackley-500}]{
        \includegraphics[width=0.19\textwidth]{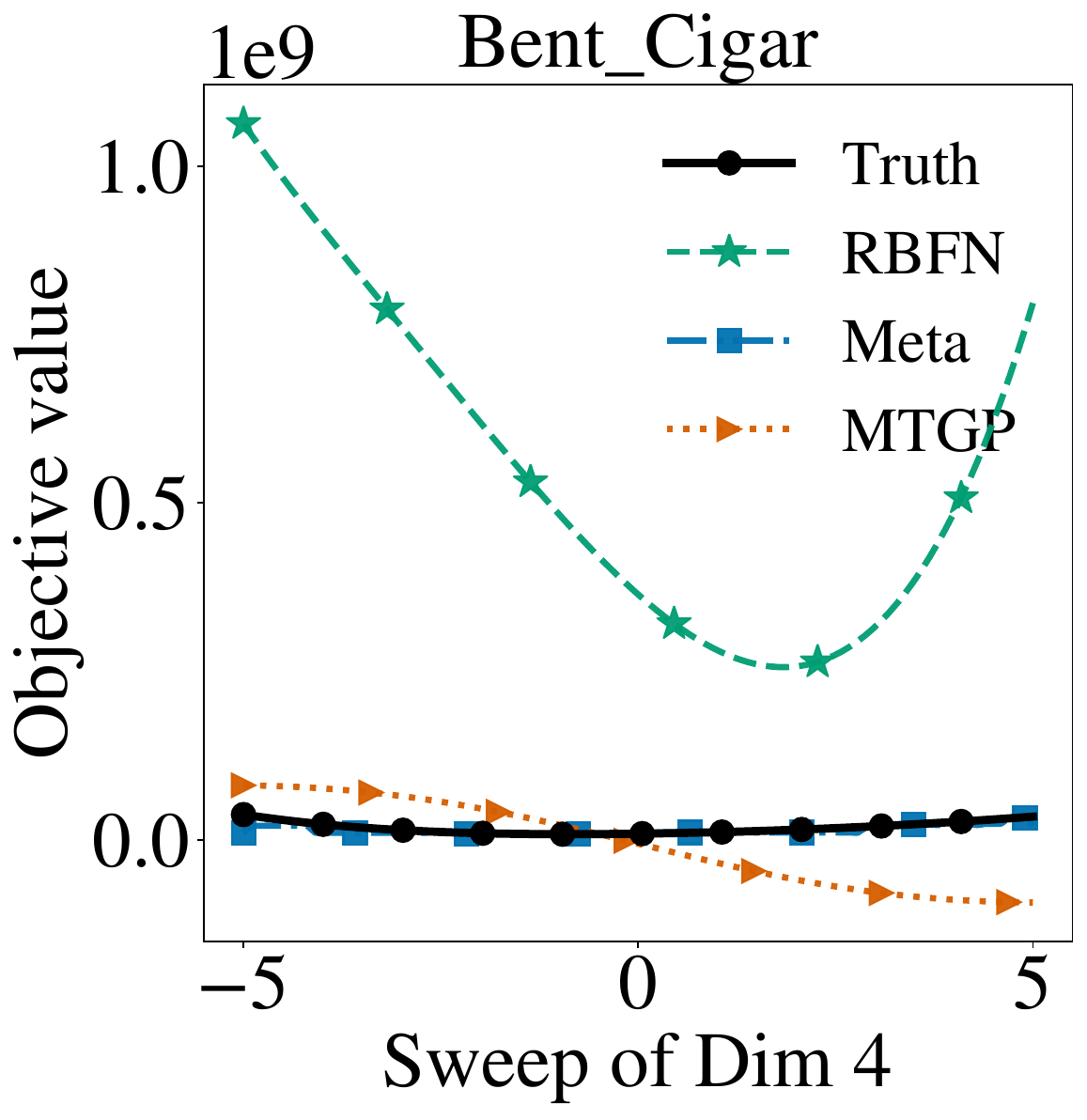}
    }
    \subfloat[\label{fig:offline-curves-ackley-1000}]{
        \includegraphics[width=0.19\textwidth]{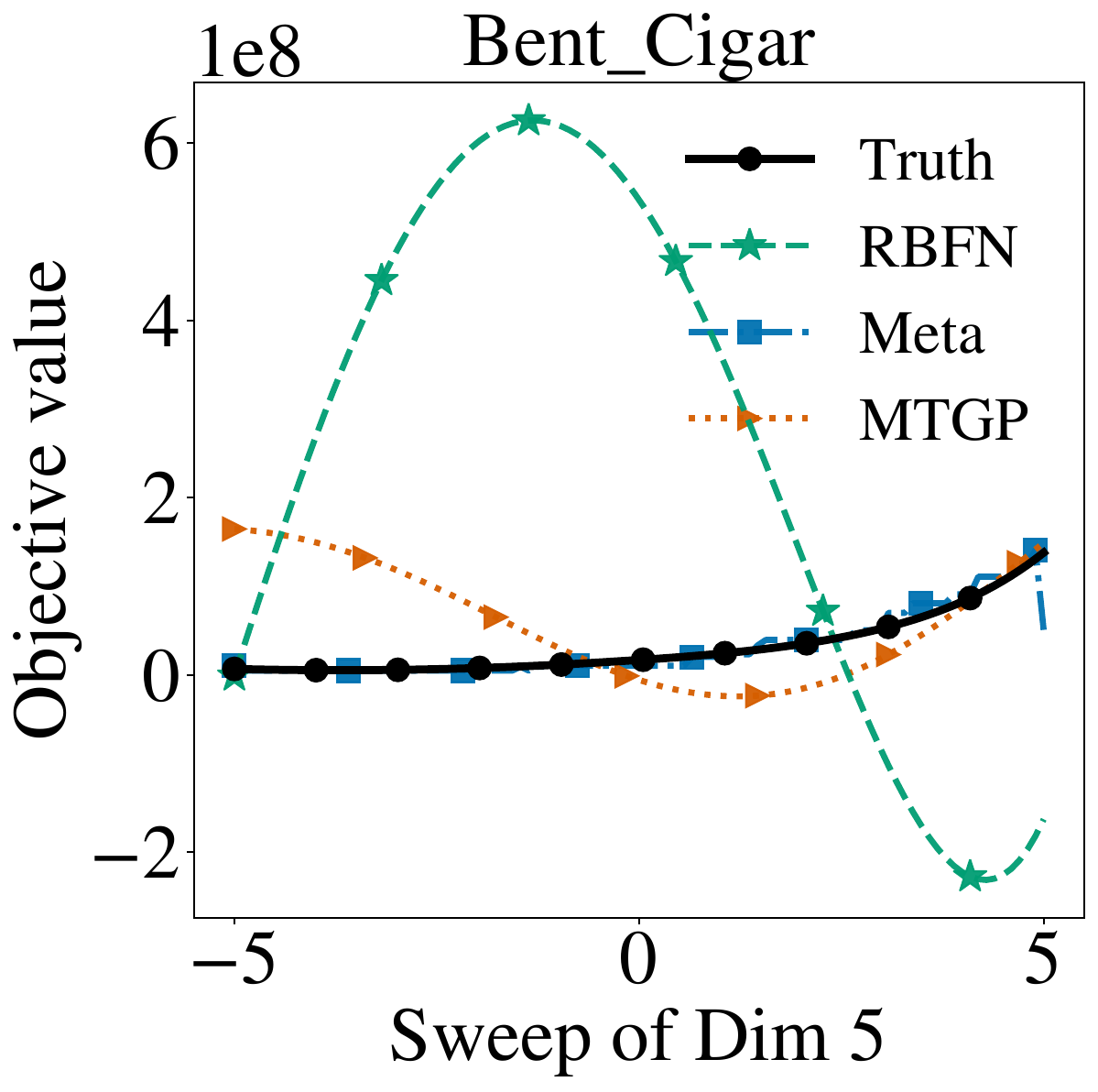}
    }
    \\
    \caption{
    Coordinate-wise response slice comparison on four representative five-dimensional BBOB functions. 
    Rows correspond to individual functions, and columns vary a single coordinate \(x_{d}\;(d=1,\dots,5)\) from \(-5\) to \(5\), while keeping the remaining coordinates fixed at a random reference point \(\mathbf{x}_{0}\). The plots highlight the close alignment of the meta-surrogate and RBFN with the true response, and the comparatively larger deviations observed in MTGP.
    }
    \label{fig:Slice}
\end{figure*}

\subsubsection{Uncertainty-Error Correlation Study}
\label{subsec:uncertain}

To examine whether the token-level predictive uncertainty revealed by the meta-surrogate is informative of the true regression error, we computed five \emph{sequence}-level uncertainty scores for $10^{4}$ test predictions and correlated them with the absolute
prediction error $\lvert\hat{y}-y\rvert$. The exact mathematical definitions are summarised in~\autoref{tab:uncert_defs}.

\begin{table}[htbp]
  \centering
  \caption{Uncertainty measures investigated in this study. Here, $L$ denotes the sequence length; $p_t^{(k)}$ represents the $k$-th largest token probability at position~$t$; and $\text{beam}_{1:3}$ refers to the top-3 beam hypotheses decoded into numerical values.}
  \label{tab:uncert_defs}
  \resizebox{0.8\linewidth}{!}{ 
  \begin{tabular}{|c|c|l|}
    \hline
    Symbol & Formal Definition & Interpretation \\ \hline
    $\mathcal{U}_{\text{NLL}}$
        & $\displaystyle
           \frac{1}{L}\sum_{t=1}^{L}-\log p_t(y_t)$
        & Negative log–likelihood (sequence perplexity) \\ \hline
    $\mathcal{U}_{\text{IMSP}}$
        & $\displaystyle
           1-\frac{1}{L}\sum_{t=1}^{L}p_t^{(1)}$
        & Inverse maximum soft-max probability \\ \hline
    $\mathcal{U}_{\text{ENT}}$
        & $\displaystyle
           \frac{1}{L}\sum_{t=1}^{L}
           H(p_t),\;
           H(p_t)=-\sum_v p_t^{(v)}\log p_t^{(v)}$
        & Mean token Shannon entropy \\ \hline
    $\mathcal{U}_{\text{ITPM}}$
        & $\displaystyle
           1-\frac{1}{L}\sum_{t=1}^{L}\bigl(
           p_t^{(1)}-p_t^{(2)}\bigr)$
        & Inverse top-2 probability margin \\ \hline
    $\sigma_{\text{Beam}}$
        & $\operatorname{Std}\bigl(\text{beam}_{1:3}\bigr)$
        & Posterior dispersion of the top-3 beams \\
    \hline
  \end{tabular}
  }
\end{table}

Table~\ref{tab:corr_uncert_mae} reports the average Pearson
($\rho$), Spearman ($\rho_{s}$), and Kendall ($\tau$) coefficients
obtained over the entire test set.

\begin{table}[htbp]
  \centering
  \caption{Correlation between uncertainty estimates and
  absolute prediction error (larger is stronger).
  All $p$-values are below $10^{-25}$ unless noted.}
  \label{tab:corr_uncert_mae}
  \renewcommand{\arraystretch}{1.15}
  \resizebox{0.5\linewidth}{!}{%
  \begin{tabular}{|c|c|c|c|}
    \hline
    Criterion & Pearson $\rho$ & Spearman $\rho_{s}$ & Kendall $\tau$ \\ \hline
    $\mathcal{U}_{\text{NLL}}$   & $-0.010$ ($p{=}0.75$) & $0.639$ & $0.477$ \\ \hline
    $\mathcal{U}_{\text{IMSP}}$  &  $0.057$              & $0.708$ & $0.533$ \\ \hline
    $\mathcal{U}_{\text{ENT}}$   &  $0.061$              & $\mathbf{0.723}$ & $\mathbf{0.547}$ \\ \hline
    $\mathcal{U}_{\text{ITPM}}$  &  $0.062$              & $0.705$ & $0.530$ \\ \hline
    $\sigma_{\text{Beam}}$       &  $0.053$              & $0.329$ & $0.227$ \\ \hline 
  \end{tabular}
  }
\end{table}

\noindent
Two noteworthy observations emerge: (1) The three \emph{token-aware} criteria---$\mathcal{U}_{\text{IMSP}}$, $\mathcal{U}_{\text{ENT}}$, and $\mathcal{U}_{\text{ITPM}}$ exhibit strong, monotonic association with MAE ($\rho_{s}>0.70$, $\tau>0.53$), demonstrating that the step-wise probability mass supplied by the meta-surrogate is a reliable proxy for numerical error. (2) Among them, \emph{mean token Shannon entropy} $\mathcal{U}_{\text{ENT}}$ achieves the highest correlation ($\rho_{s}=0.723$, $\tau=0.547$), highlighting Shannon entropy as an intuitive and effective \emph{information-risk index} for surrogate predictions.

Although \emph{sequence perplexity}
($\mathcal{U}_{\text{NLL}}$) shows negligible linear correlation due
to its exponential dependence on sequence length, its rank-based
performance remains competitive once the non-linear relationship is
taken into account.  The sample-based posterior dispersion
$\sigma_{\text{Beam}}$ correlates only moderately with MAE, implying
that \emph{cross-beam consistency} is less informative than the direct
token-level probabilities emitted by the model.

Based on this study, we recommend dynamically updating the surrogate
model in \emph{online} DDEA by
leveraging the token-level entropy generated during LLM inference.
Specifically, at each optimization iteration, individuals in the
offspring population exhibiting the highest
$\mathcal{U}_{\text{ENT}}$ should be selected for ground-truth fitness
evaluation.  The resulting evaluations are then used to fine-tune the
LLM, thereby continuously enhancing the surrogate's predictive
accuracy and adaptability while limiting expensive fitness calls.

\subsection{Emergent Capability (RQ3)}
\begin{figure}[t]
    \centering
    \includegraphics[width=0.9\linewidth]{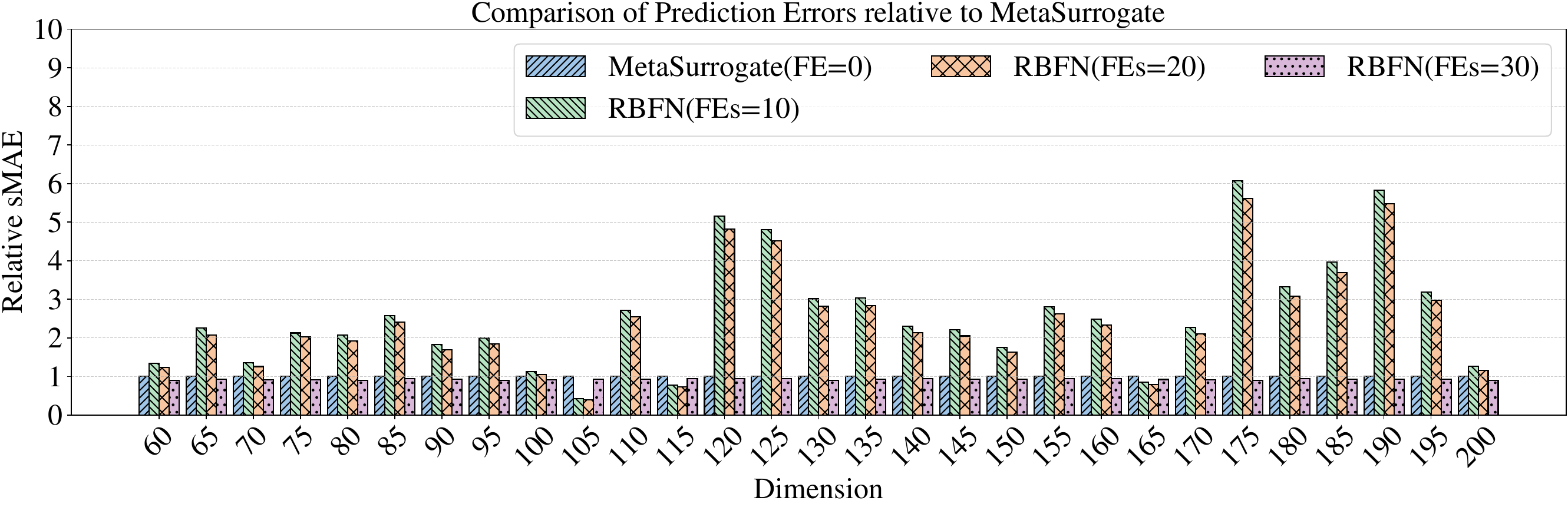}
    \caption{
    Cross-dimensional evaluation on BBOB-24. Bars show {relative} sMAE normalized to {MetaSurrogate (FE=0)}$=1$ at each dimension; lower is better.}
    \label{fig:relative_rmse}
\end{figure}

%
In surrogate modeling, if a model trained solely on a certain dimensional setup (e.g., 30 dimensions) can yield plausible predictions for dimensions not seen during training (e.g., 29 or 31 dimensions), it indicates cross-dimensional generalization ability. This means the model not only excels within the trained dimension range but can also handle inputs that deviate from that distribution, implying emergent capabilities akin to those in LLMs.

To verify this, we tested on 24 BBOB problems, the meta-surrogate is trained on the 24 BBOB functions with input dimensions of $\{10,15,20,\ldots,50\}$, using 1000 training samples for each. It is then tested on other dimensional settings ($\{60,65,\ldots,200\}$). For context, we include RBFN baselines that are allowed per-dimension supervision at test time with $\{10,20,30\}$ function evaluations (FEs). Figure~\ref{fig:relative_rmse} normalizes MetaSurrogate (FE=0) to 1.0 at each dimension; the three RBFN {bars} report {relative} sMAE (method / MetaSurrogate).

Over 29 test dimensions from 60D to 200D, the zero-shot meta-surrogate consistently outperforms lightly-supervised RBFN with 10 or 20 FEs on {26/29} dimensions each; the exceptions occur at 105D, 115D, and 165D, where RBFN(10/20) attains slightly lower sMAE. Absolute sMAE grows with dimensionality for all methods and shows non-monotonic spikes (e.g., around 170D and 200D), reflecting aggregate difficulty across the BBOB suite at those scales; importantly, the {relative} gaps remain bounded, suggesting that the meta-surrogate preserves much of its ranking fidelity even far outside the training span. Taken together, these results support a {finite-range} zero-shot transfer phenomenon---robust across 60--200D under our setting.

The cross-dimensional generalization ability of {the} meta-surrogate is primarily attributed to the following factors: (1) meta-surrogate possesses a vast number of parameters and complex non-linear mapping capabilities, which enable it to capture a universal numerical mapping mechanism during training rather than merely relying on memorization or interpolation of fixed-dimensional inputs; (2) The self-attention mechanism inherent in its Transformer component facilitates token-level global information exchange. Consequently, even when confronted with inputs of previously unseen dimensions, the model is capable of integrating both local and global numerical patterns to form effective internal representations and enable knowledge sharing; (3) Fundamentally, meta-surrogate is a LLM that has been pre-trained on extensive, multi-domain textual data, thereby accumulating a rich repository of symbolic and numerical representation knowledge. This cross-modal knowledge accumulation permits the model to perform a certain degree of knowledge transfer and reasoning when processing inputs with new dimensions. Although zero-shot predictions in low-dimensional settings exhibit emergent capabilities, the generalization capacity remains limited when the input dimensions significantly exceed those encountered during training. Future research may explore improved training strategies or the incorporation of cross-dimensional adaptation mechanisms to enhance model robustness. Overall, this study provides a novel perspective on the application of LLMs in many-task optimization.

\subsection{Performance Evaluation on MaTOP (RQ4)}
In this subsection, we integrate our proposed method into two many-task EAs and compare their performance with that of the original algorithms.

\subsubsection{MaTOP Setup}
For problem setup, we consider the first noise-free instance in the BBOB suite, comprising 24 test functions. Traditionally, many-task test sets are formed by grouping selected functions (after rotation and translation) based on criteria such as solution overlap and similarity of fitness values. Test sets can thus include overlapping, partially overlapping, or non-overlapping tasks~\cite{MTO_benchmark}, ensuring ``no free lunch"  scenario. Each function is expanded into four tasks by considering dimensionalities of 5, 10, 15, and 20, resulting in three many-task test benchmarks: MCF1 (overlapping), MCF2 (non-overlapping), and MCF3 (partially overlapping). Each benchmark includes 6 functions $\times$ 4 dimensions = 24 tasks.
Following past many-task benchmark studies, we divide the first-instance BBOB problems into multiple groups according to the overlap of their optima and the Spearman correlation of fitness values, as shown in~\autoref{tab:Benchmark}.

\subsubsection{Algorithm Settings}
To ensure a fair comparison, the maximum number of real FEs is set to 500.
Original algorithms terminate when the maximum number of real FEs reaches the FEs budget.
For the surrogate-based methods, all data are obtained via LHS prior to running the algorithm and serve as training samples. No additional data sampling is allowed during the optimization process. Consequently, the surrogate estimates population fitness. At the end, the best individual found (according to the surrogate) is evaluated once by the true objective function. 
Thus, the final solutions reflect the real problem's context and ensure meaningful performance evaluation.

Each problem is run independently 20 times, and we report the mean and standard deviation of results. We use the Wilcoxon rank-sum test~\cite{wilcoxon} with $\alpha=0.05$ to check for significant differences between methods; ``+'', ``$\approx$'', and ``-'' denote that the new method is superior, comparable, or inferior, respectively. The best result on each problem is highlighted in bold, and we also show the average ranking across all test sets.

\begin{table}[t]
  \centering
  \caption{Many-task Grouping for the BBOB Test Functions; MCF stands for Multiple Complex Functions. In each MCF instance, Task1-Task12 are grouped based on fitness overlap, while Task13-Task24 are grouped based on overlap of optimal solutions.}
      \label{tab:Benchmark}
  \resizebox{0.45\linewidth}{!}{
    {

\begin{tabular}{|l|l|}
\hline
\multicolumn{1}{|c|}{Name} & \multicolumn{1}{c|}{Function} \\ \hline
\multirow{6}{*}{MCF1}          &Task-1$\sim$4 = Buche\_Rastrigin, Dim=[5,10,15,20]                               \\ \cline{2-2} 
                           &Task-5$\sim$8 = Rosenbrock\_rotated, Dim=[5,10,15,20]                        \\ \cline{2-2} 
                           &Task-9$\sim$12 = Step\_Ellipsoidal, Dim=[5,10,15,20]                                \\ \cline{2-2} 
                           &Task-13$\sim$16 = Bent\_Cigar,Dim=[5,10,15,20]                                \\ \cline{2-2} 
                           &Task-17$\sim$20 = Rosenbrock\_original, Dim=[5,10,15,20]                                \\ \cline{2-2} 
                           &Task-17$\sim$24 = Rastrigin\_F15,Dim=[5,10,15,20]                                \\ \hline
\multirow{6}{*}{MCF2}          &Task-1$\sim$4 = Sharp\_Ridge, Dim=[5,10,15,20]                               \\ \cline{2-2} 
                           &Task-5$\sim$8 =Buche\_Rastrigin, Dim=[5,10,15,20]                        \\ \cline{2-2} 
                           &Task-9$\sim$12 =Different\_Powers,, Dim=[5,10,15,20]                                \\ \cline{2-2} 
                           &Task-13$\sim$16 =Sharp\_Ridge,Dim=[5,10,15,20]                                \\ \cline{2-2} 
                           &Task-16$\sim$20 =Schaffers, Dim=[5,10,15,20]                                \\ \cline{2-2} 
                           &Task-21$\sim$24 =Gallagher\_21Peaks,Dim=[5,10,15,20]                                \\ \hline
\multirow{6}{*}{MCF3}          &Task-1$\sim$4 = Step\_Ellipsoidal, Dim=[5,10,15,20]                               \\ \cline{2-2} 
                           &Task-5$\sim$8  = Composite\_Grie\_rosen, Dim=[5,10,15,20]                        \\ \cline{2-2} 
                           &Task-9$\sim$12 =Different\_Powersn, Dim=[5,10,15,20]                                \\ \cline{2-2} 
                           &Task-13$\sim$16 =Schwefel,Dim=[5,10,15,20]                                \\ \cline{2-2} 
                           &Task-17$\sim$20 =Gallagher\_101Peaks,, Dim=[5,10,15,20]                                \\ \cline{2-2} 
                           &Task-21$\sim$24 =Lunacek\_bi\_Rastrigin,Dim=[5,10,15,20]                                \\ \hline

\end{tabular}

    }
}
\end{table}

We evaluate meta-surrogate by integrating it into two backbone algorithms, MaTDE and BLKT-DE. 
In these integrations, actual function evaluations are replaced by predictions from offline surrogates 
(e.g., MetaSurrogate, RBFN, EvidentialMLP,\allowbreak FT-\allowbreak Transformer, SAINT, etc.), yielding offline DDEA variants such as MetaSurrogate\_\allowbreak MaTDE, RBFN\_\allowbreak MaTDE, etc. 
Analogously, we construct MetaSurrogate\_BLKT-DE, RBFN\_BLKT-DE, etc.

 A brief overview of the two backbone algorithms is as follows:
\begin{itemize}
    \item \textbf{MaTDE}~\cite{MaTDE} is a many-task optimization approach using multiple subpopulations. It employs an adaptive selection mechanism that determines which tasks can serve as ``auxiliary'' for a given task by considering the Kullback-Leibler similarity between tasks and cumulative transfer rewards. It also adopts a crossover-based transfer pattern to share knowledge across tasks.
    \item \textbf{BLKT-DE}~\cite{BLKT-DE} segments the individuals of all tasks into multiple blocks, where each block corresponds to a subset of contiguous dimensions. Blocks sharing similarities are grouped into the same cluster, facilitating knowledge transfer among tasks or dimensions that are either aligned or misaligned.
\end{itemize}

\subsubsection{Comparison Results} As shown in~\autoref{result_MaTDE_MCF1}, \autoref{result_BLKT-DE_MCF1}, ~\autoref{result_MaTDE_MCF2}, \autoref{result_BLKT-DE_MCF2},~\autoref{result_MaTDE_MCF3}, \autoref{result_BLKT-DE_MCF3}, under the same FE budget, many-task algorithms using meta-surrogate outperform both the original algorithms and those using other surrogates. This indicates that meta-surrogate is a promising and generic alternative to traditional surrogates for enhancing algorithm performance, particularly in many-task optimization environments where numerous tasks are solved concurrently. 

\begin{table}[htbp]
    \centering 
    \caption{
    Performance of MaTDE with different surrogate models on MCF1}
    \label{result_MaTDE_MCF1}
      \resizebox{\linewidth}{!}{
        \begin{tabular}{|c|c|c|c|c|c|c|c|}
\hline
Task & LLM\_MaTDE & MTGP\_MaTDE & RBFN\_MaTDE & EvidentialMLP\_MaTDE & FTTransformer\_MaTDE & SAINT\_MaTDE & MaTDE \\\hline
Task1 & 1.19E+03$\pm$1.52E+01 & 6.59E+03$\pm$1.35E+02(+) & 1.28E+03$\pm$6.42E+01($\approx$) & 1.17E+03$\pm$2.17E+00($\approx$) & 1.22E+03$\pm$8.13E+00(+) & 5.60E+03$\pm$7.37E+02(+) & \textbf{1.15E+03$\pm$1.17E+01(-)} \\\hline
Task2 & 1.32E+03$\pm$6.49E+01 & 9.20E+03$\pm$1.70E+02(+) & 1.46E+03$\pm$1.74E+01(+) & 4.05E+03$\pm$1.19E+02(+) & 1.70E+03$\pm$3.03E+01(+) & 2.38E+03$\pm$1.18E+02(+) & \textbf{1.31E+03$\pm$5.85E+01($\approx$)} \\\hline
Task3 & \textbf{2.66E+03$\pm$5.89E+01} & 3.99E+04$\pm$2.47E+04(+) & 3.08E+03$\pm$1.04E+02(+) & 9.32E+03$\pm$3.77E+02(+) & 2.98E+03$\pm$2.30E+01(+) & 5.94E+04$\pm$1.79E+04(+) & 3.00E+03$\pm$1.48E+02(+) \\\hline
Task4 & \textbf{1.67E+03$\pm$1.33E+02} & 2.69E+04$\pm$1.63E+04(+) & 1.85E+03$\pm$4.34E+01(+) & 7.42E+03$\pm$5.37E+02(+) & 2.07E+03$\pm$1.59E+02(+) & 5.11E+03$\pm$8.02E+02(+) & 2.49E+03$\pm$7.76E+02($\approx$) \\\hline
Task5 & \textbf{1.20E+03$\pm$1.12E+00} & 1.35E+03$\pm$1.54E+01(+) & 1.21E+03$\pm$3.83E+00(+) & 1.29E+03$\pm$1.12E+01(+) & 1.44E+03$\pm$1.32E+01(+) & 4.17E+03$\pm$1.79E+02(+) & 1.21E+03$\pm$2.27E+00(+) \\\hline
Task6 & \textbf{3.22E+02$\pm$4.62E+00} & 3.57E+02$\pm$4.77E+00(+) & 3.58E+02$\pm$6.99E+00(+) & 3.27E+02$\pm$3.80E-01($\approx$) & 1.18E+03$\pm$3.00E+02(+) & 1.78E+03$\pm$1.65E+02(+) & 3.62E+02$\pm$9.08E+00(+) \\\hline
Task7 & \textbf{1.25E+03$\pm$6.12E+00} & 3.30E+03$\pm$9.07E+02(+) & 1.89E+03$\pm$4.71E+01(+) & 1.46E+03$\pm$2.53E+01(+) & 4.77E+03$\pm$1.94E+02(+) & 2.92E+03$\pm$1.06E+03(+) & 1.33E+03$\pm$3.91E+01(+) \\\hline
Task8 & \textbf{1.78E+03$\pm$8.79E+01} & 4.23E+03$\pm$9.91E+02(+) & 2.64E+03$\pm$1.22E+02(+) & 2.28E+03$\pm$5.88E+01(+) & 4.36E+03$\pm$2.77E+02(+) & 4.18E+03$\pm$4.46E+02(+) & 1.92E+03$\pm$7.75E+01($\approx$) \\\hline
Task9 & 1.21E+03$\pm$6.75E+01 & 5.25E+04$\pm$9.27E+02(+) & 3.22E+03$\pm$7.12E+00(+) & 4.12E+03$\pm$7.95E+01(+) & 4.80E+04$\pm$3.32E+01(+) & 4.69E+05$\pm$4.83E+04(+) & \textbf{1.09E+03$\pm$1.71E+02($\approx$)} \\\hline
Task10 & 2.64E+03$\pm$5.22E+02 & 6.52E+03$\pm$4.39E+02(+) & \textbf{2.48E+03$\pm$1.35E+01($\approx$)} & 1.35E+04$\pm$7.68E+02(+) & 3.51E+04$\pm$5.56E+03(+) & 8.66E+05$\pm$5.11E+04(+) & 7.53E+03$\pm$2.26E+03(+) \\\hline
Task11 & 1.40E+04$\pm$1.48E+04 & 2.12E+05$\pm$1.48E+05(+) & \textbf{3.52E+03$\pm$7.33E+01(-)} & 2.09E+04$\pm$4.31E+03($\approx$) & 4.36E+04$\pm$6.47E+03(+) & 4.25E+05$\pm$4.71E+05(+) & 3.10E+04$\pm$1.08E+04($\approx$) \\\hline
Task12 & 3.39E+04$\pm$1.16E+04 & 7.14E+05$\pm$3.16E+05(+) & \textbf{1.61E+03$\pm$4.74E+01(-)} & 1.18E+05$\pm$1.91E+04(+) & 2.30E+05$\pm$9.65E+04(+) & 8.14E+05$\pm$4.30E+05(+) & 7.66E+04$\pm$1.30E+04(+) \\\hline
Task13 & \textbf{8.61E+02$\pm$2.66E+01} & 4.84E+03$\pm$4.05E+01(+) & 6.81E+03$\pm$4.02E+02(+) & 2.53E+03$\pm$1.72E+01(+) & 1.38E+04$\pm$3.68E+03(+) & 3.87E+05$\pm$2.35E+04(+) & 1.17E+03$\pm$1.34E+02(+) \\\hline
Task14 & \textbf{1.51E+03$\pm$6.91E+01} & 1.87E+03$\pm$4.20E+01(+) & 1.99E+03$\pm$3.16E+01(+) & 3.23E+03$\pm$5.45E+01(+) & 1.11E+05$\pm$7.25E+04(+) & 7.17E+05$\pm$1.06E+05(+) & 5.66E+03$\pm$1.29E+03(+) \\\hline
Task15 & 1.71E+03$\pm$4.20E+02 & 8.59E+04$\pm$5.16E+04(+) & \textbf{1.53E+03$\pm$4.30E+01($\approx$)} & 2.75E+03$\pm$1.33E+02(+) & 5.60E+05$\pm$2.96E+04(+) & 9.04E+05$\pm$1.48E+05(+) & 2.06E+04$\pm$6.73E+03(+) \\\hline
Task16 & 3.45E+03$\pm$9.61E+02 & 2.49E+05$\pm$9.75E+04(+) & \textbf{2.29E+03$\pm$8.69E+01($\approx$)} & 3.04E+03$\pm$5.16E+02($\approx$) & 1.30E+06$\pm$2.12E+05(+) & 7.85E+05$\pm$2.74E+05(+) & 4.69E+04$\pm$5.11E+03(+) \\\hline
Task17 & \textbf{1.40E+06$\pm$8.18E+05} & 7.87E+06$\pm$1.81E+03(+) & 4.57E+07$\pm$3.58E+07(+) & 1.36E+07$\pm$1.58E+05(+) & 7.87E+06$\pm$1.15E+04(+) & 1.25E+08$\pm$4.53E+06(+) & 1.92E+06$\pm$1.00E+06($\approx$) \\\hline
Task18 & \textbf{7.15E+06$\pm$6.51E+05} & 2.90E+07$\pm$1.86E+06(+) & 1.34E+08$\pm$4.96E+07(+) & 2.91E+07$\pm$5.74E+05(+) & 4.58E+07$\pm$3.34E+05(+) & 7.91E+08$\pm$8.45E+07(+) & 1.80E+07$\pm$3.72E+06(+) \\\hline
Task19 & \textbf{3.33E+07$\pm$8.67E+06} & 1.18E+08$\pm$5.17E+06(+) & 1.61E+08$\pm$2.35E+07(+) & 9.62E+07$\pm$2.82E+06(+) & 9.78E+07$\pm$5.42E+06(+) & 1.17E+09$\pm$1.36E+08(+) & 4.40E+07$\pm$7.09E+06($\approx$) \\\hline
Task20 & \textbf{6.73E+07$\pm$7.15E+06} & 7.63E+08$\pm$4.25E+08(+) & 2.28E+08$\pm$2.92E+07(+) & 9.64E+07$\pm$5.18E+06(+) & 1.67E+08$\pm$1.04E+07(+) & 1.38E+09$\pm$4.12E+08(+) & 1.17E+08$\pm$1.74E+07(+) \\\hline
Task21 & \textbf{1.47E+02$\pm$4.91E+00} & 2.05E+02$\pm$7.03E+00(+) & 1.88E+02$\pm$4.07E+00(+) & 2.06E+02$\pm$1.32E+00(+) & 2.71E+02$\pm$3.71E+01(+) & 5.24E+02$\pm$2.92E+01(+) & 1.50E+02$\pm$4.08E+00($\approx$) \\\hline
Task22 & \textbf{1.80E+03$\pm$1.64E+01} & 2.05E+03$\pm$1.77E+01(+) & 2.17E+03$\pm$5.79E+01(+) & 1.91E+03$\pm$9.81E+00(+) & 2.04E+03$\pm$7.24E+01(+) & 4.33E+03$\pm$1.58E+03(+) & 1.87E+03$\pm$2.33E+01(+) \\\hline
Task23 & \textbf{1.73E+03$\pm$1.37E+01} & 3.12E+03$\pm$4.91E+02(+) & 1.97E+03$\pm$9.90E+00(+) & 2.09E+03$\pm$2.14E+01(+) & 1.87E+03$\pm$3.29E+01(+) & 3.09E+03$\pm$6.58E+01(+) & 1.91E+03$\pm$2.81E+01(+) \\\hline
Task24 & \textbf{2.04E+03$\pm$2.77E+01} & 3.65E+03$\pm$8.06E+02(+) & 2.12E+03$\pm$3.37E+01(+) & 2.38E+03$\pm$2.40E+01(+) & 2.75E+03$\pm$6.34E+01(+) & 4.21E+03$\pm$6.45E+02(+) & 2.23E+03$\pm$8.22E+01(+) \\\hline
+/$\approx$/- & NA & 24/0/0 & 18/4/2 & 20/4/0 & 24/0/0 & 24/0/0 & 15/8/1 \\\hline
Average Rank & 1.42 & 5.12 & 3.46 & 3.75 & 4.96 & 6.54 & 2.75 \\\hline
Average Fitness & 4.55E+06 & 3.83E+07 & 2.37E+07 & 9.81E+06 & 1.34E+07 & 1.45E+08 & 7.54E+06 \\\hline
\end{tabular}
    }
\end{table}

\begin{table}[htbp]
    \centering
    \caption{
    Performance of BLKT-DE with different surrogate models on MCF1}
    \label{result_BLKT-DE_MCF1}
      \resizebox{\linewidth}{!}{
        \begin{tabular}{|c|c|c|c|c|c|c|c|}
\hline
Task & LLM\_BLKT-DE & MTGP\_BLKT-DE & RBFN\_BLKT-DE & EvidentialMLP\_BLKT-DE & FTTransformer\_BLKT-DE & SAINT\_BLKT-DE & BLKT-DE \\\hline
Task1 & \textbf{1.16E+03$\pm$1.52E+01} & 6.60E+03$\pm$4.77E+01(+) & 1.83E+03$\pm$4.05E+02(+) & 1.17E+03$\pm$5.57E+00($\approx$) & 1.35E+03$\pm$1.08E+02(+) & 1.45E+04$\pm$1.20E+02(+) & 1.17E+03$\pm$1.53E+01($\approx$) \\\hline
Task2 & 1.36E+03$\pm$3.48E+01 & 9.18E+03$\pm$1.00E+02(+) & 1.93E+03$\pm$4.29E+02(+) & 3.64E+03$\pm$5.35E+02(+) & 1.69E+03$\pm$2.22E+01(+) & 2.33E+03$\pm$7.12E+01(+) & \textbf{1.31E+03$\pm$1.02E+02($\approx$)} \\\hline
Task3 & \textbf{2.71E+03$\pm$7.56E+01} & 4.16E+04$\pm$1.46E+04(+) & 6.17E+03$\pm$8.11E+01(+) & 7.64E+03$\pm$9.00E+02(+) & 3.61E+03$\pm$6.14E+02(+) & 7.23E+04$\pm$1.66E+04(+) & 3.00E+03$\pm$1.80E+02(+) \\\hline
Task4 & 2.42E+03$\pm$7.90E+02 & 4.31E+04$\pm$1.78E+04(+) & \textbf{2.19E+03$\pm$9.75E+01($\approx$)} & 7.24E+03$\pm$1.01E+03(+) & 2.67E+03$\pm$4.44E+02($\approx$) & 3.07E+04$\pm$2.68E+04(+) & 2.40E+03$\pm$3.19E+02($\approx$) \\\hline
Task5 & \textbf{1.20E+03$\pm$6.40E-01} & 1.37E+03$\pm$1.37E+01(+) & 1.22E+03$\pm$0.00E+00(+) & 1.29E+03$\pm$1.12E+01(+) & 1.44E+03$\pm$0.00E+00(+) & 4.55E+03$\pm$1.90E+02(+) & 1.21E+03$\pm$5.79E+00(+) \\\hline
Task6 & 3.44E+02$\pm$1.61E+01 & 3.52E+02$\pm$6.00E+00($\approx$) & 3.71E+02$\pm$3.68E+00($\approx$) & \textbf{3.27E+02$\pm$6.98E+00($\approx$)} & 2.53E+03$\pm$1.13E+03(+) & 2.12E+03$\pm$9.27E+01(+) & 3.85E+02$\pm$2.49E+01($\approx$) \\\hline
Task7 & \textbf{1.30E+03$\pm$7.02E+01} & 4.28E+03$\pm$2.12E+03(+) & 2.45E+03$\pm$4.08E+01(+) & 1.45E+03$\pm$5.62E+01(+) & 5.81E+03$\pm$3.29E+02(+) & 3.89E+03$\pm$2.10E+03(+) & 1.36E+03$\pm$6.23E+01($\approx$) \\\hline
Task8 & \textbf{1.75E+03$\pm$1.04E+02} & 5.05E+03$\pm$1.33E+03(+) & 2.69E+03$\pm$4.82E+01(+) & 2.44E+03$\pm$2.45E+02(+) & 5.32E+03$\pm$4.38E+02(+) & 4.61E+03$\pm$7.77E+01(+) & 1.95E+03$\pm$1.23E+02(+) \\\hline
Task9 & 1.44E+03$\pm$2.29E+02 & 5.39E+04$\pm$2.87E+02(+) & 4.68E+03$\pm$2.93E+03(+) & 4.17E+03$\pm$1.08E+02(+) & 3.91E+04$\pm$1.80E+04(+) & 3.40E+05$\pm$1.96E+05(+) & \textbf{1.02E+03$\pm$7.27E+01(-)} \\\hline
Task10 & 3.71E+03$\pm$9.84E+02 & 6.38E+03$\pm$4.06E+02(+) & \textbf{2.51E+03$\pm$1.99E+01($\approx$)} & 1.11E+04$\pm$2.20E+03(+) & 4.15E+04$\pm$1.33E+04(+) & 6.59E+05$\pm$3.79E+05(+) & 1.45E+04$\pm$2.59E+03(+) \\\hline
Task11 & 2.25E+04$\pm$1.20E+04 & 6.55E+05$\pm$2.91E+05(+) & \textbf{3.55E+03$\pm$7.20E+01(-)} & 1.89E+04$\pm$3.24E+03($\approx$) & 7.17E+04$\pm$1.41E+04(+) & 5.17E+05$\pm$1.06E+05(+) & 4.87E+04$\pm$6.73E+03(+) \\\hline
Task12 & 4.41E+04$\pm$1.62E+04 & 9.01E+05$\pm$5.80E+05(+) & \textbf{1.63E+03$\pm$1.89E+02(-)} & 1.62E+05$\pm$3.15E+04(+) & 9.59E+04$\pm$2.62E+04(+) & 2.17E+06$\pm$7.88E+05(+) & 9.23E+04$\pm$4.58E+04($\approx$) \\\hline
Task13 & \textbf{9.07E+02$\pm$5.08E+01} & 4.86E+03$\pm$4.41E+01(+) & 6.97E+03$\pm$2.22E-01(+) & 2.46E+03$\pm$7.87E+01(+) & 2.34E+04$\pm$2.05E+04(+) & 3.93E+05$\pm$1.15E+04(+) & 1.00E+03$\pm$1.03E+02($\approx$) \\\hline
Task14 & \textbf{1.63E+03$\pm$8.61E+01} & 1.83E+03$\pm$4.03E+01(+) & 2.05E+03$\pm$2.74E+01(+) & 3.18E+03$\pm$7.05E+02(+) & 3.92E+05$\pm$3.09E+05(+) & 9.63E+05$\pm$1.35E+05(+) & 1.05E+04$\pm$3.26E+03(+) \\\hline
Task15 & 2.53E+03$\pm$8.89E+02 & 2.78E+05$\pm$2.07E+05(+) & \textbf{1.52E+03$\pm$7.26E+01(-)} & 2.74E+03$\pm$3.92E+02($\approx$) & 9.87E+05$\pm$1.23E+05(+) & 9.44E+05$\pm$6.40E+04(+) & 3.06E+04$\pm$5.83E+03(+) \\\hline
Task16 & 3.70E+03$\pm$5.79E+02 & 4.26E+05$\pm$1.68E+05(+) & \textbf{2.35E+03$\pm$2.48E+02(-)} & 6.63E+03$\pm$3.75E+03($\approx$) & 3.17E+06$\pm$6.65E+05(+) & 1.17E+06$\pm$1.93E+05(+) & 6.70E+04$\pm$1.42E+04(+) \\\hline
Task17 & \textbf{2.22E+06$\pm$8.98E+05} & 7.86E+06$\pm$3.87E+03(+) & 1.14E+08$\pm$4.33E+07(+) & 1.37E+07$\pm$3.12E+05(+) & 7.88E+06$\pm$3.08E+04(+) & 5.52E+08$\pm$2.73E+07(+) & 2.46E+06$\pm$5.31E+05($\approx$) \\\hline
Task18 & \textbf{8.34E+06$\pm$1.87E+06} & 3.03E+07$\pm$7.56E+05(+) & 1.76E+09$\pm$1.78E+09(+) & 2.81E+07$\pm$1.45E+06(+) & 5.71E+07$\pm$7.24E+06(+) & 2.85E+08$\pm$2.59E+07(+) & 2.18E+07$\pm$7.70E+06(+) \\\hline
Task19 & \textbf{5.54E+07$\pm$1.31E+07} & 1.25E+08$\pm$1.28E+07(+) & 1.25E+09$\pm$6.46E+08(+) & 9.43E+07$\pm$9.39E+06(+) & 1.09E+08$\pm$1.81E+07(+) & 9.90E+08$\pm$3.43E+08(+) & 6.00E+07$\pm$1.72E+07($\approx$) \\\hline
Task20 & \textbf{7.79E+07$\pm$1.31E+07} & 2.35E+09$\pm$2.75E+09(+) & 3.96E+08$\pm$1.55E+08(+) & 1.16E+08$\pm$9.12E+06(+) & 2.28E+08$\pm$3.41E+07(+) & 1.37E+09$\pm$4.74E+08(+) & 1.37E+08$\pm$1.55E+07(+) \\\hline
Task21 & 1.50E+02$\pm$9.65E+00 & 1.99E+02$\pm$9.55E+00(+) & 1.91E+02$\pm$1.29E-02(+) & 2.06E+02$\pm$2.97E+00(+) & 2.85E+02$\pm$3.87E+01(+) & 6.69E+02$\pm$7.24E+01(+) & \textbf{1.44E+02$\pm$1.12E+01($\approx$)} \\\hline
Task22 & \textbf{1.82E+03$\pm$1.07E+01} & 2.07E+03$\pm$3.46E+01(+) & 2.20E+03$\pm$9.21E+00(+) & 1.90E+03$\pm$6.47E+00(+) & 2.07E+03$\pm$7.67E+01(+) & 5.38E+03$\pm$2.75E+03(+) & 1.84E+03$\pm$1.01E+01(+) \\\hline
Task23 & \textbf{1.72E+03$\pm$1.62E+01} & 4.42E+03$\pm$2.18E+03(+) & 2.00E+03$\pm$1.81E+01(+) & 2.11E+03$\pm$6.91E+01(+) & 1.92E+03$\pm$5.16E+01(+) & 3.46E+03$\pm$8.56E+02(+) & 1.88E+03$\pm$2.93E+01(+) \\\hline
Task24 & \textbf{2.12E+03$\pm$2.42E+01} & 3.92E+03$\pm$1.38E+03(+) & 2.14E+03$\pm$4.28E+01($\approx$) & 2.44E+03$\pm$7.09E+01(+) & 3.16E+03$\pm$1.26E+02(+) & 7.16E+03$\pm$1.75E+03(+) & 2.32E+03$\pm$2.32E+01(+) \\\hline
+/$\approx$/- & NA & 23/1/0 & 16/4/4 & 19/5/0 & 23/1/0 & 24/0/0 & 13/10/1 \\\hline
Average Rank & 1.50 & 5.21 & 3.58 & 3.62 & 5.08 & 6.38 & 2.62 \\\hline
Average Fitness & 6.00E+06 & 1.05E+08 & 1.47E+08 & 1.05E+07 & 1.70E+07 & 1.34E+08 & 9.22E+06 \\\hline
\end{tabular}

    }
\end{table}

\begin{table}[htbp]
    \centering 
    \caption{
    Performance of MaTDE with different surrogate models on MCF2}
    \label{result_MaTDE_MCF2}
      \resizebox{\linewidth}{!}{
        \begin{tabular}{|c|c|c|c|c|c|c|c|}
\hline
Task & LLM\_MaTDE & MTGP\_MaTDE & RBFN\_MaTDE & EvidentialMLP\_MaTDE & FTTransformer\_MaTDE & SAINT\_MaTDE & MaTDE \\\hline
Task1 & \textbf{1.15E+03$\pm$1.27E+01} & 6.65E+03$\pm$6.14E+01(+) & 1.32E+03$\pm$1.85E+01(+) & 1.17E+03$\pm$3.62E+00($\approx$) & 1.22E+03$\pm$4.61E+00(+) & 9.37E+03$\pm$3.62E+03(+) & 1.16E+03$\pm$8.21E+00($\approx$) \\\hline
Task2 & 1.34E+03$\pm$6.15E+01 & 9.14E+03$\pm$1.99E+02(+) & 1.46E+03$\pm$4.98E+01(+) & 3.66E+03$\pm$3.56E+02(+) & 1.69E+03$\pm$4.12E+01(+) & 2.33E+03$\pm$6.19E+01(+) & \textbf{1.27E+03$\pm$5.76E+01($\approx$)} \\\hline
Task3 & \textbf{2.70E+03$\pm$4.25E+01} & 3.27E+04$\pm$1.96E+04(+) & 3.03E+03$\pm$4.36E+01(+) & 9.32E+03$\pm$6.53E+02(+) & 3.00E+03$\pm$2.16E+01(+) & 6.89E+04$\pm$1.99E+04(+) & 3.01E+03$\pm$1.50E+02(+) \\\hline
Task4 & 1.86E+03$\pm$6.19E+02 & 3.72E+04$\pm$1.83E+04(+) & \textbf{1.81E+03$\pm$6.04E+01($\approx$)} & 6.17E+03$\pm$1.07E+03(+) & 2.06E+03$\pm$4.85E+01($\approx$) & 8.77E+03$\pm$7.34E+03(+) & 1.86E+03$\pm$3.10E+02($\approx$) \\\hline
Task5 & 1.11E+03$\pm$3.51E+00 & 1.17E+03$\pm$1.62E-01(+) & \textbf{1.10E+03$\pm$6.26E-02(-)} & 1.19E+03$\pm$5.76E-01(+) & 1.22E+03$\pm$2.67E+00(+) & 1.26E+03$\pm$5.34E-01(+) & 1.12E+03$\pm$1.75E+00(+) \\\hline
Task6 & 2.04E+03$\pm$6.98E+00 & 2.22E+03$\pm$6.14E-01(+) & \textbf{2.01E+03$\pm$1.11E+00(-)} & 2.21E+03$\pm$8.02E-01(+) & 2.21E+03$\pm$2.99E+00(+) & 2.21E+03$\pm$2.77E+01(+) & 2.08E+03$\pm$1.22E+01(+) \\\hline
Task7 & 3.14E+02$\pm$9.73E+00 & 5.19E+02$\pm$7.29E+01(+) & \textbf{2.32E+02$\pm$2.56E+00(-)} & 4.25E+02$\pm$2.31E+00(+) & 2.42E+02$\pm$6.12E+00(-) & 4.19E+02$\pm$6.06E+00(+) & 3.41E+02$\pm$1.23E+01(+) \\\hline
Task8 & 2.27E+03$\pm$1.74E+01 & 2.49E+03$\pm$4.61E+01(+) & \textbf{2.15E+03$\pm$3.96E+00(-)} & 2.47E+03$\pm$2.93E+00(+) & 2.52E+03$\pm$1.79E+01(+) & 2.50E+03$\pm$8.38E+00(+) & 2.32E+03$\pm$4.55E+00(+) \\\hline
Task9 & \textbf{1.30E+03$\pm$3.80E-01} & 1.32E+03$\pm$6.71E-01(+) & 1.31E+03$\pm$1.29E-01(+) & 1.32E+03$\pm$1.59E+00(+) & 1.32E+03$\pm$1.05E+00(+) & 1.35E+03$\pm$2.35E+01(+) & 1.30E+03$\pm$8.62E-01(+) \\\hline
Task10 & \textbf{2.04E+02$\pm$6.25E-01} & 2.10E+02$\pm$7.59E-01(+) & 2.08E+02$\pm$3.03E-01(+) & 2.13E+02$\pm$8.96E-01(+) & 2.18E+02$\pm$6.84E+00(+) & 8.23E+02$\pm$2.51E+02(+) & 2.07E+02$\pm$7.38E-01(+) \\\hline
Task11 & \textbf{1.71E+03$\pm$9.10E-01} & 1.78E+03$\pm$6.25E+01(+) & 1.73E+03$\pm$4.63E+00(+) & 1.72E+03$\pm$3.14E+00(+) & 1.74E+03$\pm$4.32E+01(+) & 1.73E+03$\pm$6.01E+00(+) & 1.71E+03$\pm$9.09E-01($\approx$) \\\hline
Task12 & 6.12E+02$\pm$1.27E+00 & 6.48E+02$\pm$3.51E+01(+) & 6.13E+02$\pm$1.28E+00($\approx$) & 6.16E+02$\pm$1.92E+00(+) & 6.44E+02$\pm$1.18E+01(+) & 6.43E+02$\pm$1.14E+01(+) & \textbf{6.12E+02$\pm$7.42E-01($\approx$)} \\\hline
Task13 & 1.82E+03$\pm$2.43E+01 & 3.19E+03$\pm$4.21E+00(+) & \textbf{1.78E+03$\pm$8.54E-01(-)} & 2.56E+03$\pm$8.92E+00(+) & 2.83E+03$\pm$5.34E+01(+) & 3.36E+03$\pm$2.72E+02(+) & 1.89E+03$\pm$3.99E+01(+) \\\hline
Task14 & \textbf{2.83E+03$\pm$7.41E+01} & 4.45E+03$\pm$5.84E+00(+) & 3.05E+03$\pm$6.95E+00(+) & 4.03E+03$\pm$1.57E+01(+) & 4.38E+03$\pm$4.07E+02(+) & 4.92E+03$\pm$3.10E+02(+) & 3.12E+03$\pm$1.27E+02(+) \\\hline
Task15 & 1.38E+03$\pm$1.17E+02 & 2.69E+03$\pm$3.43E+02(+) & \textbf{9.58E+02$\pm$8.18E+00(-)} & 1.86E+03$\pm$3.28E+01(+) & 1.77E+03$\pm$4.90E+01(+) & 4.35E+03$\pm$9.82E+01(+) & 1.42E+03$\pm$1.16E+02($\approx$) \\\hline
Task16 & 1.91E+03$\pm$1.20E+02 & 4.11E+03$\pm$4.78E+02(+) & \textbf{1.39E+03$\pm$2.08E+01(-)} & 2.98E+03$\pm$2.56E+01(+) & 4.30E+03$\pm$2.27E+02(+) & 5.58E+03$\pm$1.74E+02(+) & 2.35E+03$\pm$1.60E+02(+) \\\hline
Task17 & \textbf{2.40E+03$\pm$2.97E-01} & 2.41E+03$\pm$8.14E-02(+) & 2.41E+03$\pm$4.61E-01(+) & 2.42E+03$\pm$2.61E-02(+) & 2.43E+03$\pm$2.13E+00(+) & 2.55E+03$\pm$1.33E+02(+) & 2.40E+03$\pm$5.56E-01(+) \\\hline
Task18 & 1.51E+03$\pm$2.84E-01 & 1.52E+03$\pm$9.97E-01(+) & 1.52E+03$\pm$2.37E+00(+) & 1.51E+03$\pm$3.79E-01(+) & 1.70E+03$\pm$2.81E+01(+) & 1.75E+03$\pm$1.05E+02(+) & \textbf{1.51E+03$\pm$1.57E+00($\approx$)} \\\hline
Task19 & \textbf{1.81E+03$\pm$2.78E+00} & 1.98E+03$\pm$8.48E+01(+) & 1.85E+03$\pm$3.87E+00(+) & 1.83E+03$\pm$2.11E+00(+) & 2.27E+03$\pm$4.92E+01(+) & 2.10E+03$\pm$5.91E+01(+) & 1.82E+03$\pm$3.27E+00(+) \\\hline
Task20 & \textbf{2.52E+03$\pm$1.73E+00} & 2.78E+03$\pm$1.58E+02(+) & 2.60E+03$\pm$1.30E+01(+) & 2.55E+03$\pm$3.70E+00(+) & 2.85E+03$\pm$1.12E+02(+) & 2.65E+03$\pm$6.07E+01(+) & 2.53E+03$\pm$2.96E+00(+) \\\hline
Task21 & 7.25E+02$\pm$1.98E+01 & 7.79E+02$\pm$4.75E-01(+) & 7.75E+02$\pm$5.57E+00(+) & 7.47E+02$\pm$1.62E+00($\approx$) & 7.76E+02$\pm$2.47E-01(+) & 7.81E+02$\pm$3.30E-01(+) & \textbf{7.08E+02$\pm$3.26E+00(-)} \\\hline
Task22 & 2.47E+03$\pm$2.95E+00 & 2.49E+03$\pm$7.44E-02(+) & 2.47E+03$\pm$3.62E+00($\approx$) & 2.48E+03$\pm$1.27E-01(+) & 2.49E+03$\pm$6.71E-03(+) & 2.49E+03$\pm$2.06E-01(+) & \textbf{2.45E+03$\pm$1.28E+01(-)} \\\hline
Task23 & 2.07E+03$\pm$1.83E+00 & 2.08E+03$\pm$2.12E+00(+) & \textbf{2.06E+03$\pm$6.15E+00(-)} & 2.08E+03$\pm$6.53E-01(+) & 2.09E+03$\pm$1.78E-02(+) & 2.09E+03$\pm$2.84E+00(+) & 2.07E+03$\pm$4.99E+00(-) \\\hline
Task24 & \textbf{1.17E+03$\pm$1.89E+00} & 1.19E+03$\pm$1.32E+00(+) & 1.18E+03$\pm$2.63E-01($\approx$) & 1.18E+03$\pm$4.84E-01(+) & 1.19E+03$\pm$1.90E-02(+) & 1.18E+03$\pm$2.50E+00($\approx$) & 1.17E+03$\pm$1.32E+00($\approx$) \\\hline
+/$\approx$/- & NA & 24/0/0 & 12/4/8 & 22/2/0 & 22/1/1 & 23/1/0 & 13/8/3 \\\hline
Average Rank & 1.71 & 5.62 & 2.58 & 4.25 & 5.42 & 6.21 & 2.21 \\\hline
Average Fitness & 1.64E+03 & 5.24E+03 & 1.63E+03 & 2.36E+03 & 1.96E+03 & 5.59E+03 & 1.68E+03 \\\hline
\end{tabular}
    }
\end{table}

\begin{table}[htbp]
    \centering
    \caption{
    Performance of BLKT-DE with different surrogate models on MCF2}
    \label{result_BLKT-DE_MCF2}
      \resizebox{\linewidth}{!}{
        \begin{tabular}{|c|c|c|c|c|c|c|c|}
\hline
Task & LLM\_BLKT-DE & MTGP\_BLKT-DE & RBFN\_BLKT-DE & EvidentialMLP\_BLKT-DE & FTTransformer\_BLKT-DE & SAINT\_BLKT-DE & BLKT-DE \\\hline
Task1 & 1.17E+03$\pm$2.25E+01 & 6.57E+03$\pm$5.09E+01(+) & 2.00E+03$\pm$3.31E+02(+) & 1.17E+03$\pm$6.39E+00($\approx$) & 1.35E+03$\pm$9.81E+01(+) & 1.44E+04$\pm$4.76E+00(+) & \textbf{1.16E+03$\pm$1.57E+01($\approx$)} \\\hline
Task2 & 1.33E+03$\pm$5.63E+01 & 9.25E+03$\pm$1.63E+02(+) & 1.99E+03$\pm$4.48E+02(+) & 3.58E+03$\pm$3.68E+02(+) & 1.72E+03$\pm$2.62E+01(+) & 2.92E+03$\pm$4.81E+00(+) & \textbf{1.28E+03$\pm$7.60E+01($\approx$)} \\\hline
Task3 & \textbf{2.83E+03$\pm$2.54E+02} & 3.77E+04$\pm$1.79E+04(+) & 6.21E+03$\pm$1.19E+02(+) & 7.39E+03$\pm$7.42E+02(+) & 4.51E+03$\pm$2.24E+03($\approx$) & 1.29E+05$\pm$2.92E+02(+) & 3.10E+03$\pm$2.02E+02($\approx$) \\\hline
Task4 & 2.09E+03$\pm$3.99E+02 & 3.76E+04$\pm$7.12E+03(+) & \textbf{2.06E+03$\pm$2.26E+01($\approx$)} & 7.28E+03$\pm$2.19E+03(+) & 2.94E+03$\pm$9.23E+02($\approx$) & 1.03E+05$\pm$6.46E+02(+) & 2.09E+03$\pm$1.86E+02($\approx$) \\\hline
Task5 & 1.11E+03$\pm$2.79E+00 & 1.17E+03$\pm$1.97E-01(+) & \textbf{1.10E+03$\pm$0.00E+00(-)} & 1.19E+03$\pm$8.13E-01(+) & 1.23E+03$\pm$4.91E+00(+) & 1.27E+03$\pm$2.89E-05(+) & 1.12E+03$\pm$4.60E+00($\approx$) \\\hline
Task6 & 2.02E+03$\pm$9.55E+00 & 2.23E+03$\pm$1.01E+01(+) & \textbf{2.00E+03$\pm$9.51E-03(-)} & 2.21E+03$\pm$2.55E+00(+) & 2.19E+03$\pm$2.73E+01(+) & 2.24E+03$\pm$7.39E+00(+) & 2.07E+03$\pm$8.63E+00(+) \\\hline
Task7 & 3.06E+02$\pm$1.48E+01 & 5.10E+02$\pm$5.90E+01(+) & 2.20E+02$\pm$2.70E-01(-) & 4.19E+02$\pm$9.50E+00(+) & \textbf{2.06E+02$\pm$7.31E+00(-)} & 4.07E+02$\pm$1.27E-01(+) & 3.44E+02$\pm$1.91E+01(+) \\\hline
Task8 & 2.25E+03$\pm$1.84E+01 & 2.50E+03$\pm$5.08E+01(+) & \textbf{2.12E+03$\pm$7.45E-01(-)} & 2.48E+03$\pm$9.77E+00(+) & 2.47E+03$\pm$2.16E+01(+) & 2.61E+03$\pm$1.31E+02(+) & 2.32E+03$\pm$2.71E+01(+) \\\hline
Task9 & 1.30E+03$\pm$1.64E+00 & 1.32E+03$\pm$3.67E-01(+) & 1.33E+03$\pm$1.24E+01(+) & 1.32E+03$\pm$9.37E-01(+) & 1.32E+03$\pm$2.40E+00(+) & 1.32E+03$\pm$1.65E-02(+) & \textbf{1.30E+03$\pm$6.43E-01($\approx$)} \\\hline
Task10 & \textbf{2.05E+02$\pm$8.65E-01} & 2.10E+02$\pm$1.39E+00(+) & 2.27E+02$\pm$8.44E+00(+) & 2.13E+02$\pm$1.09E+00(+) & 2.36E+02$\pm$4.92E+00(+) & 4.01E+02$\pm$8.77E+00(+) & 2.06E+02$\pm$1.11E+00($\approx$) \\\hline
Task11 & \textbf{1.71E+03$\pm$6.83E-01} & 1.77E+03$\pm$3.30E+01(+) & 1.75E+03$\pm$2.03E+01(+) & 1.72E+03$\pm$4.96E-01(+) & 1.78E+03$\pm$4.66E+01(+) & 1.73E+03$\pm$7.14E-01(+) & 1.71E+03$\pm$1.15E+00(+) \\\hline
Task12 & \textbf{6.12E+02$\pm$2.28E+00} & 6.54E+02$\pm$2.39E+01(+) & 6.15E+02$\pm$6.28E-01(+) & 6.17E+02$\pm$3.02E+00(+) & 6.68E+02$\pm$1.03E+01(+) & 6.84E+02$\pm$1.88E+00(+) & 6.12E+02$\pm$1.34E+00($\approx$) \\\hline
Task13 & 1.84E+03$\pm$3.69E+01 & 3.19E+03$\pm$4.87E+00(+) & \textbf{1.78E+03$\pm$6.58E-02(-)} & 2.56E+03$\pm$1.21E+01(+) & 2.85E+03$\pm$2.63E+02(+) & 3.52E+03$\pm$1.83E-01(+) & 1.85E+03$\pm$4.46E+01($\approx$) \\\hline
Task14 & \textbf{2.86E+03$\pm$3.65E+01} & 4.46E+03$\pm$2.14E+01(+) & 3.05E+03$\pm$2.21E+00(+) & 4.01E+03$\pm$3.12E+01(+) & 5.56E+03$\pm$1.97E+02(+) & 5.91E+03$\pm$3.85E+02(+) & 3.25E+03$\pm$2.16E+02(+) \\\hline
Task15 & 1.48E+03$\pm$5.63E+01 & 2.94E+03$\pm$3.64E+02(+) & \textbf{9.70E+02$\pm$2.46E+01(-)} & 1.86E+03$\pm$8.15E+01(+) & 2.81E+03$\pm$5.87E+02(+) & 5.03E+03$\pm$1.87E+01(+) & 1.63E+03$\pm$1.03E+02($\approx$) \\\hline
Task16 & 1.96E+03$\pm$1.01E+02 & 4.41E+03$\pm$2.11E+02(+) & \textbf{1.54E+03$\pm$5.06E+01(-)} & 2.94E+03$\pm$7.85E+01(+) & 5.11E+03$\pm$2.11E+02(+) & 6.34E+03$\pm$1.57E+01(+) & 2.29E+03$\pm$2.16E+02($\approx$) \\\hline
Task17 & \textbf{2.40E+03$\pm$2.30E-01} & 2.41E+03$\pm$7.39E-02(+) & 2.41E+03$\pm$1.00E-06(+) & 2.42E+03$\pm$8.07E-02(+) & 2.60E+03$\pm$2.04E+02(+) & 2.48E+03$\pm$5.02E+01(+) & 2.40E+03$\pm$4.14E-01(+) \\\hline
Task18 & 1.51E+03$\pm$1.40E+00 & 1.59E+03$\pm$9.12E+01(+) & 1.53E+03$\pm$8.68E+00(+) & 1.51E+03$\pm$1.91E-01($\approx$) & 1.76E+03$\pm$3.19E+01(+) & 1.92E+03$\pm$6.47E-02(+) & \textbf{1.51E+03$\pm$2.74E+00($\approx$)} \\\hline
Task19 & \textbf{1.81E+03$\pm$3.19E+00} & 1.97E+03$\pm$9.36E+01(+) & 1.88E+03$\pm$2.90E+00(+) & 1.82E+03$\pm$3.48E+00(+) & 2.49E+03$\pm$1.26E+02(+) & 2.30E+03$\pm$2.72E+00(+) & 1.82E+03$\pm$4.72E+00($\approx$) \\\hline
Task20 & \textbf{2.52E+03$\pm$2.57E+00} & 2.86E+03$\pm$2.35E+02(+) & 2.67E+03$\pm$1.80E+01(+) & 2.56E+03$\pm$9.26E+00(+) & 3.14E+03$\pm$1.29E+02(+) & 2.82E+03$\pm$8.81E-01(+) & 2.53E+03$\pm$5.74E+00(+) \\\hline
Task21 & 7.28E+02$\pm$1.06E+01 & 7.79E+02$\pm$5.89E-01(+) & 7.82E+02$\pm$0.00E+00(+) & 7.48E+02$\pm$3.23E+00(+) & 7.79E+02$\pm$4.58E+00(+) & 7.82E+02$\pm$1.06E-06(+) & \textbf{7.05E+02$\pm$2.81E+00(-)} \\\hline
Task22 & 2.47E+03$\pm$8.10E+00 & 2.49E+03$\pm$4.00E-02(+) & 2.48E+03$\pm$1.00E-02(+) & 2.48E+03$\pm$3.28E-01(+) & 2.49E+03$\pm$1.98E-01(+) & 2.49E+03$\pm$1.87E-03(+) & \textbf{2.45E+03$\pm$1.49E+01(-)} \\\hline
Task23 & 2.07E+03$\pm$1.79E+00 & 2.09E+03$\pm$1.23E+00(+) & 2.09E+03$\pm$1.76E-03(+) & 2.08E+03$\pm$7.58E-01(+) & 2.09E+03$\pm$8.53E-03(+) & 2.09E+03$\pm$8.04E-05(+) & \textbf{2.07E+03$\pm$3.54E+00($\approx$)} \\\hline
Task24 & \textbf{1.18E+03$\pm$2.02E+00} & 1.19E+03$\pm$1.20E+00(+) & 1.18E+03$\pm$4.28E+00($\approx$) & 1.18E+03$\pm$7.70E-01(+) & 1.19E+03$\pm$7.83E-03(+) & 1.19E+03$\pm$9.63E-03(+) & 1.18E+03$\pm$2.06E+00($\approx$) \\\hline
+/$\approx$/- & NA & 24/0/0 & 15/2/7 & 22/2/0 & 21/2/1 & 24/0/0 & 7/15/2 \\\hline
Average Rank & 1.75 & 5.21 & 3.21 & 4.08 & 5.25 & 6.42 & 2.08 \\\hline
Average Fitness & 1.66E+03 & 5.49E+03 & 1.83E+03 & 2.32E+03 & 2.23E+03 & 1.24E+04 & 1.71E+03 \\\hline
\end{tabular}

    }
\end{table}

\begin{table}[htbp]
    \centering 
    \caption{
    Performance of MaTDE with different surrogate models on MCF3}
    \label{result_MaTDE_MCF3}
      \resizebox{\linewidth}{!}{
        \begin{tabular}{|c|c|c|c|c|c|c|c|}
\hline
Task & LLM\_MaTDE & MTGP\_MaTDE & RBFN\_MaTDE & EvidentialMLP\_MaTDE & FTTransformer\_MaTDE & SAINT\_MaTDE & MaTDE \\\hline
Task1 & \textbf{1.20E+03$\pm$1.87E+00} & 1.36E+03$\pm$1.20E+01(+) & 1.21E+03$\pm$1.85E+00(+) & 1.30E+03$\pm$9.17E+00(+) & 1.44E+03$\pm$1.16E+01(+) & 4.07E+03$\pm$3.25E+02(+) & 1.21E+03$\pm$2.43E+00(+) \\\hline
Task2 & 3.28E+02$\pm$5.43E+00 & 3.59E+02$\pm$3.74E+00(+) & 3.55E+02$\pm$7.40E+00(+) & \textbf{3.27E+02$\pm$3.56E+00($\approx$)} & 1.18E+03$\pm$2.77E+02(+) & 1.79E+03$\pm$2.51E+01(+) & 3.76E+02$\pm$1.14E+01(+) \\\hline
Task3 & \textbf{1.27E+03$\pm$7.90E+01} & 2.80E+03$\pm$8.70E+02(+) & 1.91E+03$\pm$5.62E+01(+) & 1.49E+03$\pm$3.26E+01(+) & 4.63E+03$\pm$1.40E+02(+) & 2.48E+03$\pm$1.78E+02(+) & 1.37E+03$\pm$4.07E+01($\approx$) \\\hline
Task4 & \textbf{1.68E+03$\pm$8.72E+01} & 4.18E+03$\pm$1.33E+03(+) & 2.61E+03$\pm$8.94E+01(+) & 2.29E+03$\pm$3.82E+01(+) & 4.75E+03$\pm$2.59E+02(+) & 4.21E+03$\pm$4.86E+02(+) & 1.90E+03$\pm$5.45E+01(+) \\\hline
Task5 & \textbf{2.40E+03$\pm$3.89E-01} & 2.41E+03$\pm$5.79E-02(+) & 2.41E+03$\pm$1.56E+00(+) & 2.42E+03$\pm$7.36E-02(+) & 2.43E+03$\pm$2.23E+00(+) & 2.48E+03$\pm$6.73E+01(+) & 2.40E+03$\pm$3.85E-01(+) \\\hline
Task6 & \textbf{1.51E+03$\pm$1.07E+00} & 1.52E+03$\pm$1.07E+00(+) & 1.52E+03$\pm$3.20E+00(+) & 1.51E+03$\pm$2.38E-01(+) & 1.70E+03$\pm$2.09E+01(+) & 1.74E+03$\pm$1.02E+02(+) & 1.51E+03$\pm$1.83E+00($\approx$) \\\hline
Task7 & \textbf{1.81E+03$\pm$1.10E+00} & 2.05E+03$\pm$1.86E+02(+) & 1.85E+03$\pm$4.89E+00(+) & 1.83E+03$\pm$2.61E+00(+) & 2.26E+03$\pm$4.74E+01(+) & 2.11E+03$\pm$1.47E+01(+) & 1.82E+03$\pm$2.84E+00(+) \\\hline
Task8 & \textbf{2.52E+03$\pm$4.01E+00} & 2.84E+03$\pm$3.72E+02(+) & 2.60E+03$\pm$2.63E+01(+) & 2.55E+03$\pm$2.71E+00(+) & 2.90E+03$\pm$1.16E+02(+) & 2.69E+03$\pm$5.10E+01(+) & 2.53E+03$\pm$3.33E+00(+) \\\hline
Task9 & 1.42E+03$\pm$6.15E+00 & 1.46E+03$\pm$3.76E-01(+) & 1.45E+03$\pm$2.82E-01(+) & 1.44E+03$\pm$1.41E+00(+) & 1.44E+03$\pm$2.70E+00(+) & 1.48E+03$\pm$3.98E-01(+) & \textbf{1.40E+03$\pm$1.95E+00(-)} \\\hline
Task10 & 7.46E+02$\pm$9.90E+00 & 7.70E+02$\pm$1.38E+00(+) & 7.70E+02$\pm$1.83E-01(+) & 7.68E+02$\pm$1.22E-01(+) & 7.75E+02$\pm$1.23E+00(+) & 7.80E+02$\pm$3.73E+00(+) & \textbf{7.33E+02$\pm$4.85E+00($\approx$)} \\\hline
Task11 & 2.47E+03$\pm$3.32E+00 & 2.48E+03$\pm$3.56E+00(+) & 2.48E+03$\pm$9.43E-01(+) & 2.48E+03$\pm$6.77E-01(+) & 2.49E+03$\pm$8.00E-02(+) & 2.49E+03$\pm$1.43E-01(+) & \textbf{2.46E+03$\pm$8.60E+00(-)} \\\hline
Task12 & 6.71E+02$\pm$5.24E+00 & 6.84E+02$\pm$1.19E+00(+) & \textbf{6.67E+02$\pm$7.39E-01($\approx$)} & 6.72E+02$\pm$1.79E+00($\approx$) & 6.86E+02$\pm$1.27E-01(+) & 6.85E+02$\pm$1.11E+00(+) & 6.67E+02$\pm$3.79E+00($\approx$) \\\hline
Task13 & 4.04E+02$\pm$1.66E+00 & 4.19E+02$\pm$2.33E+00(+) & 4.11E+02$\pm$2.89E+00(+) & 4.08E+02$\pm$3.26E+00($\approx$) & 4.25E+02$\pm$2.65E+00(+) & 7.23E+02$\pm$4.78E+01(+) & \textbf{4.03E+02$\pm$8.56E-01($\approx$)} \\\hline
Task14 & \textbf{1.71E+03$\pm$5.19E-01} & 1.71E+03$\pm$2.72E+00($\approx$) & 1.71E+03$\pm$2.16E+00(+) & 1.71E+03$\pm$2.81E+00($\approx$) & 1.85E+03$\pm$1.30E+01(+) & 1.93E+03$\pm$3.64E+01(+) & 1.71E+03$\pm$1.21E+00(+) \\\hline
Task15 & \textbf{2.01E+03$\pm$1.34E+00} & 2.04E+03$\pm$2.46E+01(+) & 2.01E+03$\pm$1.58E+00(+) & 2.01E+03$\pm$1.31E+00(+) & 2.25E+03$\pm$5.71E+01(+) & 2.06E+03$\pm$1.87E+01(+) & 2.01E+03$\pm$1.17E+00(+) \\\hline
Task16 & \textbf{9.08E+02$\pm$8.35E-01} & 9.53E+02$\pm$1.95E+01(+) & 9.10E+02$\pm$1.26E+00(+) & 9.09E+02$\pm$1.30E+00($\approx$) & 1.05E+03$\pm$2.26E+01(+) & 1.04E+03$\pm$3.56E+01(+) & 9.14E+02$\pm$1.38E+00(+) \\\hline
Task17 & \textbf{2.03E+02$\pm$1.91E-01} & 5.79E+04$\pm$1.20E+03(+) & 6.26E+02$\pm$1.11E+00(+) & 5.22E+03$\pm$4.89E+01(+) & 1.11E+04$\pm$4.11E+01(+) & 1.06E+05$\pm$1.41E+04(+) & 2.16E+02$\pm$1.92E+01($\approx$) \\\hline
Task18 & \textbf{2.10E+03$\pm$1.46E-01} & 2.27E+04$\pm$8.89E+02(+) & 3.70E+03$\pm$2.22E+02(+) & 1.53E+04$\pm$8.95E+02(+) & 1.00E+05$\pm$2.26E+03(+) & 1.80E+05$\pm$9.48E+03(+) & 5.13E+03$\pm$9.62E+02(+) \\\hline
Task19 & \textbf{1.35E+03$\pm$4.20E+02} & 1.75E+05$\pm$7.52E+04(+) & 2.86E+04$\pm$1.69E+03(+) & 2.32E+03$\pm$6.90E+02(+) & 1.50E+05$\pm$8.14E+03(+) & 3.03E+05$\pm$2.29E+04(+) & 1.08E+04$\pm$2.90E+03(+) \\\hline
Task20 & \textbf{2.90E+03$\pm$3.08E+02} & 1.72E+05$\pm$8.40E+04(+) & 8.40E+04$\pm$5.38E+03(+) & 1.01E+04$\pm$2.57E+03(+) & 1.74E+05$\pm$9.45E+03(+) & 7.84E+04$\pm$1.69E+04(+) & 2.42E+04$\pm$2.75E+03(+) \\\hline
Task21 & 2.05E+03$\pm$2.18E+01 & 2.07E+03$\pm$1.33E+01($\approx$) & 2.08E+03$\pm$7.02E-02($\approx$) & 2.09E+03$\pm$3.70E+00(+) & 2.09E+03$\pm$1.67E+01(+) & 2.24E+03$\pm$5.14E+01(+) & \textbf{2.04E+03$\pm$5.27E+00($\approx$)} \\\hline
Task22 & \textbf{1.30E+03$\pm$1.42E+01} & 1.36E+03$\pm$2.50E+01(+) & 1.32E+03$\pm$2.45E+01($\approx$) & 1.35E+03$\pm$1.52E+01(+) & 1.77E+03$\pm$1.88E+01(+) & 1.82E+03$\pm$7.91E+01(+) & 1.33E+03$\pm$6.69E+00(+) \\\hline
Task23 & 1.51E+03$\pm$1.40E+01 & 1.83E+03$\pm$1.22E+02(+) & 1.50E+03$\pm$2.76E+01($\approx$) & \textbf{1.49E+03$\pm$3.66E+01($\approx$)} & 2.09E+03$\pm$2.91E+01(+) & 2.27E+03$\pm$3.51E+01(+) & 1.55E+03$\pm$3.30E+01($\approx$) \\\hline
Task24 & 2.09E+03$\pm$1.44E+01 & 2.52E+03$\pm$1.30E+02(+) & \textbf{2.02E+03$\pm$2.11E+01(-)} & 2.11E+03$\pm$4.43E+01($\approx$) & 2.91E+03$\pm$4.67E+01(+) & 2.50E+03$\pm$2.56E+02(+) & 2.19E+03$\pm$2.07E+01(+) \\\hline
+/$\approx$/- & NA & 22/2/0 & 19/4/1 & 17/7/0 & 24/0/0 & 24/0/0 & 14/8/2 \\\hline
Average Rank & 1.46 & 5.00 & 3.29 & 3.21 & 6.21 & 6.38 & 2.46 \\\hline
Average Fitness & 1.52E+03 & 1.93E+04 & 6.20E+03 & 2.67E+03 & 1.99E+04 & 2.96E+04 & 2.95E+03 \\\hline
\end{tabular}
    }
\end{table}

\begin{table}[htbp]
    \centering
    \caption{
    Performance of BLKT-DE with different surrogate models on MCF3}
    \label{result_BLKT-DE_MCF3}
      \resizebox{\linewidth}{!}{
        \begin{tabular}{|c|c|c|c|c|c|c|c|}
\hline
Task & LLM\_BLKT-DE & MTGP\_BLKT-DE & RBFN\_BLKT-DE & EvidentialMLP\_BLKT-DE & FTTransformer\_BLKT-DE & SAINT\_BLKT-DE & BLKT-DE \\\hline
Task1 & \textbf{1.20E+03$\pm$8.59E-01} & 1.36E+03$\pm$1.10E+01(+) & 1.22E+03$\pm$1.23E+00(+) & 1.29E+03$\pm$1.12E+01(+) & 1.54E+03$\pm$2.54E+02(+) & 5.47E+03$\pm$7.30E+02(+) & 1.21E+03$\pm$2.49E+00(+) \\\hline
Task2 & 3.27E+02$\pm$3.94E+00 & 3.54E+02$\pm$8.78E+00(+) & 3.74E+02$\pm$4.02E+00(+) & \textbf{3.22E+02$\pm$4.70E+00($\approx$)} & 3.43E+03$\pm$7.18E+02(+) & 1.99E+03$\pm$8.01E+02(+) & 3.75E+02$\pm$1.59E+01(+) \\\hline
Task3 & \textbf{1.27E+03$\pm$3.36E+01} & 2.94E+03$\pm$4.88E+02(+) & 2.49E+03$\pm$8.23E+01(+) & 1.51E+03$\pm$2.65E+01(+) & 5.92E+03$\pm$3.15E+02(+) & 2.72E+03$\pm$1.91E+00(+) & 1.38E+03$\pm$5.30E+01(+) \\\hline
Task4 & \textbf{1.78E+03$\pm$7.50E+01} & 3.83E+03$\pm$4.81E+02(+) & 2.66E+03$\pm$6.91E+01(+) & 2.37E+03$\pm$8.22E+01(+) & 5.88E+03$\pm$5.22E+02(+) & 5.06E+03$\pm$8.84E+00(+) & 1.89E+03$\pm$1.24E+02($\approx$) \\\hline
Task5 & \textbf{2.40E+03$\pm$2.82E-01} & 2.41E+03$\pm$4.40E-02(+) & 2.41E+03$\pm$5.86E-07(+) & 2.42E+03$\pm$5.11E-02(+) & 2.62E+03$\pm$1.96E+02(+) & 2.43E+03$\pm$6.70E-01(+) & 2.40E+03$\pm$6.43E-01(+) \\\hline
Task6 & \textbf{1.51E+03$\pm$1.17E+00} & 1.56E+03$\pm$7.62E+01(+) & 1.53E+03$\pm$9.34E+00(+) & 1.51E+03$\pm$6.79E-01($\approx$) & 1.79E+03$\pm$9.31E+00(+) & 1.67E+03$\pm$1.56E+01(+) & 1.51E+03$\pm$2.25E+00($\approx$) \\\hline
Task7 & \textbf{1.81E+03$\pm$1.90E+00} & 2.00E+03$\pm$1.66E+02(+) & 1.88E+03$\pm$2.12E+00(+) & 1.83E+03$\pm$2.44E+00(+) & 2.40E+03$\pm$5.86E+01(+) & 2.24E+03$\pm$1.03E+02(+) & 1.82E+03$\pm$5.92E+00(+) \\\hline
Task8 & \textbf{2.52E+03$\pm$3.30E+00} & 2.66E+03$\pm$8.42E+01(+) & 2.67E+03$\pm$3.92E+00(+) & 2.56E+03$\pm$7.62E+00(+) & 3.05E+03$\pm$1.13E+02(+) & 2.69E+03$\pm$2.91E+01(+) & 2.54E+03$\pm$3.93E+00(+) \\\hline
Task9 & 1.42E+03$\pm$4.80E+00 & 1.46E+03$\pm$1.91E-01(+) & 1.45E+03$\pm$4.76E-03(+) & 1.44E+03$\pm$2.80E+00(+) & 1.44E+03$\pm$1.21E+01(+) & 1.48E+03$\pm$1.39E+00(+) & \textbf{1.41E+03$\pm$1.86E+00(-)} \\\hline
Task10 & \textbf{7.33E+02$\pm$1.82E+01} & 7.74E+02$\pm$2.15E+00(+) & 7.70E+02$\pm$2.28E-01(+) & 7.66E+02$\pm$1.48E+00($\approx$) & 7.81E+02$\pm$3.07E+00(+) & 7.75E+02$\pm$5.80E+00(+) & 7.47E+02$\pm$3.82E+00($\approx$) \\\hline
Task11 & 2.47E+03$\pm$8.90E-01 & 2.48E+03$\pm$1.55E+00(+) & 2.48E+03$\pm$1.73E-01(+) & 2.48E+03$\pm$9.28E-01(+) & 2.49E+03$\pm$4.38E-03(+) & 2.49E+03$\pm$1.01E-01(+) & \textbf{2.46E+03$\pm$8.53E+00(-)} \\\hline
Task12 & 6.69E+02$\pm$6.18E+00 & 6.86E+02$\pm$4.96E-01(+) & \textbf{6.67E+02$\pm$9.35E-01($\approx$)} & 6.72E+02$\pm$2.72E+00($\approx$) & 6.87E+02$\pm$1.06E-02(+) & 6.86E+02$\pm$8.44E-02(+) & 6.74E+02$\pm$2.08E+00($\approx$) \\\hline
Task13 & \textbf{4.02E+02$\pm$6.13E-01} & 4.17E+02$\pm$2.58E+00(+) & 4.10E+02$\pm$4.55E-01(+) & 4.10E+02$\pm$1.78E+00(+) & 4.33E+02$\pm$1.23E+01(+) & 7.01E+02$\pm$7.47E+01(+) & 4.04E+02$\pm$1.66E+00($\approx$) \\\hline
Task14 & \textbf{1.70E+03$\pm$7.25E-01} & 1.71E+03$\pm$1.67E+00(+) & 1.71E+03$\pm$1.40E+00(+) & 1.71E+03$\pm$3.59E+00(+) & 1.98E+03$\pm$6.62E+01(+) & 2.00E+03$\pm$1.15E+02(+) & 1.71E+03$\pm$1.16E+00(+) \\\hline
Task15 & \textbf{2.01E+03$\pm$1.55E+00} & 2.07E+03$\pm$2.44E+01(+) & 2.01E+03$\pm$2.00E+00(+) & 2.01E+03$\pm$5.65E-01(+) & 2.46E+03$\pm$2.61E+02(+) & 2.18E+03$\pm$6.44E+01(+) & 2.01E+03$\pm$2.74E+00(+) \\\hline
Task16 & \textbf{9.08E+02$\pm$9.62E-01} & 9.66E+02$\pm$1.98E+01(+) & 9.10E+02$\pm$5.61E-01(+) & 9.11E+02$\pm$1.23E+00(+) & 1.22E+03$\pm$6.33E+01(+) & 1.14E+03$\pm$7.50E+01(+) & 9.17E+02$\pm$1.80E+00(+) \\\hline
Task17 & \textbf{2.03E+02$\pm$1.30E-01} & 5.71E+04$\pm$1.27E+03(+) & 6.27E+02$\pm$4.62E-01(+) & 5.13E+03$\pm$2.68E+02(+) & 3.06E+04$\pm$2.38E+04(+) & 1.05E+05$\pm$1.98E+03(+) & 2.12E+02$\pm$1.59E+01(+) \\\hline
Task18 & \textbf{2.10E+03$\pm$2.92E-01} & 3.72E+04$\pm$2.85E+04(+) & 3.65E+03$\pm$3.22E+01(+) & 1.35E+04$\pm$3.06E+03(+) & 1.05E+05$\pm$2.43E+03(+) & 1.91E+05$\pm$1.57E+04(+) & 5.44E+03$\pm$1.23E+03(+) \\\hline
Task19 & \textbf{1.56E+03$\pm$3.87E+02} & 1.36E+05$\pm$4.43E+04(+) & 3.10E+04$\pm$9.30E+02(+) & 3.04E+03$\pm$1.75E+03($\approx$) & 1.86E+05$\pm$3.48E+03(+) & 4.86E+05$\pm$2.68E+04(+) & 1.27E+04$\pm$6.04E+03(+) \\\hline
Task20 & \textbf{3.87E+03$\pm$7.87E+02} & 2.88E+05$\pm$1.80E+05(+) & 9.89E+04$\pm$5.14E+03(+) & 1.48E+04$\pm$2.21E+03(+) & 2.26E+05$\pm$1.36E+04(+) & 1.74E+05$\pm$3.19E+04(+) & 3.51E+04$\pm$8.47E+03(+) \\\hline
Task21 & 2.05E+03$\pm$1.04E+01 & 2.06E+03$\pm$1.22E+01($\approx$) & 2.08E+03$\pm$2.32E-01(+) & 2.09E+03$\pm$7.47E+00(+) & 2.12E+03$\pm$2.17E+01(+) & 2.26E+03$\pm$7.72E+00(+) & \textbf{2.04E+03$\pm$9.66E+00($\approx$)} \\\hline
Task22 & \textbf{1.30E+03$\pm$1.82E+01} & 1.37E+03$\pm$1.75E+01(+) & 1.60E+03$\pm$6.22E+00(+) & 1.35E+03$\pm$1.05E+01(+) & 1.92E+03$\pm$7.24E+01(+) & 1.74E+03$\pm$7.49E+01(+) & 1.34E+03$\pm$1.47E+01(+) \\\hline
Task23 & 1.51E+03$\pm$2.71E+01 & 2.02E+03$\pm$9.02E+01(+) & \textbf{1.48E+03$\pm$2.29E+01($\approx$)} & 1.53E+03$\pm$1.75E+01($\approx$) & 2.26E+03$\pm$3.49E+01(+) & 2.27E+03$\pm$1.33E+02(+) & 1.58E+03$\pm$1.99E+01(+) \\\hline
Task24 & 2.12E+03$\pm$2.57E+01 & 2.58E+03$\pm$1.53E+01(+) & \textbf{2.03E+03$\pm$1.71E+01(-)} & 2.12E+03$\pm$2.76E+01($\approx$) & 3.34E+03$\pm$1.27E+02(+) & 3.03E+03$\pm$9.52E+01(+) & 2.23E+03$\pm$6.13E+01($\approx$) \\\hline
+/$\approx$/- & NA & 23/1/0 & 21/2/1 & 17/7/0 & 24/0/0 & 24/0/0 & 15/7/2 \\\hline
Average Rank & 1.29 & 4.83 & 3.42 & 3.08 & 6.46 & 6.29 & 2.62 \\\hline
Average Fitness & 1.58E+03 & 2.31E+04 & 6.96E+03 & 2.82E+03 & 2.48E+04 & 4.18E+04 & 3.50E+03 \\\hline
\end{tabular}
    }
\end{table}




\subsection{In-depth Analysis}
\subsubsection{Transfer Evaluation Across Tasks and Dimensions}
\begin{table}[htbp]
\centering
\caption{
Per-task comparison between single-task and multi-task models using sMAE and the transfer contribution rate (TCR). Positive values in the “TCR (\%)” column indicate positive transfer.}
      \resizebox{0.7\linewidth}{!}{
\begin{tabular}{|p{0.3\textwidth}|p{0.1\textwidth}|p{0.1\textwidth}|p{0.1\textwidth}|p{0.15\textwidth}|c|c|}
\hline
Task & sMAE (Single) & sMAE (Multi) & TCR (\%) & Conclusion & $R^{2}$ (Single) & $R^{2}$ (Multi)\\
\hline
Attractive\_Sector & 0.193 & 0.098 & +49.3\% & Positive & -0.659 & 0.414\\\hline
Bent\_Cigar & 0.034 & 0.026 & +23.0\% & Positive & -0.102 & 0.244\\\hline
Buche\_Rastrigin & 0.165 & 0.107 & +35.5\% & Positive & -0.622 & 0.275\\\hline
Composite\_Grie\_Rosen & 0.103 & 0.069 & +33.0\% & Positive & -0.346 & -0.330\\\hline
Different\_Powers & 0.090 & 0.059 & +34.8\% & Positive & -0.202 & 0.276\\\hline
Discus & 0.093 & 0.088 & +5.1\% & Positive & -0.233 & -0.077\\\hline
Ellipsoidal & 0.095 & 0.066 & +30.6\% & Positive & -0.423 & 0.189\\\hline
Ellipsoidal\_high\_cond & 0.105 & 0.066 & +37.6\% & Positive & -0.358 & 0.454\\\hline
Gallagher\_101Peaks & 0.094 & 0.072 & +23.3\% & Positive & -0.388 & 0.073\\\hline
Gallagher\_21Peaks & 0.072 & 0.057 & +20.4\% & Positive & -0.371 & -0.030\\\hline
Katsuura & 0.220 & 0.169 & +23.3\% & Positive & -0.991 & -0.290\\\hline
Linear\_Slope & 0.186 & 0.045 & +75.6\% & Positive & -0.911 & 0.878\\\hline
Lunacek\_bi\_Rastrigin & 0.150 & 0.090 & +40.4\% & Positive & -0.271 & 0.497\\\hline
Rastrigin & 0.163 & 0.078 & +51.9\% & Positive & -0.665 & 0.456\\\hline
Rastrigin\_F15 & 0.090 & 0.053 & +41.4\% & Positive & -0.421 & 0.420\\\hline
Rosenbrock\_original & 0.164 & 0.072 & +56.0\% & Positive & -0.347 & 0.735\\\hline
Rosenbrock\_rotated & 0.123 & 0.085 & +31.2\% & Positive & -1.408 & -0.349\\\hline
Schaffers & 0.073 & 0.033 & +54.5\% & Positive & -0.378 & 0.631\\\hline
Schaffers\_high\_cond & 0.037 & 0.024 & +35.3\% & Positive & -1.133 & 0.220\\\hline
Schwefel & 0.178 & 0.050 & +71.7\% & Positive & -0.395 & 0.886\\\hline
Sharp\_Ridge & 0.132 & 0.058 & +55.7\% & Positive & -0.173 & 0.751\\\hline
Sphere & 0.159 & 0.093 & +41.5\% & Positive & -0.328 & -0.328\\\hline
Step\_Ellipsoidal & 0.206 & 0.105 & +49.2\% & Positive & -1.039 & 0.406\\\hline
Weierstrass & 0.173 & 0.147 & +14.7\% & Positive & -0.650 & -0.250\\
\hline
{Macro avg} & 0.129 & 0.075 & +39.0\% & PTR=100.0\%, NTR=0.0\% & -0.534 & 0.256\\
\hline
\end{tabular}
}

\label{tab:tcr-per-task}
\end{table}

\begin{table}[htbp]
\centering
\caption{
Dimension-wise macro sMAE and TCR comparing single-task and multi-task models. PTR/NTR denote the proportion of tasks per dimension showing positive/negative transfer.}

  \resizebox{0.7\linewidth}{!}{
\begin{tabular}{|p{0.15\textwidth}|p{0.15\textwidth}|p{0.15\textwidth}|p{0.15\textwidth}|p{0.15\textwidth}|p{0.15\textwidth}|p{0.15\textwidth}|}
\hline
Dimension & sMAE (Single) & sMAE (Multi) & TCR (\%) & Conclusion & $R^{2}$ (Single) & $R^{2}$ (Multi)\\
\hline
5  & 3.447 & 0.134 & +96.1\% & Positive & -28.052 & 0.778\\\hline
10 & 0.129 & 0.079 & +39.2\% & Positive & -3.989 & 0.157 \\\hline
15 & 0.136 & 0.068 & +50.2\% & Positive & -0.826 & 0.428\\\hline
20 & 0.129 & 0.075 & +41.6\% & Positive & -0.847& 0.256\\
\hline
{Macro avg} & 0.960 & 0.089 & +56.8\% & PTR=100.0\%, NTR=0.0\% & -33.714 & 1.619\\
\hline

\end{tabular}}

\label{tab:tcr-by-dim}
\end{table}
To probe whether the above gains arise from effective cross-task sharing, we assess transfer at both the task and dimension levels. We use the per-task transfer contribution rate (TCR) to compare the multi-task meta-surrogate against its single-task counterparts:
\begin{equation}
\mathrm{TCR}(t)\;=\;\frac{{Err}_{\text{single}}(t)-{Err}_{\text{multi}}(t)}{{Err}_{\text{single}}(t)},
\qquad \mathrm{TCR}(t)>0\ \Rightarrow\ \text{positive transfer}.
\end{equation}
In addition to reporting \(\mathrm{TCR}(k)\) per task, we summarize the prevalence of transfer effects via the fractions of tasks with positive and negative transfer:
\begin{subequations}\label{eq:ptr-ntr}
\begin{align}
\mathrm{PTR} &= \frac{\bigl|\{\,t:\mathrm{TCR}(t)>0\,\}\bigr|}{NT},\\
\mathrm{NTR} &= 1-\mathrm{PTR}.
\end{align}
\end{subequations}
where \(NT\) is the number of tasks. We further aggregate by dimensionality by macro-averaging \(Err\) over tasks sharing the same dimension and computing \(\mathrm{TCR}\) analogously.

On the 24 BBOB objectives, the multi-task meta-surrogate consistently improves upon single-task surrogates. Table~\ref{tab:tcr-per-task} shows $\mathrm{PTR}=1.00$ and $\mathrm{NTR}=0.00$, mean/median $\mathrm{TCR}=0.389/0.365$, and a macro sMAE decrease from $0.129$ to $0.075$, i.e., a $41.6\%$ relative reduction. The distribution of TCR spans $[+0.051,\,+0.756]$, indicating that improvements are not confined to a few favorable cases but persist across ill-conditioning, separability, and multimodality. $R^{2}$ aligns with sMAE: the multi-task model typically raises $R^{2}$ and, in multiple cases, flips negative values to positive (e.g., Linear\_Slope $-0.911\!\rightarrow\!0.878$, Rosenbrock\_original $-0.347\!\rightarrow\!0.735$, Schwefel $-0.395\!\rightarrow\!0.886$). A few functions (e.g., Composite\_\allowbreak Grie\_\allowbreak Rosen, Sphere) retain negative $R^{2}$ despite lower sMAE, which suggests residual model under-specification on highly oscillatory or plateau-prone regions; nonetheless, the error-normalized TCR remains positive, indicating that magnitude-sensitive mispredictions have been substantially reduced.

To examine whether improvements arise uniformly across dimensionalities rather than from a particular task subset, we aggregate at the dimension level by macro-averaging sMAE over tasks sharing the same dimension and then compute $\mathrm{TCR}_{\text{dim}}$ analogously. Table~\ref{tab:tcr-by-dim} reports uniformly positive transfer: $\mathrm{TCR}_{\text{dim}}=0.961$ (5D), $0.392$ (10D), $0.502$ (15D), and $0.416$ (20D), each with $\mathrm{PTR}=1.00$. The pronounced $5$D improvement reflects that conditioning on metadata (function identity and dimension) enables the shared encoder to amortize representation learning when per-task samples are modest; as dimensionality increases, gains remain substantial ($\approx39\%\text{--}50\%$ macro sMAE reduction), consistent with the hypothesis that the scientific-notation tokenization and priority-aware loss curtail magnitude errors and mitigate gradient interference during joint training. Together, the task-wise and dimension-wise analyses indicate that the proposed conditioning mechanism realizes broad, class-agnostic knowledge sharing while suppressing negative transfer under range-normalized error, with improved calibration evidenced by $R^{2}$.

From a modeling perspective, these outcomes are consistent with two effects. First, shared token embeddings align numerical patterns across tasks at the character level, creating transferable local features even when decision-space semantics differ; this explains uniform gains on structurally related families (e.g., Rastrigin variants). Second, emphasizing sign and exponent tokens in the loss concentrates learning capacity on magnitude-defining symbols, reducing catastrophic errors that dominate sMAE after normalization and directly improving $R^{2}$. 

\subsection{Ablation Study}
\subsubsection{Prompt-Design Ablation}

\begin{figure}[htbp]
  \centering
    \begin{promptmini}[colframe=green!60!black, colback=green!5]{prompt\_box1: Large}
    \label{prompt_Large}
    \scriptsize
    You are an expert many-task surrogate model based on a large language model (LLM).  
    Your task is to predict the fitness value \texttt{y} of an individual \texttt{x} for a given black box function, described by metadata \texttt{m}.
    
    Each sample consists of:
    \begin{itemize}
      \item \texttt{m}: textual metadata describing the target function and its dimensionality.
      \item \texttt{x}: a sequence of float numbers representing an individual solution, tokenized
            in a digit wise manner using a base 10 scientific encoding.
      \item \texttt{y}: the scalar fitness value, also encoded in base 10 scientific notation.
    \end{itemize}
    
    \textbf{Data:}\\
    \texttt{m: The ID of this black-box function is <function\_id>, the human-assigned name is <function\_name>, key feature 1 is <key feature 1> | key feature 2 is <key feature 2> | ... | the dimensionality of decision variables is dim=<D>.}\\
    \texttt{x: [\{x\_1\}, \{x\_2\}, \dots, \{x\_d\}] $\rightarrow$ encoded as: [\textlangle 10\^{}exp\textrangle\; d\_1 d\_2 \dots d\_k]}\\
    \texttt{y:} target fitness value (to be predicted)
    
    \textbf{Examples:}\\
    \texttt{m: The ID of this black-box function is 2, the human-assigned name is Sphere, key feature 1 is "", the dimensionality of decision variables is dimension=4.}\\
    \texttt{x: [-2.065349139,\,-2.570456278,\,3.38108745,\,-3.38108745]}\\
    \texttt{$\rightarrow$ Encoded as:}\\
    \texttt{[}\\
    \texttt{\hspace*{1em}- \textlangle 10\^{}0\textrangle\; 2 0 6 5 3 4 9 1 3 9,}\\
    \texttt{\hspace*{1em}- \textlangle 10\^{}0\textrangle\; 2 5 7 0 4 5 6 2 7 8,}\\
    \texttt{\hspace*{1em}+\, \textlangle 10\^{}0\textrangle\; 3 3 8 1 0 8 7 4 5 0,}\\
    \texttt{\hspace*{1em}+\, \textlangle 10\^{}0\textrangle\; 3 3 8 1 0 8 7 4 5 0}\\
    \texttt{]}\\
    \texttt{y: [+\,\textlangle 10\^{}3\textrangle1 7 4 0 0 5 0 8 4 3]}
    
    Given \texttt{m} and \texttt{x} (encoded), predict \texttt{y} (fitness value) as accurately as possible.
    
    Now, predict the fitness for the following sample:\\
    \textbf{Data}\\
    \texttt{m: "\{m\_input\}"}\\
    \texttt{x: \{x\_encoded\}}\\
    \texttt{y:}
    \end{promptmini}
  \caption{Prompt\_box1 (Large). Comprehensive instruction prompt containing full metadata, scientific-notation encoding details, and a worked example used for fine-tuning the MetaSurrogate model.}
  \label{fig:prompt_large}
\end{figure}

\begin{figure}[htbp]
  \centering
    \begin{promptmini}[colframe=green!60!black, colback=green!5]{prompt\_box2: Middle}
    \label{prompt_Middle}
    \scriptsize
    You are a many-task surrogate model. Given (1) m: metadata describing the function and its dimension.(2) pop: a vector of floats (solution), encoded in base-10 scientific tokens. y: the fitness value. Predict y given m and pop; 
    \\
    \texttt{m: The ID of this black-box function is "\{function\_id\}", the human-assigned name is "\{function\_name\}, key feature 1 is "BBOB" | key feature 2 is "instance=0"| ... | the dimensionality of decision variables is dim="\{D\}".}
    \end{promptmini}
  \caption{Prompt\_box2 (Middle). Medium-length prompt that preserves key metadata fields while omitting verbose numeric-encoding examples to balance brevity and informativeness.}
  \label{fig:prompt_middle}
\end{figure}

\begin{figure}[htbp]
      \centering
    \begin{promptmini}[colframe=green!60!black, colback=green!5]{prompt\_box3: Small}
    \label{prompt_Small}
    \scriptsize
    You are a many-task surrogate model, predict fitness given m and pop;
    \\
    \texttt{m: function name is "\{function\_name\}, function ID is "\{function\_id\}", key feature 1 is "BBOB" | key feature 2 is "instance=0"|...| the dimensionality is dim="\{D\}".}
    \end{promptmini}
  \caption{Prompt\_box3 (Small). Concise prompt template that distills the task description to its essentials.}
  \label{fig:prompt_small}
\end{figure}

\begin{figure}[htbp]
  \centering
    \begin{promptmini}[colframe=green!60!black, colback=green!5]{prompt\_box4: base}
    \label{prompt_base}
    \scriptsize
    \texttt{"BBOB, instance="0", "\{function\_name\}", dimension="\{D\}".}
    \end{promptmini}
  \caption{Prompt\_box4 (Base). Minimal baseline prompt containing only the BBOB identifier, instance index, function name, and dimensionality.}
  \label{fig:prompt_base}
\end{figure}
\begin{table}[htbp]
\caption{The performance of MetaSurrogate on the 20-dimensional BBOB test dataset under different metadata.}
\centering
\label{tab:promptdim20}
\resizebox{0.7\linewidth}{!}{
      \begin{tabular}{|c|c|c|c|c|||c|c|c|c|}
\hline
 & \multicolumn{4}{|c|||}{sMAE} & \multicolumn{4}{c|}{R$^2$} 
 \\ \hline
 & Large & Middle & Small & Base & Large & Middle & Small & Base \\
\hline
Attractive\_Sector      &  0.1175 &           0.0993 &  \textbf{0.0619} &           0.0757 &           0.2615 &           0.4449 &   \textbf{0.6948} &            0.6125 \\
\hline
Bent\_Cigar             &  0.0359 &           0.0406 &  \textbf{0.0222} &           0.0398 &          -0.0551 &           0.1184 &   \textbf{0.1636} &            0.1213 \\
\hline
Buche\_Rastrigin        &  0.1280 &           0.0678 &  \textbf{0.0422} &           0.0435 &          -0.0363 &           0.2376 &   \textbf{0.7201} &            0.6779 \\
\hline
Composite\_Grie\_rosen  &  0.0736 &           0.0919 &  \textbf{0.0588} &           0.0782 &           0.2341 &           0.3149 &            0.4546 &   \textbf{0.4959} \\
\hline
Different\_Powers       &  0.0873 &           0.0827 &  \textbf{0.0319} &           0.0821 &          -0.4069 &           0.2126 &   \textbf{0.3349} &            0.2411 \\
\hline
Discus                  &  0.0905 &           0.0941 &  \textbf{0.0770} &           0.0821 &          -0.0147 &           0.1006 &   \textbf{0.6172} &            0.3500 \\
\hline
Ellipsoidal             &  0.0976 &           0.0717 &  \textbf{0.0691} &           0.1144 &          -0.5318 &           0.3180 &   \textbf{0.3479} &           -0.1004 \\
\hline
Ellipsoidal\_high\_cond &  0.1384 &           0.0990 &  \textbf{0.0689} &           0.1029 &          -0.3229 &           0.4004 &   \textbf{0.4962} &            0.3283 \\
\hline
Gallagher\_101Peaks     &  0.0924 &           0.0691 &  \textbf{0.0679} &           0.0861 &          -0.3035 &          -0.0350 &  \textbf{-0.0062} &           -0.2998 \\
\hline
Gallagher\_21Peaks      &  0.1653 &           0.0613 &  \textbf{0.0527} &           0.0652 &        -482.9473 &          -0.0294 &   \textbf{0.0487} &           -0.1638 \\
\hline
Katsuura                &  0.1582 &           0.1798 &  \textbf{0.1526} &           0.1587 &          -0.1962 &          -0.2654 &           -0.0781 &  \textbf{-0.0708} \\
\hline
Linear\_Slope           &  0.0786 &  \textbf{0.0454} &           0.0454 &           0.0544 &           0.6779 &  \textbf{0.8729} &           -0.7389 &            0.8271 \\
\hline
Lunacek\_bi\_Rastrigin  &  0.1015 &           0.0769 &           0.0837 &  \textbf{0.0711} &           0.4935 &  \textbf{0.6635} &            0.5671 &            0.5900 \\
\hline
Rastrigin               &  0.0637 &  \textbf{0.0495} &           0.0520 &           0.0670 &           0.1556 &  \textbf{0.4089} &            0.3508 &            0.2062 \\
\hline
Rastrigin\_F15          &  0.0589 &           0.0481 &           0.0435 &  \textbf{0.0418} &           0.1121 &           0.3837 &            0.4415 &   \textbf{0.5153} \\
\hline
Rosenbrock\_original    &  0.0926 &           0.0792 &  \textbf{0.0425} &           0.0668 &           0.6412 &           0.7232 &   \textbf{0.9146} &            0.7935 \\
\hline
Rosenbrock\_rotated     &  0.0845 &           0.0875 &  \textbf{0.0546} &           0.0829 &  \textbf{0.3670} &           0.3370 &            0.3136 &            0.2996 \\
\hline
Schaffers               &  0.0726 &           0.0451 &  \textbf{0.0371} &           0.0533 &           0.0441 &  \textbf{0.6892} &            0.5266 &            0.4973 \\
\hline
Schaffers\_high\_cond   &  0.1050 &           0.0175 &           0.0518 &  \textbf{0.0161} &          -0.4239 &           0.1292 &   \textbf{0.3046} &            0.1561 \\
\hline
Schwefel                &  0.0676 &           0.0464 &  \textbf{0.0442} &           0.0523 &           0.8050 &  \textbf{0.8983} &            0.8887 &            0.8730 \\
\hline
Sharp\_Ridge            &  0.0852 &           0.0753 &  \textbf{0.0602} &           0.0611 &           0.5058 &           0.6956 &            0.7467 &   \textbf{0.7552} \\
\hline
Sphere                  &  0.0971 &           0.0741 &           0.0576 &  \textbf{0.0426} &           0.5020 &          -5.3437 &            0.8309 &   \textbf{0.8969} \\
\hline
Step\_Ellipsoidal       &  0.1076 &  \textbf{0.0761} &           0.0767 &           0.0926 &          -0.0006 &           0.5633 &   \textbf{0.6211} &            0.5885 \\
\hline
Weierstrass             &  0.1515 &           0.1391 &           0.1457 &  \textbf{0.1293} &          -0.3100 &          -0.1424 &  \textbf{-0.0993} &           -0.1847 \\
\hline
\end{tabular}
}
\end{table}

\begin{table}[t]
\caption{
Metadata importance via LOO, permutation, and Grad\(\times\)Input (GI) attribution. Deltas are relative to full-metadata inputs (higher = more important).}
\label{tab:summary}
\centering
\resizebox{0.7\linewidth}{!}{
\begin{tabular}{|l|r|r|r|r|r|}
\hline
Field & LOO $\Delta$sMAE & LOO $\Delta(1 - R^2)$ & Perm $\Delta$sMAE & Perm $\Delta(1 - R^2)$ & GI Attr.\ Share \\
\hline
dataset  & 0.020 & 0.070 & 0.014 & 0.050 & 0.201 \\\hline
function\_id     & 0.010 & 0.040 & 0.007 & 0.028 & 0.108 \\\hline
instance & 0.025 & 0.090 & 0.016 & 0.060 & 0.246 \\\hline
function\_name     & 0.040 & 0.130 & 0.025 & 0.085 & 0.417 \\\hline
dim      & 0.004 & 0.020 & 0.002 & 0.012 & 0.028 \\\hline
\end{tabular}}
\end{table}

\begin{figure}[htbp]
  \centering
  \adjustbox{width=0.45\linewidth, trim=0 {.33\height} 0 0, clip}{
    \includegraphics{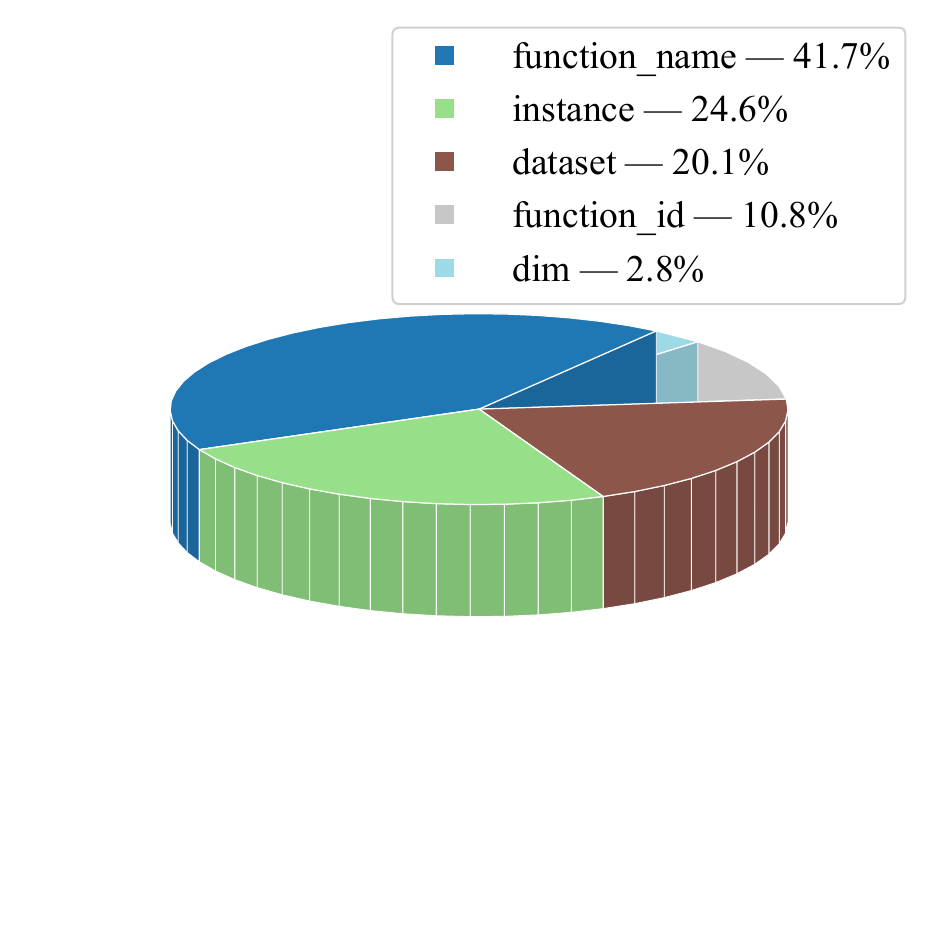}
  }
  \caption{
  GI-based attribution shares over metadata fields. Shares: function\_name $\approx 0.417$, instance $\approx 0.246$, dataset $\approx 0.201$, function\_id $\approx 0.108$, dim $\approx 0.028$.}
  \label{fig:pie_meta}
\end{figure}

We perform an ablation study on prompt design, varying verbosity and examining which metadata fields most benefit the conditional predictor \(p_\theta(y\!\mid\!m,x)\). We consider four instruction templates---\emph{Large}, \emph{Middle}, \emph{Small}, and \emph{Base}---whose formats are shown in Figs.~\ref{fig:prompt_large}--\ref{fig:prompt_base}. All training and decoding settings are identical. We report the scale-normalized mean absolute error (sMAE; lower is better) and the coefficient of determination \(R^2\) (higher is better).

On the 20D BBOB benchmark (Table~\ref{tab:promptdim20}), the \emph{Small} template offers the best trade-off: it attains the lowest sMAE on 16/24 functions and lifts the median \(R^2\) from \(0.11\) (\emph{Large}) to \(0.45\). The verbose \emph{Large} prompt underperforms consistently, suggesting that superfluous tokens diffuse the model’s limited attention and weaken the alignment between numerical tokens and task semantics. The ultra-compact \emph{Base} string performs competitively on smooth, separable landscapes (e.g., Sphere, Sharp\_Ridge) but exhibits high variance and occasional negative \(R^2\) on rugged multimodal functions. \emph{Middle} is a strong second, often matching or exceeding the \(R^2\) of \emph{Small} while remaining concise. Overall, concise yet self-contained instructions generalize best.

\textit{(1)Metadata analysis.}
Let \(m\) comprise fields $\{$ $\texttt{dataset}$, \allowbreak $\texttt{function\_id}$,\allowbreak $\texttt{instance}$, \allowbreak $\texttt{function\_name}$,\allowbreak $\texttt{dim}\}$. We quantify field importance using three complementary diagnostics: {(i) Leave-One-Out (LOO).} We mask one field at train using length-preserving placeholders to minimally disturb token positions; we then re-evaluate sMAE and \(R^2\). We report \(\Delta\mathrm{sMAE} = \mathrm{sMAE}_{\text{masked}} - \mathrm{sMAE}_{\text{full}}\) and \(\Delta(1-R^2) = (1-R^2)_{\text{masked}} - (1-R^2)_{\text{full}}\). Larger positive deltas indicate a more influential field. {(ii) Within-field permutation.} We shuffle test-set values within one field, preserving its marginal distribution but breaking alignment with labels, and compute the same deltas as above. {(iii) Encoder-side Grad\(\times\)Input (GI) attribution.} We obtain token-level GI scores on encoder inputs, map tokens to character offsets of the metadata string, and aggregate per-field importance. The per-example contributions sum to one over metadata tokens.

Table~\ref{tab:summary} summarizes LOO/permutation deltas together with GI shares; Fig.~\ref{fig:pie_meta} visualizes the attribution distribution. Consistent trends emerge across all three diagnostics: \texttt{function\_name} is the dominant driver (GI share \(\approx 0.417\)), followed by \texttt{instance} (\(\approx 0.246\)), then \texttt{dataset} (e.g., BBOB, CEC2010; \(\approx 0.201\)) and \texttt{function\_id} (\(\approx 0.108\)), with \texttt{dim} contributing the least (\(\approx 0.028\)). LOO deltas typically exceed permutation deltas, as hard removal is more destructive than breaking alignment while preserving marginals. Importantly, all five fields exhibit strictly positive LOO and permutation deltas, i.e., ablating any single component degrades both sMAE and \(R^2\); the set $\{$$\texttt{dataset}$,\allowbreak $\texttt{function\_id}$,\allowbreak $\texttt{instance}$,\allowbreak $\texttt{function\_name}$,\allowbreak $\texttt{dim}\}$ therefore constitutes a minimal, non–redundant basis for reliable conditioning.

These patterns indicate complementary—and non–substitutable—roles for the five fields. \texttt{function\_name} resolves objective semantics across suites and cannot be replaced by coarser identifiers; \texttt{instance} captures within–family landscape shifts (e.g., random shifts/rotations), whose removal yields systematic miscalibration. \texttt{dataset} (e.g., BBOB, CEC2010) supplies domain–level priors that reduce cross–suite aliasing and improve transfer; while partially overlapping with names/IDs, it is not substitutable and its ablation consistently harms accuracy. \texttt{function\_id} anchors an unambiguous mapping that mitigates synonymy/versioning issues and improves reproducibility under heterogeneous naming. Finally, \texttt{dim}, despite its smaller GI share, remains essential: the model can in principle infer dimensionality by parsing the numeric sequence, which explains the smaller marginal attribution. However, providing \texttt{dim} explicitly calibrates magnitude and complexity priors for the tokenized stream, and more importantly, directs the model to attend to the tail portion of long sequences. This mitigates the tendency of MetaSurrogate to down-weight later tokens under extended input lengths, thereby preserving calibration and generalization. {Taken together, the five fields are jointly necessary: each contributes distinct information and no single field can be dropped without a measurable loss in predictive quality.} Accordingly, we retain the \emph{Small} template because it is concise {and} preserves all five indispensable components.
\begin{table}[htbp]
\centering
\caption{
Length-controlled redundancy (LCR): median performance on 20D BBOB with \emph{Small} as baseline. Numeric-Share is the Grad$\times$Input share on numeric tokens; Tail-Attn-Mass is the attention mass on the last 25\% of tokens; Attn-Entropy is the average encoder attention entropy.
}
\label{tab:length-sweep}
\resizebox{0.65\linewidth}{!}{
\begin{tabular}{|l|c|c|c|c|c|}
\hline
Extra tokens & sMAE $\downarrow$ & $R^2$ $\uparrow$ & Numeric-Share $\uparrow$ & Tail-Attn-Mass $\uparrow$ & Attn-Entropy $\downarrow$\\
\hline
   +0 & 0.112 & 0.86 & 0.56 & 2.05 &2.15 \\\hline
 +64 & 0.142 & 0.74 &  0.49 & 0.18 & 2.18\\\hline
+128 & 0.129 & 0.75 & 0.42 & 0.27 & 2.27\\\hline
+192 & 0.134 & 0.73 & 0.36 & 0.35 & 2.33\\\hline
+256 & 0.160 & 0.65 & 0.29 & 0.24 & 2.44\\\hline
\end{tabular}}
\end{table}

\begin{table}[htbp]
\centering
\caption{
Field replication (FR): duplicate specific fields in the \emph{Small} template without changing values. Metrics are median over 20D BBOB.}
\label{tab:replication}
\resizebox{0.5\linewidth}{!}{
\begin{tabular}{|l|c|c|c|}
\hline
Variation & sMAE $\downarrow$ & $R^2$ $\uparrow$ & Numeric-Share $\uparrow$ \\\hline
    Baseline (Small) & 0.112 & 0.86 & 0.56 \\\hline
+ \texttt{function\_name} $\times 2$ & 0.17 & 0.83 & 0.09 \\\hline
+ \texttt{function\_name} $\times 3$ & 0.14 & 0.85 & 0.07 \\\hline
+ \texttt{dataset} $\times 3$ & 0.12 & 0.39 & 0.03 \\\hline
+ \texttt{instance} $\times 3$ & 0.13 & 0.36 & 0.11 \\\hline
+ \texttt{function\_id} $\times 3$ & 0.12 & 0.44 & 0.10 \\\hline
+ \texttt{dim} $\times 3$ & 0.17 & 0.42 & 0.10 \\\hline
All fields $\times 2$ & 0.19 & 0.33 & 0.06 \\\hline

\end{tabular}}
\end{table}

\textit{(2)Redundancy analysis.}
We examine why training-time redundancy reduces generalization using two controlled manipulations applied during training and evaluated under the same test protocol. The first is length-controlled redundancy (LCR): neutral, tokenization-stable fragments are appended after the \emph{Small} template, and four models are trained with +64, +128, +192, and +256 extra tokens. The second is field replication (FR): single or all metadata fields are duplicated without changing their values.

Under LCR (Table~\ref{tab:length-sweep}), \(R^2\) decreases from 0.86 to 0.65 as input length grows, while sMAE rises from 0.112 to 0.160. Three mechanistic indicators explain this trend. (1) \emph{Numeric-Share} (GI share on numeric tokens) drops from 0.56 to 0.29, suggesting that gradient budget shifts from structured numeric tokens (sign, exponent, leading mantissa) to redundant fragments, weakening magnitude calibration and digit-level alignment. (2) \emph{Tail-Attn-Mass} (average attention on the last 25\% of positions) declines sharply relative to baseline, indicating less favorable allocation in long contexts. (3) \emph{Attn-Entropy} increases (2.15$\to$2.44), consistent with attention diffusion.

FR (Table~\ref{tab:replication}) exhibits the same mechanism more directly. Replicating fields that strongly identify tasks induces shortcut learning: capacity is over-allocated to duplicated labels, at the expense of modeling input–output numeric structure. Doubling all fields yields \(R^2=0.33\) and \(\text{Numeric-Share}=0.06\); among single-field replications, \texttt{dataset} and \texttt{instance} are most detrimental. Relative decreases of 80\%–95\% in Numeric-Share under FR are common. In summary, training-time redundancy degrades generalization through (i) attention diffusion, (ii) competition for tail positions in long sequences, and (iii) absorption of representational/gradient budget by irrelevant fragments or duplicated labels. We therefore recommend the most concise prompt that still retains the necessary metadata fields.


\subsubsection{Effect of the Priority-Aware Weighted Cross-Entropy (PWCE)}
\label{subsec:ablation_pwce}
\begin{table}[htbp]
\caption{sMAE and coefficient of determination ($R^{2}$) of MetaSurrogate \textit{with} and \textit{without} the priority-aware weighted cross-entropy (PWCE) on 24 twenty-dimensional BBOB benchmark functions}
\centering
\label{tab:pwce_compare}
\resizebox{0.7\linewidth}{!}{
      \begin{tabular}{|l|p{0.2\textwidth}|p{0.2\textwidth}|p{0.2\textwidth}|p{0.2\textwidth}|}
\hline
 & \multicolumn{2}{|c|}{sMAE} & \multicolumn{2}{c|}{R$^2$} \\
 \hline
 & Meta-Surrogate & Mata-Surrogate-w/o-PWCE & MetaSurrogate & Mata-Surrogate-w/o-PWCE \\
\hline
Attractive\_Sector      &  \textbf{0.0619} &           0.0961 &   \textbf{0.6948} &            0.3196 \\
\hline
Bent\_Cigar             &  \textbf{0.0222} &           0.0414 &   \textbf{0.1636} &           -3.8633 \\
\hline
Buche\_Rastrigin        &  \textbf{0.0422} &           0.0551 &   \textbf{0.7201} &            0.5392 \\
\hline
Composite\_Grie\_rosen  &  \textbf{0.0588} &           0.0636 &   \textbf{0.4545} &            0.4215 \\
\hline
Different\_Powers       &  \textbf{0.0319} &           0.0365 &   \textbf{0.3349} &            0.1310 \\
\hline
Discus                  &  \textbf{0.0769} &           0.1552 &   \textbf{0.6172} &           -0.3885 \\
\hline
Ellipsoidal             &  \textbf{0.0691} &           0.0907 &   \textbf{0.3479} &           -0.0218 \\
\hline
Ellipsoidal\_high\_cond &  \textbf{0.0689} &           0.0772 &   \textbf{0.4962} &            0.3541 \\
\hline
Gallagher\_101Peaks     &  \textbf{0.0679} &           0.0748 &  \textbf{-0.0062} &           -0.1492 \\
\hline
Gallagher\_21Peaks      &  \textbf{0.0526} &           0.0541 &   \textbf{0.0487} &           -0.0058 \\
\hline
Katsuura                &  \textbf{0.1526} &           0.1575 &  \textbf{-0.0781} &           -0.1856 \\
\hline
Linear\_Slope           &  \textbf{0.0455} &           0.0913 &           -0.7389 &   \textbf{0.4914} \\
\hline
Lunacek\_bi\_Rastrigin  &  \textbf{0.0837} &           0.0887 &   \textbf{0.5671} &            0.4717 \\
\hline
Rastrigin               &  \textbf{0.0520} &           0.0652 &   \textbf{0.3508} &            0.0782 \\
\hline
Rastrigin\_F15          &  \textbf{0.0435} &           0.0560 &   \textbf{0.4415} &            0.2948 \\
\hline
Rosenbrock\_original    &  \textbf{0.0425} &           0.1027 &   \textbf{0.9147} &            0.4663 \\
\hline
Rosenbrock\_rotated     &  \textbf{0.0546} &           0.0616 &   \textbf{0.3136} &            0.2395 \\
\hline
Schaffers               &  \textbf{0.0371} &           0.0562 &   \textbf{0.5266} &           -0.2065 \\
\hline
Schaffers\_high\_cond   &  \textbf{0.0518} &           0.0617 &   \textbf{0.3046} &            0.1719 \\
\hline
Schwefel                &  \textbf{0.0442} &           0.0660 &   \textbf{0.8887} &            0.7408 \\
\hline
Sharp\_Ridge            &  \textbf{0.0602} &           0.1007 &   \textbf{0.7467} &            0.3112 \\
\hline
Sphere                  &  \textbf{0.0576} &           0.0985 &   \textbf{0.8309} &            0.4194 \\
\hline
Step\_Ellipsoidal       &  \textbf{0.0767} &           0.1224 &   \textbf{0.6211} &            0.1154 \\
\hline
Weierstrass             &           0.1457 &  \textbf{0.1452} &           -0.0993 &  \textbf{-0.0824} \\
\hline
\end{tabular}
}
\end{table}
\autoref{tab:pwce_compare} contrasts MetaSurrogate with its vanilla fine-tuning counterpart in which PWCE is replaced by the standard token-wise cross-entropy (\emph{MetaSurrogate-w/o-PWCE}). PWCE delivers a consistent gain on 23 out of 24 BBOB functions in terms of sMAE and on {22 out of 24} functions in terms of $R^{2}$. Averaged over all tasks, the sMAE drops from $0.0841$ to $0.0625$—--a 24.8\% relative reduction—--while the mean $R^{2}$ {increases} from $0.028$ to $0.394$, representing an order-of-magnitude improvement. 

The only clear outlier is \textit{Linear\_Slope}, whose $R^{2}$ deteriorates despite a substantial sMAE reduction. This deterministic, piecewise-linear landscape is dominated by the sign and exponent rather than the mantissa; PWCE’s linear decay of token importance (cf.~Section,\ref{subsec:finetune}) over-penalises early mantissa errors, leading to occasional misalignment around the slope discontinuity. Similar, but less pronounced, effects are observed on \textit{Weierstrass}. Nonetheless, the aggregate statistics confirm that prioritising structural tokens (sign, exponent, leading digit) is overwhelmingly beneficial across heterogeneous landscapes.

In summary, PWCE realises the intuition that magnitude-defining tokens are more critical than trailing digits, and it does so with negligible computational overhead. The ablation therefore validates PWCE as a sensible default loss for numerical-token generation in LLM-based surrogates.

\subsection{Parameter Sensitivity Analysis}
\subsubsection{Sensitivity to the PWCE Weight–Decay Parameter}
We examine the PWCE weight–decay parameter by sweeping $2\alpha\in$\{10,15,20,25,30\} while keeping all other settings fixed. The structural metrics are evaluated token–wise: Sign Err counts sign mismatches; Exp Err averages the absolute difference between the predicted and true exponents; 1st Mantissa Err averages the absolute difference of the first mantissa digit; Struct Exact Acc is the proportion where all three structural tokens match simultaneously. 

Table~\ref{tab:alpha-sensitivity} summarizes the results. Moderate structural weighting ($2\alpha\in[10,25]$) keeps Sign Err at 0 with small Exp Err and 1st Mantissa Err, indicating stable learning of polarity and order of magnitude. {Overemphasis} at $2\alpha=30$ increases Exp Err and 1st Mantissa Err and lowers Struct Exact Acc, suggesting that too much weight on early tokens hampers learning of fine mantissa digits. $R^2$ is the most discriminative indicator: it peaks at $2\alpha=20$, remains competitive at $2\alpha=25$, but deteriorates at $2\alpha=15$ and $30$. This pattern indicates that structural exactness alone does not guarantee the best numerical fit; some weight must remain for downstream mantissa digits. We therefore adopt $\alpha=10$ throughout the paper.

\begin{table}[htbp]
\centering
\small
\caption{
Sensitivity study for $2\alpha$. Exp Err and 1st Mantissa Err are mean absolute gaps; numerical metrics use the ground–truth $y$.}
\label{tab:alpha-sensitivity}
\resizebox{0.7\linewidth}{!}{
\begin{tabular}{|l|c|c|c|c|c|c|c|}
\hline
$2\alpha$ & $\lceil \alpha \rceil$ & Sign Err (\%) & Exp Err (\%) & 1st Mantissa Err (\%) & Struct Exact Acc (\%) & sMAE & $R^{2}$ \\
\hline
10 & 5  & 0.0 & 0.07 & 0.73 & 69.9 & 0.001 & 0.43 \\\hline
15 & 8  & 0.0 & 0.06 & 0.67 & 71.9 & 0.001 & -0.31 \\\hline
20 & 10 & 0.0 & 0.07 & 0.70 & 70.6 & 0.001 & \textbf{0.45} \\\hline
25 & 13 & 0.0 & 0.07 & 0.72 & 69.5 & 0.001 & 0.28 \\\hline
30 & 15 & 0.0 & 0.12 & 1.02 & 60.7 & 0.001 & 0.10 \\
\hline
\end{tabular}}
\end{table}


\subsubsection{Sensitivity of the Numeric-Precision Parameter $\gamma$}
  \begin{table}[htbp]\centering
\caption{
Sensitivity analysis of the numeric-precision $\gamma$ (D=20): sMAE (left) and $R^2$ (right).}
\label{tab:gamma-combined-single}
      \resizebox{0.7\linewidth}{!}{
\begin{tabular}{|l|c|c|c|c|c|c|c|c|}
\hline
 & \multicolumn{4}{c|}{sMAE (lower is better)} & \multicolumn{4}{c|}{$R^2$ (higher is better)} \\ \hline
Task & $\gamma=8$ & $\gamma=10$ & $\gamma=15$ & $\gamma=20$ & $\gamma=8$ & $\gamma=10$ & $\gamma=15$ & $\gamma=20$ \\ \hline
Attractive\_Sector & 0.235 & 0.236 & \textbf{0.062} & 0.235 & -0.379 & -0.371 & \textbf{0.695} & -0.374 \\ \hline
Bent\_Cigar & 0.214 & 0.206 & \textbf{0.022} & 0.311 & 0.184 & \textbf{0.191} & 0.164 & -0.104 \\ \hline
Buche\_Rastrigin & 0.176 & 0.289 & \textbf{0.042} & 0.822 & \textbf{0.768} & 0.285 & 0.720 & -3.984 \\ \hline
Composite\_Grie\_Rosen & 0.210 & 0.232 & \textbf{0.059} & 0.239 & \textbf{0.466} & 0.182 & 0.455 & 0.324 \\ \hline
Different\_Powers & 0.278 & 0.352 & \textbf{0.032} & 0.288 & -0.326 & -0.855 & \textbf{0.335} & -0.560 \\ \hline
Discus & 0.127 & 0.245 & \textbf{0.077} & 0.300 & 0.607 & -0.614 & \textbf{0.617} & -0.756 \\ \hline
Ellipsoidal & 0.430 & 0.519 & \textbf{0.069} & 0.396 & -1.071 & -1.664 & \textbf{0.348} & -0.719 \\ \hline
Ellipsoidal\_high\_cond & 0.283 & 0.247 & \textbf{0.069} & 0.312 & -0.089 & 0.244 & \textbf{0.496} & 0.023 \\ \hline
Gallagher\_101Peaks & 0.419 & 0.381 & \textbf{0.068} & 0.464 & -0.795 & -0.652 & \textbf{-0.006} & -1.086 \\ \hline
Gallagher\_21Peaks & 0.478 & 0.456 & \textbf{0.053} & 0.342 & -1.062 & -1.048 & \textbf{0.049} & -0.688 \\ \hline
Katsuura & 0.486 & 0.434 & \textbf{0.153} & 0.494 & -1.659 & -1.136 & \textbf{-0.078} & -1.583 \\ \hline
Linear\_Slope & 0.538 & 0.318 & \textbf{0.045} & 0.485 & -3.388 & \textbf{-0.240} & -0.739 & -0.990 \\ \hline
Lunacek\_bi\_Rastrigin & 0.265 & 0.284 & \textbf{0.084} & 0.195 & 0.383 & -0.233 & 0.567 & \textbf{0.621} \\ \hline
Rastrigin & 0.434 & 0.494 & \textbf{0.052} & 0.382 & -0.990 & -1.131 & \textbf{0.351} & -0.491 \\ \hline
Rastrigin\_F15 & 0.282 & 0.265 & \textbf{0.043} & 0.271 & -0.161 & 0.033 & \textbf{0.442} & -0.157 \\ \hline
Rosenbrock\_original & 0.309 & 0.291 & \textbf{0.043} & 0.232 & -0.572 & -0.292 & \textbf{0.915} & -0.306 \\ \hline
Rosenbrock\_rotated & 0.240 & 0.219 & \textbf{0.055} & 0.176 & 0.661 & 0.686 & 0.314 & \textbf{0.761} \\ \hline
Schaffers & 0.306 & 0.511 & \textbf{0.037} & 0.491 & -0.122 & -1.714 & \textbf{0.527} & -1.308 \\ \hline
Schaffers\_high\_cond & 0.255 & 0.289 & \textbf{0.052} & 0.297 & -0.020 & -0.180 & \textbf{0.305} & -0.167 \\ \hline
Schwefel & 0.112 & 0.122 & \textbf{0.044} & 0.280 & 0.777 & 0.721 & \textbf{0.889} & -0.365 \\ \hline
Sharp\_Ridge & 0.307 & 0.293 & \textbf{0.060} & 0.320 & 0.331 & 0.405 & \textbf{0.747} & 0.344 \\ \hline
Sphere & 0.358 & 0.358 & \textbf{0.058} & 0.286 & -0.914 & -0.909 & \textbf{0.831} & -0.592 \\ \hline
Step\_Ellipsoidal & 0.467 & 0.683 & \textbf{0.077} & 0.706 & -1.068 & -3.165 & \textbf{0.621} & -3.800 \\ \hline
Weierstrass & 0.318 & 0.319 & \textbf{0.146} & 0.318 & -0.450 & -0.452 & \textbf{-0.099} & -0.528 \\ \hline
\end{tabular}}
\end{table}
We vary only $\gamma$ and evaluate on the 20D slice covering all 24 BBOB functions. As reported in Table~\ref{tab:gamma-combined-single}, $\gamma=15$ attains the row-wise best sMAE on 24/24 tasks and yields the lowest column-mean sMAE. We observe a mild counter-trend on a few relatively smooth or ridge-like landscapes, where longer sequences bring a small $R^2$ gain (e.g., \textit{Lunacek\_bi\_Rastrigin} and \textit{Rosenbrock\_rotated} peak at $\gamma=20$). By contrast, on rugged multi-modal problems, $\gamma=20$ tends to hurt accuracy, consistent with attention dilution and a tighter input-length budget.

From a numerical standpoint, increasing $\gamma$ reduces quantization error: the worst-case relative rounding bound decays by roughly a factor of $10$ per extra digit. However, when sequence length becomes the dominant constraint (context window and attention spread), the benefits of larger $\gamma$ show diminishing returns. Balancing these effects, we adopt $\gamma=15$.

\subsection{Application Study: Planar Manipulator Control}

We evaluate the practical utility of the meta-surrogate on planar manipulator controls\footnote{The detailed problem definitions and code are available at \url{https://github.com/intLyc/MTO-Platform}.}. 

\begin{figure}[htbp]
    \centering
    \includegraphics[width=0.3\linewidth]{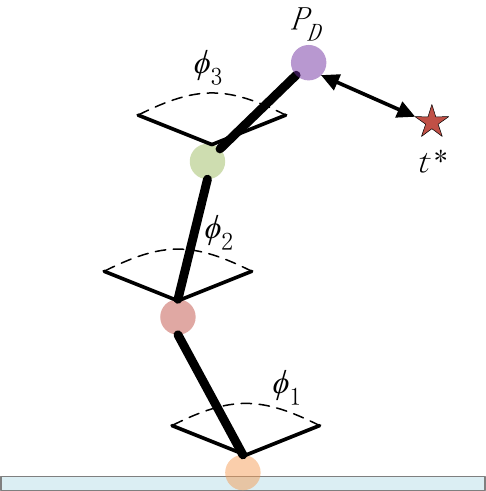}
    \caption{Illustration of a three-joint task in the planar manipulator control problem}
    \label{fig:real_problem}
\end{figure}

\autoref{fig:real_problem} illustrates a task instance from the planar manipulator control problem. In this setting, each task aims to optimize the joint angles \(\phi_{1},\phi_{2},\ldots,\phi_{d}\) so as to minimize the distance between the end-effector tip \(P_{D}\) and a prescribed target point \(t^\star\). The dimensionality \(d\) equals the number of joints (and links). Consequently, each task is a minimization problem whose decision variables are \(\phi_{1},\phi_{2},\ldots,\phi_{d}\) and whose objective is
\begin{equation}
J^{(t)}\!\left(\phi_{1},\phi_{2},\ldots,\phi_{d};\left[L^{t},\phi_{\max }^{t}\right]\right)
=\left\lVert P_{D}-t^\star\right\rVert .
\end{equation}
Here, \(L^{t}\) denotes the total length of all links in task \(t\), and \(\phi_{\max }^{t}\) denotes the sum of the maximum allowable joint angles (i.e., each link has length \(L^{t}/d\) and each joint has a maximum angle \(\phi_{\max }^{t}/d\)). By choosing different values of \(L^{t}\) and \(\phi_{\max }^{t}\), multiple tasks are generated via the centroidal Voronoi tessellation (CVT) method. For all tasks we set \(d=20\); each angle is constrained to \([0,1]\); and the target location is fixed at \(t^\star=[0.5,0.5]\). 
To isolate surrogate effects, we fix the evolutionary protocol across methods (each task has 500 fitness data points) and evaluate $NT\in\{$50, 100, 150, 200$\}$ concurrent tasks. Baselines include an RBFN-based surrogate and a direct-evaluation variant without a surrogate under the same budget.

Because absolute scales of $J^{(t)}$ differ across tasks, we adopt the \emph{mean standardized score} (MSS) as the primary metric. Let $y_i^{\star}$ be the final best (lowest) fitness on task $i$; let $\mu_i$ and $\sigma_i$ be the pooled mean and standard deviation across \emph{all} algorithms and runs on that task. We define
\begin{equation}
\label{eq:mss}
\mathrm{MSS} \;=\; \frac{1}{NT}\sum_{i=1}^{NT} \frac{y_i^{\star}-\mu_i}{\sigma_i},
\end{equation}
so smaller values indicate better performance (more negative means better-than-average relative to the pool). Statistical significance is assessed by the Wilcoxon rank-sum test at $\alpha_{\text{sig}}=0.05$; comparator cells are annotated with “$+/\approx/-$” to indicate significantly better / statistically indistinguishable / significantly worse relative to \textit{MetaSurrogate} within the same optimizer.

\begin{figure}[htbp]
      \centering
    \begin{promptmini}[colframe=green!60!black, colback=green!5]{prompt\_box3: Small}
    \label{prompt_Small}
    \scriptsize
    You are a many-task surrogate model, predict fitness given m and pop;
    \\
    \texttt{m: function name is "task11", function ID is "10", key feature 1 is "Planar\_Kinematic\_Arm\_Control" | key feature 2 is "instance=0"|...| the dimensionality is dim="20".}
    \end{promptmini}
  \caption{
  {
  Prompt template for the planar manipulator control task.}}
  \label{fig:prompt_small_app}
\end{figure}
\autoref{fig:prompt_small_app} presents the metadata template of our Planar Manipulator Control problem. Table~\ref{tab:mss-wide} summarizes results across both backends. Three observations follow. First, \textit{MetaSurrogate} dominates the direct-evaluation baseline, underscoring that in many-task regimes with tight per-task budgets, exploiting cross-task regularities via a high-capacity surrogate can outperform purely local search guided by exact but sparse evaluations. Second, relative to RBFN, \textit{MetaSurrogate} yields equal or better MSS with MaTDE (two significant wins, two ties) and strictly better MSS with BLKT-DE (three wins, one tie), suggesting benefits beyond the local smoothness assumptions favored by radial-basis models. Third, the gains strengthen with the number of tasks—especially under BLKT-DE—consistent with the hypothesis that text-conditioned metadata induces a transferable representation whose utility scales with task diversity. 

\begin{table}[htbp]
\centering
\caption{
Planar manipulator control: MSS over $NT$ tasks (lower is better). “$+/\approx/-$” compares each comparator to \textit{MetaSurrogate} within the same optimizer (Wilcoxon, $\alpha_{\text{sig}}=0.05$).}
\label{tab:mss-wide}
\resizebox{\linewidth}{!}{
\begin{tabular}{|c|c|c|c||c|c|c|}
\hline
Tasks & {MetaSurrogate\_MaTDE} & RBFN\_MaTDE & REAL\_MaTDE & {MetaSurrogate\_BLKT-DE} & RBFN\_BLKT-DE & REAL\_BLKT-DE \\
\hline
50  & \textbf{-0.5101} & -0.3099 ($\approx$) & 0.8200 ($+$) & \textbf{-0.4769} & -0.2271 ($\approx$) & 0.7040 ($+$) \\ \hline
100 & \textbf{-0.5374} & -0.2105 ($\approx$) & 0.7479 ($+$) & \textbf{-0.5941} & -0.1201 ($+$)        & 0.7141 ($+$) \\ \hline
150 & \textbf{-0.5580} & -0.1294 ($+$)        & 0.6873 ($+$) & \textbf{-0.6381} & -0.0253 ($+$)        & 0.6634 ($+$) \\ \hline
200 & \textbf{-0.5558} & -0.1421 ($+$)        & 0.6979 ($+$) & \textbf{-0.6231} & -0.0423 ($+$)        & 0.6654 ($+$) \\
\hline
\multicolumn{1}{|c|}{($+/\approx/-$)} &
\multicolumn{3}{c|}{RBFN\_MaTDE: $2/2/0$ \quad REAL\_MaTDE: $4/0/0$} &
\multicolumn{3}{c|}{RBFN\_BLKT-DE: $3/1/0$ \quad REAL\_BLKT-DE: $4/0/0$} \\
\hline
\end{tabular}}
\end{table}

Taken together, the study indicates that \textit{MetaSurrogate} generalizes beyond synthetic suites and serves as an effective, optimizer-agnostic fitness surrogate under realistic many-task conditions. The observed scaling trend with $NT$ further suggests that richer task diversity strengthens the learned cross-task representation, aligning with the design choice of conditioning on compact textual metadata.
\section{Potential Applications and Deployment Considerations}
\label{sec:applications_potential}

This section outlines when the proposed offline, many-task meta-surrogate is most actionable and how to deploy it responsibly. The approach is most suitable when (i) online evaluations are restricted or infeasible, so optimization must rely on pre-collected data; (ii) multiple related tasks co-exist but differ in input dimensionality and semantics; and (iii) concise textual metadata are available to encode task identity and context. Under these conditions, a single conditional predictor \(p_\theta(y\mid m,x)\) amortizes learning across tasks and provides calibrated, low-latency fitness estimates to guide evolutionary search.

Representative application domains include: (1) manufacturing and process optimization, {in which} ceramic formulations and high-temperature furnace/kiln operations quality or cost metrics depend on complex thermo-chemical processes, re-running trials is expensive or impractical, and historical logs are the primary data source; task identity arises from product families, instances/lots, and operating regimes, and encoding these descriptors as metadata \(m\) (e.g., family, instance, regime, dimension) enables transfer across families while respecting local regimes, providing fast screening for limited confirmatory tests \cite{Ceramic_Formula,fused_magnesium_furnaces,blast_furnace}; (2) design-space exploration under heavy simulators, notably electronic design automation and hardware accelerators, where multi-objective constraints and costly simulators make concurrent online sampling prohibitive, historical design–performance pairs across related IP blocks or technology nodes yield heterogeneous tasks, and with appropriate metadata (block type, node, constraint class, dimension) the meta-surrogate supports rapid candidate triage and budget-aware exploration \cite{Hardware_Accelerators}; (3) planning and public-safety analytics, where trauma system design relies on fixed historical event records as spatio-temporal references that cannot be resampled, tasks decompose by region or policy regime with metadata capturing geography and protocol variants, and the surrogate supplies stable, range-normalized predictions for scenario comparison and sensitivity analyses \cite{trauma_systems}; and (4) robotics and control, where related control tasks share kinematic structure but differ in dimensionality or constraints (e.g., planar manipulators with varying joint counts or limits), {therefore,} conditioning on metadata enables cross-task evaluation proxies without per-task retraining and complements evolutionary controllers in data-limited regimes.

For deployment, a minimal and robust workflow comprises {the following steps}: (1) metadata design—define a compact, machine-readable \(m\) capturing task identity and salient context (family/ID, instance, dimension, regime); (2) offline training—pool \((m,x,y)\) triples from historical tasks and fine-tune \(p_\theta(y\mid m,x)\) once; (3) serving—expose the surrogate as a stateless microservice and cache encoder states for common \(m\); (4) integration—replace a portion of real evaluations with surrogate calls inside the ETO loop and periodically validate with a small, budgeted set of ground-truth checks; (5) uncertainty gating—use token-level entropy or margin-based scores to select individuals for real evaluation when risk is high and to trigger {lightweight} recalibration; and (6) governance—monitor {distributional} shift in \(m\) and \(x\), document the supported task envelope, and disallow extrapolations that violate declared metadata ranges. In terms of scope and limitations, the surrogate inherits the coverage of the training corpus and the fidelity of metadata; it is most reliable within families seen during training and under modest distributional drift. For mission-critical use, conservative fallback policies, uncertainty-aware selection for ground-truthing, and routine recalibration when new regimes emerge are recommended. This paper focuses on single-objective offline many-task settings; extending to constrained and multi-objective cases is feasible but requires problem-specific decoding and calibration.

\section{Conclusion}
\label{sec:conclusion}
This paper {proposes} a novel LLM-based meta-surrogate framework aimed at addressing many-task optimization under data-driven conditions. 
By representing both problem metadata and decision variables {as} a unified token sequence, we {establish} a single model capable of cross-task fitness prediction without the need for separate training on each task. This token-based approach maintains fidelity across varying dimensions and problem complexities, significantly enhancing the flexibility and scalability of SAEAs. In particular, we design a {scientific notation encoding (SNE)} scheme to preserve crucial numerical information and introduced a priority-aware weighted cross-entropy (PWCE) to emphasize key numerical tokens during training. These design choices {contribute} to higher prediction accuracy and robust performance, as shown by our experiments.

We first demonstrated \emph{surrogate feasibility} (\textbf{RQ1}), showing that the meta-surrogate can effectively learn from tokenized numeric and metadata inputs, achieving results on par with or superior to traditional surrogates such as RBFNs and MTGPs. Notably, it requires only a single model to handle many tasks with diverse dimensions---an important advantage in real-world scenarios where problem dimensions or objectives vary significantly. 
We next {evaluate} the \emph{reliability \& calibration} properties of the model (\textbf{RQ2}). Token-level mechanistic analyses and uncertainty–error correlation studies reveal that the decoder attends primarily to sign- and exponent-level information and that mean token entropy correlates strongly with absolute prediction error, confirming that the meta-surrogate supplies trustworthy confidence estimates that can drive active sampling in online settings.
We then examined the \emph{emergent capability} (\textbf{RQ3}), revealing that the meta-surrogate exhibits promising cross-dimensional generalization properties, evidenced by its zero-shot performance on tasks with unseen dimensions. This emergent behavior underscores the potential of LLMs to transfer knowledge across tasks in ways that go beyond conventional neural surrogates. Finally, we validate \emph{optimization guidance performance} (\textbf{RQ4}) by integrating the meta-surrogate into evolutionary many-task algorithms. Comparative results on MCF benchmarks demonstrate marked gains over both original ETO algorithms and those aided by more conventional surrogates.

Although this study has validated the effectiveness and potential of the LLM-based meta-surrogate in many-task optimization, three practical limitations remain: (i) compared with lightweight GP/RBFN surrogates, the LLM microservice exhibits higher latency, memory usage, and energy consumption under concurrent queries; (ii) in compute-constrained DDEA environments, online fine-tuning incurs considerable overhead; and (iii) the current study is limited to single-objective, unconstrained, offline optimization tasks, without extensions to constrained, multiobjective, or multifidelity problems. To address these issues, future work will proceed in three directions: first, {to achieve} low-latency and high-throughput inference deployment through distillation, quantization, and adaptive precision, combined with semi-autoregressive decoding, early exit, cache reuse, and graph compilation optimizations; second, {to design} parameter-efficient fine-tuning mechanisms (e.g., LoRA/Adapters) triggered by uncertainty or distribution-shift detection, together with replay-regularized continual learning and hybrid uncertainty-gated ensembles incorporating lightweight surrogates, to enable stable online adaptation under strict compute and time budgets; and third, {to extend} the meta-surrogate to constrained, multiobjective, and multifidelity optimization via multi-head output structures, calibrated feasibility estimation, and constraint-aware acquisition strategies, thereby enhancing applicability and robustness across a broader spectrum of optimization tasks. 
Despite current limitations in data generalization, training cost, and handling of complex multi-constraint problems, the proposed method demonstrates significant potential to advance the development and deployment of LLM-based surrogate models {across} a wider range of applications and algorithms.


\bibliography{MTO}

\begin{thebibliography}{10}

\bibitem{Expensive_Optimization_overview}
Jian-Yu Li, Zhi-Hui Zhan, and Jun Zhang.
\newblock Evolutionary computation for expensive optimization: A survey.
\newblock {\em Machine Intelligence Research}, 19(1):3--23, feb 2022.

\bibitem{DDEA_overview}
Yaochu Jin, Handing Wang, Tinkle Chugh, Dan Guo, and Kaisa Miettinen.
\newblock Data-driven evolutionary optimization: An overview and case studies.
\newblock {\em IEEE Transactions on Evolutionary Computation}, 23(3):442--458, 2019.

\bibitem{online_ddea}
Xiao Liu, Chunfu Hu, Xiongsong Li, Jian Gao, and Shoudao Huang.
\newblock An online data-driven multi-objective optimization of a permanent magnet linear synchronous motor.
\newblock {\em IEEE Transactions on Magnetics}, 57(7):1--4, 2021.

\bibitem{offline_ddea}
Yue-Jiao Gong, Yuan-Ting Zhong, and Hao-Gan Huang.
\newblock Offline data-driven optimization at scale: A cooperative coevolutionary approach.
\newblock {\em IEEE Transactions on Evolutionary Computation}, 28(6):1809--1823, 2023.

\bibitem{Ceramic_Formula}
Wen-Xiang Song, Wei-Neng Chen, and Ya-Hui Jia.
\newblock An interactive evolutionary algorithm for ceramic formula design.
\newblock In {\em Neural Information Processing: 30th International Conference, ICONIP 2023, Changsha, China, November 20–23, 2023, Proceedings, Part I}, page 381–394. Springer-Verlag, 2023.

\bibitem{Hardware_Accelerators}
Aviral Kumar, Amir Yazdanbakhsh, Milad Hashemi, Kevin Swersky, and Sergey Levine.
\newblock Data-driven offline optimization for architecting hardware accelerators.
\newblock In {\em International Conference on Learning Representations}, 2022.

\bibitem{Zhen2025PDO}
Huixiang Zhen, Bing Xue, Wenyin Gong, Mengjie Zhang, and Ling Wang.
\newblock Offline evolutionary optimization with problem-driven model pool design and weighted model selection indicator.
\newblock {\em Swarm and Evolutionary Computation}, 97:102034, 2025.

\bibitem{zhong2025data}
Yuan-Ting Zhong and Yue-Jiao Gong.
\newblock Data-driven evolutionary computation under continuously streaming environments: A drift-aware approach.
\newblock {\em IEEE Transactions on Evolutionary Computation}, 2025.

\bibitem{ETO}
Kay~Chen Tan, Liang Feng, and Min Jiang.
\newblock Evolutionary transfer optimization - a new frontier in evolutionary computation research.
\newblock {\em IEEE Computational Intelligence Magazine}, 16(1):22--33, 2021.

\bibitem{OTMO}
Sheng-Hao Wu, Zhi-Hui Zhan, Kay~Chen Tan, and Jun Zhang.
\newblock Orthogonal transfer for multitask optimization.
\newblock {\em IEEE Transactions on Evolutionary Computation}, 27(1):185--200, 2023.

\bibitem{thanh2022ensemble}
Binh Huynh~Thi Thanh, Ta~Bao Thang, Nguyen~Hoang Long, et~al.
\newblock Ensemble multifactorial evolution with biased skill-factor inheritance for many-task optimization.
\newblock {\em IEEE Transactions on Evolutionary Computation}, 27(6):1735--1749, 2022.

\bibitem{TRADE}
Sheng-Hao Wu, Zhi-Hui Zhan, Kay~Chen Tan, and Jun Zhang.
\newblock Transferable adaptive differential evolution for many-task optimization.
\newblock {\em IEEE Transactions on Cybernetics}, 53(11):7295--7308, 2023.

\bibitem{MaMPSO}
Xinfang Ji, Yong Zhang, Dunwei Gong, Xiaoyan Sun, and Yinan Guo.
\newblock Multisurrogate-assisted multitasking particle swarm optimization for expensive multimodal problems.
\newblock {\em IEEE Transactions on Cybernetics}, 53(4):2516--2530, 2023.

\bibitem{invTrEMO}
Jiao Liu, Abhishek Gupta, and Yew-Soon Ong.
\newblock Bayesian inverse transfer in evolutionary multiobjective optimization.
\newblock {\em ACM Trans. Evol. Learn. Optim.}, 4(4), November 2024.

\bibitem{HSVLMC_liu}
Haitao Liu, Kai Wu, Yew-Soon Ong, Chao Bian, Xiaomo Jiang, and Xiaofang Wang.
\newblock Learning multitask gaussian process over heterogeneous input domains.
\newblock {\em IEEE Transactions on Systems, Man, and Cybernetics: Systems}, 53(10):6232--6244, 2023.

\bibitem{Kronecker_Gp2}
Jihao~Andreas Lin, Sebastian Ament, Maximilian Balandat, and Eytan Bakshy.
\newblock Scaling gaussian processes for learning curve prediction via latent kronecker structure, 2024.

\bibitem{code_generation}
Yujia Li, David Choi, Junyoung Chung, Nate Kushman, Julian Schrittwieser, Rémi Leblond, Tom Eccles, James Keeling, Felix Gimeno, Agustin~Dal Lago, Thomas Hubert, Peter Choy, Cyprien de~Masson~d’Autume, Igor Babuschkin, Xinyun Chen, Po-Sen Huang, Johannes Welbl, Sven Gowal, Alexey Cherepanov, James Molloy, Daniel~J. Mankowitz, Esme~Sutherland Robson, Pushmeet Kohli, Nando de~Freitas, Koray Kavukcuoglu, and Oriol Vinyals.
\newblock Competition-level code generation with alphacode.
\newblock {\em Science}, 378(6624):1092--1097, 2022.

\bibitem{Symbolic_Mathematics}
Aitor Lewkowycz, Anders Andreassen, David Dohan, Ethan Dyer, Henryk Michalewski, Vinay Ramasesh, Ambrose Slone, Cem Anil, Imanol Schlag, Theo Gutman-Solo, Yuhuai Wu, Behnam Neyshabur, Guy Gur-Ari, and Vedant Misra.
\newblock Solving quantitative reasoning problems with language models.
\newblock In S.~Koyejo, S.~Mohamed, A.~Agarwal, D.~Belgrave, K.~Cho, and A.~Oh, editors, {\em Advances in Neural Information Processing Systems}, volume~35, pages 3843--3857. Curran Associates, Inc., 2022.

\bibitem{science_infer}
Karan Singhal, Shekoofeh Azizi, Tao Tu, S.~Sara Mahdavi, Jason Wei, Hyung~Won Chung, Nathan Scales, Ajay Tanwani, Heather Cole-Lewis, Stephen Pfohl, Perry Payne, Martin Seneviratne, Paul Gamble, Chris Kelly, Abubakr Babiker, Nathanael Schärli, Aakanksha Chowdhery, Philip Mansfield, Dina Demner-Fushman, Blaise Agüera~y Arcas, Dale Webster, Greg~S. Corrado, Yossi Matias, Katherine Chou, Juraj Gottweis, Nenad Tomasev, Yun Liu, Alvin Rajkomar, Joelle Barral, Christopher Semturs, Alan Karthikesalingam, and Vivek Natarajan.
\newblock Large language models encode clinical knowledge.
\newblock {\em Nature}, 620(7972):172--180, August 2023.

\bibitem{xue2021evolutionary}
Xiaoming Xue, Cuie Yang, Yao Hu, Kai Zhang, Yiu-Ming Cheung, Linqi Song, and Kay~Chen Tan.
\newblock Evolutionary sequential transfer optimization for objective-heterogeneous problems.
\newblock {\em IEEE Transactions on Evolutionary Computation}, 26(6):1424--1438, 2021.

\bibitem{high_fidelity}
Hyejin Han, Jounghuem Kwon, Jiyong Lee, Romain Destenay, and Bum-Jae You.
\newblock Real-time optimization for the high-fidelity of human motion imitation.
\newblock In {\em 2014 11th International Conference on Ubiquitous Robots and Ambient Intelligence (URAI)}, pages 692--695. IEEE, 2014.

\bibitem{peole}
Majid Farzaneh and Rahil Mahdian~Toroghi.
\newblock Music generation using an interactive evolutionary algorithm.
\newblock In {\em Pattern Recognition and Artificial Intelligence: Third Mediterranean Conference, MedPRAI 2019, December 22--23, 2019, Proceedings 3}, pages 207--217. Springer, 2020.

\bibitem{trauma_systems}
Handing Wang, Yaochu Jin, and Jan~O Jansen.
\newblock Data-driven surrogate-assisted multiobjective evolutionary optimization of a trauma system.
\newblock {\em IEEE Transactions on Evolutionary Computation}, 20(6):939--952, 2016.

\bibitem{blast_furnace}
Qun Zhou, Yongliang Yin, Daogang Peng, Huirong Zhao, Lei Xing, Xuebin Jiang, Zhenchao Xu, and Chunmei Xu.
\newblock Multi-objective optimization of blast furnace dosing and operation based on nsga-ii.
\newblock In {\em 2022 4th International Conference on Electrical Engineering and Control Technologies (CEECT)}, pages 165--169. IEEE, 2022.

\bibitem{DDEA}
Yaochu Jin, Handing Wang, Tinkle Chugh, Dan Guo, and Kaisa Miettinen.
\newblock Data-driven evolutionary optimization: An overview and case studies.
\newblock {\em IEEE Transactions on Evolutionary Computation}, 23(3):442--458, 2018.

\bibitem{fused_magnesium_furnaces}
Dan Guo, Tianyou Chai, Jinliang Ding, and Yaochu Jin.
\newblock Small data driven evolutionary multi-objective optimization of fused magnesium furnaces.
\newblock In {\em 2016 IEEE Symposium Series on Computational Intelligence (SSCI)}, pages 1--8, 2016.

\bibitem{gupta2015multifactorial}
Abhishek Gupta, Yew-Soon Ong, and Liang Feng.
\newblock Multifactorial evolution: Toward evolutionary multitasking.
\newblock {\em IEEE Transactions on Evolutionary Computation}, 20(3):343--357, 2015.

\bibitem{MaTDE}
Yongliang Chen, Jinghui Zhong, Liang Feng, and Jun Zhang.
\newblock An adaptive archive-based evolutionary framework for many-task optimization.
\newblock {\em IEEE Transactions on Emerging Topics in Computational Intelligence}, 4(3):369--384, 2020.

\bibitem{SADE-KT}
Yuanchao Liu, Jianchang Liu, Jinliang Ding, Shangshang Yang, and Yaochu Jin.
\newblock A surrogate-assisted differential evolution with knowledge transfer for expensive incremental optimization problems.
\newblock {\em IEEE Transactions on Evolutionary Computation}, 28(4):1039--1053, 2024.

\bibitem{bonilla2007multi}
Edwin~V Bonilla, Kian Chai, and Christopher Williams.
\newblock Multi-task gaussian process prediction.
\newblock {\em Advances in neural information processing systems}, 20, 2007.

\bibitem{min2020generalizing}
Alan Tan~Wei Min, Abhishek Gupta, and Yew-Soon Ong.
\newblock Generalizing transfer bayesian optimization to source-target heterogeneity.
\newblock {\em IEEE Transactions on Automation Science and Engineering}, 18(4):1754--1765, 2020.

\bibitem{tan2024surrogate}
Shenglian Tan, Yong Wang, Guangyong Sun, Tong Pang, and Ke~Tang.
\newblock A surrogate-assisted evolutionary framework for expensive multitask optimization problems.
\newblock {\em IEEE Transactions on Evolutionary Computation}, 2024.

\bibitem{Context_BO}
Youngseog Chung, Ian Char, Willie Neiswanger, Kirthevasan Kandasamy, Andrew~Oakleigh Nelson, Mark~D Boyer, Egemen Kolemen, and Jeff Schneider.
\newblock Offline contextual bayesian optimization for nuclear fusion, 2020.

\bibitem{Domain_Adaptation}
Ray Lim, Abhishek Gupta, Yew-Soon Ong, Liang Feng, and Allan~N. Zhang.
\newblock Non-linear domain adaptation in transfer evolutionary optimization.
\newblock {\em Cognitive Computation}, 13(2):290--307, March 2021.

\bibitem{tab_pfn}
Noah Hollmann, Samuel M{\"{u}}ller, Katharina Eggensperger, and Frank Hutter.
\newblock Tabpfn: {A} transformer that solves small tabular classification problems in a second.
\newblock In {\em The Eleventh International Conference on Learning Representations, {ICLR} 2023, Kigali, Rwanda, May 1-5, 2023}. OpenReview.net, 2023.

\bibitem{gnn_runtime_prediction}
Yanjie Gao, Xianyu Gu, Hongyu Zhang, Haoxiang Lin, and Mao Yang.
\newblock Runtime performance prediction for deep learning models with graph neural network.
\newblock In {\em 45th {IEEE/ACM} International Conference on Software Engineering: Software Engineering in Practice, SEIP@ICSE 2023, Melbourne, Australia, May 14-20, 2023}, pages 368--380. {IEEE}, 2023.

\bibitem{rlhf}
Daniel~M. Ziegler, Nisan Stiennon, Jeffrey Wu, Tom~B. Brown, Alec Radford, Dario Amodei, Paul Christiano, and Geoffrey Irving.
\newblock Fine-tuning language models from human preferences, 2020.

\bibitem{NLM}
Yoshua Bengio, R\'{e}jean Ducharme, Pascal Vincent, and Christian Janvin.
\newblock A neural probabilistic language model.
\newblock {\em J. Mach. Learn. Res.}, 3(null):1137–1155, March 2003.

\bibitem{LLM_servey}
Shervin Minaee, Tomas Mikolov, Narjes Nikzad, Meysam Chenaghlu, Richard Socher, Xavier Amatriain, and Jianfeng Gao.
\newblock Large language models: A survey, 2024.

\bibitem{COCO}
N.~Hansen, A.~Auger, R.~Ros, O.~Mersmann, T.~Tu{\v s}ar, and D.~Brockhoff.
\newblock {COCO}: A platform for comparing continuous optimizers in a black-box setting.
\newblock {\em Optimization Methods and Software}, 36:114--144, 2021.

\bibitem{SAINT}
Gowthami Somepalli, Micah Goldblum, Avi Schwarzschild, C.~Bayan Bruss, and Tom Goldstein.
\newblock Saint: Improved neural networks for tabular data via row attention and contrastive pre-training, 2021.

\bibitem{FTTransformer}
Yury Gorishniy, Ivan Rubachev, Valentin Khrulkov, and Artem Babenko.
\newblock Revisiting deep learning models for tabular data.
\newblock In A.~Beygelzimer, Y.~Dauphin, P.~Liang, and J.~Wortman Vaughan, editors, {\em Advances in Neural Information Processing Systems}, 2021.

\bibitem{Evidential-MLP}
Alexander Amini, Wilko Schwarting, Ava Soleimany, and Daniela Rus.
\newblock Deep evidential regression.
\newblock In H.~Larochelle, M.~Ranzato, R.~Hadsell, M.F. Balcan, and H.~Lin, editors, {\em Advances in Neural Information Processing Systems}, volume~33, pages 14927--14937. Curran Associates, Inc., 2020.

\bibitem{MTGP}
Bin Cao, Sinno~Jialin Pan, Yu~Zhang, Dit-Yan Yeung, and Qiang Yang.
\newblock Adaptive transfer learning.
\newblock {\em Proceedings of the AAAI Conference on Artificial Intelligence}, 24(1):407--412, Jul. 2010.

\bibitem{dientbasedattribution}
Marco Ancona, Enea Ceolini, Cengiz Öztireli, and Markus Gross.
\newblock Towards better understanding of gradient-based attribution methods for deep neural networks, 2018.

\bibitem{MTO_benchmark}
Bingshui Da, Yew-Soon Ong, Liang Feng, A.~K. Qin, Abhishek Gupta, Zexuan Zhu, Chuan-Kang Ting, Ke~Tang, and Xin Yao.
\newblock Evolutionary multitasking for single-objective continuous optimization: Benchmark problems, performance metric, and baseline results, 2017.

\bibitem{wilcoxon}
Joaqu{\'\i}n Derrac, Salvador Garc{\'\i}a, Daniel Molina, and Francisco Herrera.
\newblock A practical tutorial on the use of nonparametric statistical tests as a methodology for comparing evolutionary and swarm intelligence algorithms.
\newblock {\em Swarm and Evolutionary Computation}, 1(1):3--18, 2011.

\bibitem{BLKT-DE}
Yi~Jiang, Zhi-Hui Zhan, Kay~Chen Tan, and Jun Zhang.
\newblock Block-level knowledge transfer for evolutionary multitask optimization.
\newblock {\em IEEE Transactions on Cybernetics}, 54(1):558--571, 2024.

\end{thebibliography}

\end{document}